\documentclass[12pt]{article}

\usepackage{amsmath,amsfonts,amssymb,amsthm,graphicx,color,subfigure}
\usepackage[normalem]{ulem}
\usepackage{xcolor}
\usepackage[utf8]{inputenc}
\usepackage{verbatim}
\usepackage{algorithmic}
\usepackage{algorithm}
\usepackage{multirow}
\usepackage{tikz}
\usepackage{tikz-qtree}
\usepackage{url}
\usepackage{cite}
\usetikzlibrary{shadows,trees}

\DeclareGraphicsExtensions{.pdf,.png,.jpg}

\def\cmc{\textsc{CEC-IB}}

\renewcommand{\algorithmiccomment}[1]{\bgroup\hfill//~#1\egroup}

\def\N{\mathcal{N}}
\def\Z{\mathbb{Z}}
\def\R{\mathbb{R}}

\def\F{\mathcal{G}}

\def\1{\mathds{1}}

\def\X{\mathcal{X}}
\def\E{\mathrm{E}}
\def\Y{\mathcal{Y}}
\def\Z{\mathcal{Z}}

\newtheorem{theorem}{Theorem}[section]

\newtheorem{lemma}{Lemma}[section]

\theoremstyle{definition}
\newtheorem{definition}{Definition}[section]
\newtheorem{remark}{Remark}[section]

\newtheorem{example}{Example}[section]

\newcommand{\crossfi}{H(\N(\mu_{Y_i},\Sigma_{Y_i}))} 

\tikzset{font=\footnotesize,
edge from parent fork down,
level distance=1.75cm,
every node/.style=
    {top color=white,
    minimum height=8mm,
    align=center,
    text depth = 0pt
    },
edge from parent/.style=
    {draw=blue!50,
    thick
    }}

\title{Semi-supervised cross-entropy clustering with information bottleneck constraint}
\author{Marek \'Smieja$^{1}$, Bernhard C. Geiger$^{2}$}
\date{\normalsize $^{1}$Faculty of Mathematics and Computer Science\\ Jagiellonian University\\ \L{}ojasiewicza 6, 30-348 Krakow, Poland\\
$^{2}$Institute for Communications Engineering\\ Technical University of Munich\\Theresienstr. 90, D-80333 Munich, Germany}

\begin{document}

\maketitle

\begin{abstract}
In this paper, we propose a semi-supervised clustering method, {\cmc}, that models data with a set of Gaussian distributions and that retrieves clusters based on a partial labeling provided by the user (partition-level side information). By combining the ideas from cross-entropy clustering (CEC) with those from the information bottleneck method (IB), our method trades between three conflicting goals: the accuracy with which the data set is modeled, the simplicity of the model, and the consistency of the clustering with side information. Experiments demonstrate that {\cmc} has a performance comparable to Gaussian mixture models (GMM) in a classical semi-supervised scenario, but is faster, more robust to noisy labels, automatically determines the optimal number of clusters, and performs well when not all classes are present in the side information. Moreover, in contrast to other semi-supervised models, it can be successfully applied in discovering natural subgroups if the partition-level side information is derived from the top levels of a hierarchical clustering.
\end{abstract}

\section{Introduction}

Clustering is one of the core techniques of machine learning and data analysis, and aims at partitioning data sets based on, e.g., the internal similarity of the resulting clusters. While clustering is an unsupervised technique, one can improve its performance by introducing additional knowledge as side information. This is the field of semi-supervised or constrained clustering.

One classical type of side information in clustering are pairwise constraints: human experts determine whether a given pair of data points belongs to the same (must-link) or to different clusters (cannot-link) \cite{basu2008constrained}. Although this approach received high attention in the last decade, the latest reports \cite{yi2012semi} suggest that in real-life problems it is difficult to answer whether or not two objects belong to the same group without a deeper knowledge of data set. This is even more problematic as erroneous pairwise constraints can easily lead to contradictory side information \cite{fei2005bayesian}. 

A possible remedy is to let experts categorize a set of data points rather than specifying pairwise constraints. This \emph{partition-level side information} was proposed in \cite{basu2003semi} and recently considered in \cite{liu2015clustering}. The concept is related to partial labeling applied in semi-supervised classification and assumes that a small portion of data is labeled. In contrast to semi-supervised classification \cite{zhu2009introduction, Tu2016673}, the number of categories is not limited to the true number of classes; in semi-supervised clustering one may discover several clusters among unlabeled data points. Another advantage of partition-level side information is that, in contrast to pairwise constraints, it does not become self-contradictory if some data points are mislabeled.

In this paper, we introduce a semi-supervised clustering method, {\cmc}, based on partition-level side information. {\cmc} combines Cross-Entropy Clustering (CEC) \cite{tabor2014cross, cec_guide, spurek2017active}, a model-based clustering technique, with the Information Bottleneck (IB) method \cite{tishby2000information, chechik2005information}
to build {\bf the smallest model that preserves the side information and provides a good model of the data distribution}. In other words, {\cmc} automatically determines the required number of clusters to trade between model complexity, model accuracy, and consistency with the side information.

Consistency with side information is ensured by penalizing solutions in which data points from different categories are put in the same cluster. Since modeling a category by multiple clusters is not penalized, one can apply {\cmc} to obtain a fine clustering even if the  the human expert categorizes the data into only few basic groups, see Figure \ref{fig:hierarchyDiagram}. Although this type of side information seems to be a perfect use case for cannot-link constraints, the computational cost of introducing side information to {\cmc} is negligible while the incorporation of cannot-link constraints to similar Gaussian mixture model (GMM) approaches requires the use of graphical models, which involves high computational cost.  {\cmc} thus combines the flexibility of cannot-link constraints with an efficient implementation.

\begin{figure}[t]
\centering
\begin{tikzpicture}[every tree node/.style={top color=white,
    bottom color=blue!25,
    rectangle,rounded corners,
    minimum height=8mm,
    draw=blue!75,
    very thick,
    align=center,
    text depth = 0pt},edge from parent/.style={draw=blue!50,
    thick}]
\Tree [.\node (Root) {Data};
        [.{category A}
            [.{A$_1$} ]
            [.\node[draw=none, bottom color=white]{$\hspace{1em}\ldots \hspace{1em}$};]
            [.{A$_k$} ] ] 
        [.{category B}
            [.{B$_1$} ]
            [.\node[draw=none, bottom color=white]{ $ \hspace{1em} \ldots \hspace{1em} $};]
            [.{B$_k$} ] ] 
]
    \begin{scope}[xshift=1.6in,every tree node/.style={top color=white,
    bottom color=red!25,
    rectangle,rounded corners,
    minimum height=8mm,
    draw=red!75,
    very thick,
    align=center,
    text depth = 0pt},edge from parent path={}]
\Tree [.\node[draw=none, bottom color=white]{}; 
[ .\node[draw=none, bottom color=white, text=red]{ $\Longrightarrow$};]
 [ .{side-information} ]
 ]
\end{scope}
    \begin{scope}[xshift=-1.6in,every tree node/.style={top color=white,
    bottom color=red!25,
    rectangle,rounded corners,
    minimum height=8mm,
    draw=red!75,
    very thick,
    align=center,
    text depth = 0pt},edge from parent path={}]
\Tree [.\node[draw=none, bottom color=white]{}; 
 [ .{side-information} ]
 [ .\node[draw=none, bottom color=white, text=red]{ $\Longleftarrow$};]
 ]
\end{scope}
\end{tikzpicture}
\caption{Subgroups discovery task. The expert provides side information by dividing a data set into two categories. Making use of this knowledge, the algorithm discovers natural subgroups more reliably than in the unsupervised case.}
\label{fig:hierarchyDiagram}
\end{figure}
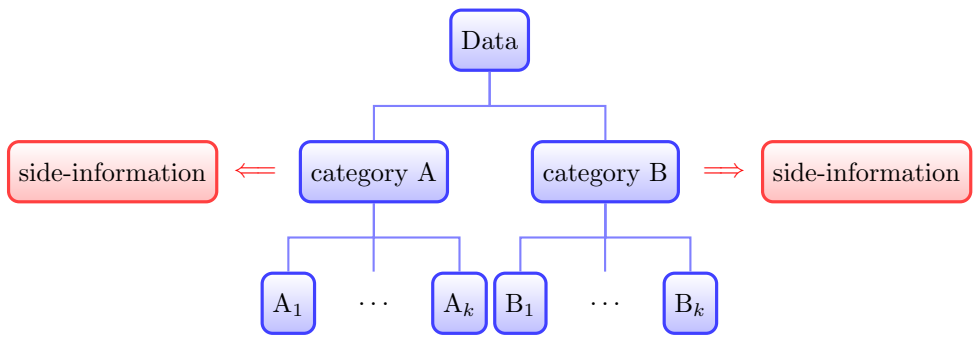




We summarize the main contributions and the outline of our paper:
\begin{enumerate}
 \item We combine ideas from model-based CEC and from the information bottleneck method to formulate our clustering method {\cmc} for both complete and partial side information (Sections~\ref{sec:model:IB} and~\ref{sec:model:partial}). The proposed method does not require the true number of clusters as an input.
 \item We propose a modified Hartigan algorithm to optimize the {\cmc} cost function (Section \ref{sec:model:algo}). The algorithm has a complexity that is linear in the number of data points in each iteration, and it usually requires less iterations than the expectation-maximization (EM) algorithm used for fitting GMMs (\ref{app:optimizationTime}).
 \item We provide a theoretical analysis of the trade-off between the CEC and the IB cost function in Section~\ref{sec:weight}. This places the parameter selection problem on a solid mathematical ground (see Theorems \ref{thm:split} and \ref{thm:limitBeta}). 
 \item We perform extensive experiments demonstrating that {\cmc} is more robust to noisy side information (i.e., miscategorized data points) than state-of-the-art approaches to semi-supervised clustering (Section~\ref{sec:experiment:noisy}). Moreover, {\cmc} performs well when not all categories are present in the side information (Section~\ref{sec:experiment:few}), even though the true number of clusters is not specified.
 \item We perform two case studies: In Section~\ref{sec:chem}, a human expert provided a partition-level side information about the division of chemical compounds into two basic groups (as in Figure \ref{fig:hierarchyDiagram}); {\cmc} discovers natural chemical subgroups more reliably than other semi-supervised methods, even if some labels are misspecified by the expert (Figure \ref{fig:chem-intro}). The second case study in Section~\ref{sec:image} applies {\cmc} to image segmentation.
\end{enumerate}

\begin{figure}
\begin{center}
 \includegraphics[width=2.5in]{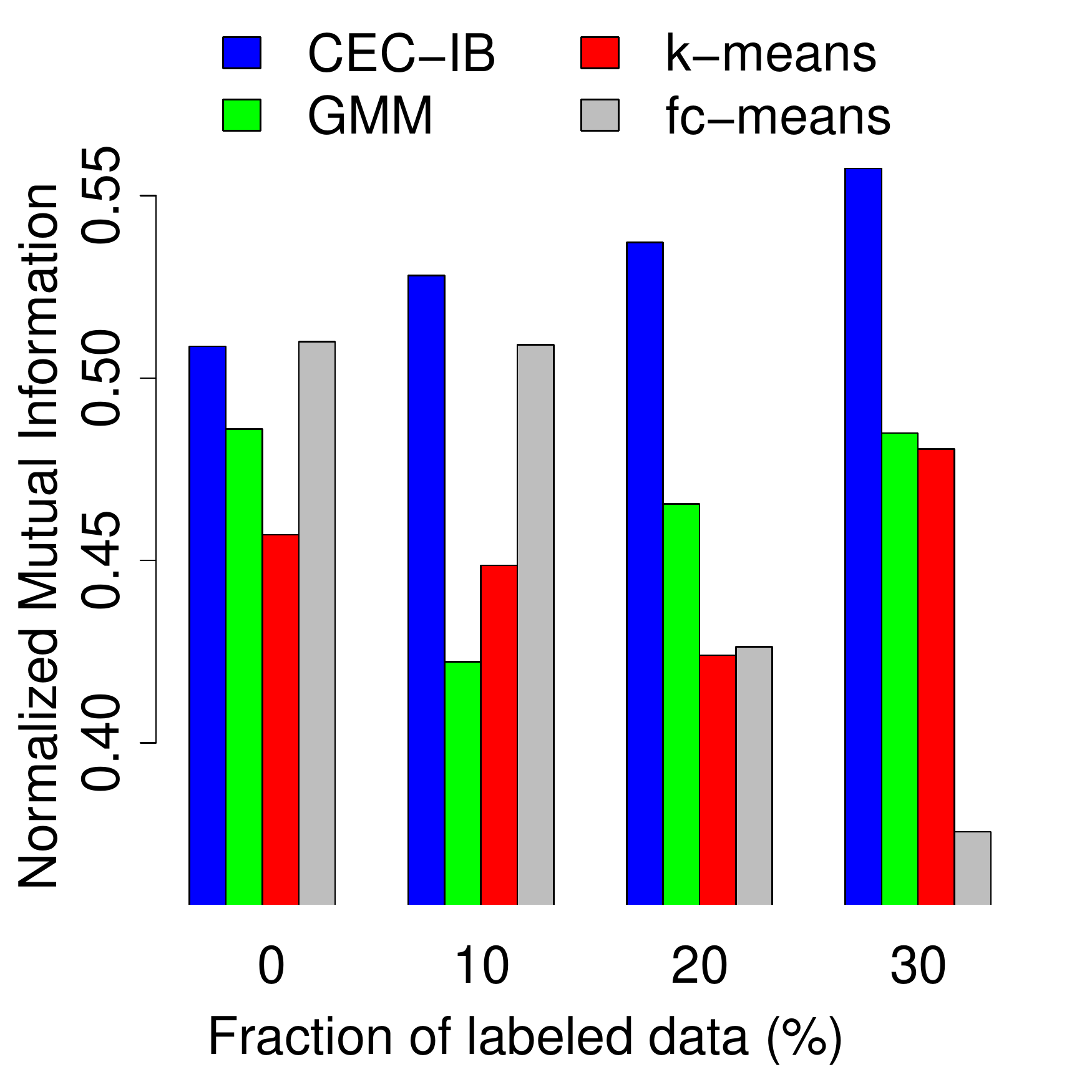}
\end{center}
\caption{Detection of chemical subgroups when 10\% of side information was erroneous. The results of {\cmc}, GMM with cannot-link constraints, constrained k-means and fuzzy c-means (fc-means) were measured by normalized mutual information.}
\label{fig:chem-intro}
\end{figure}


\section{Related work}\label{sec:related}

Clustering has been an important topic in machine learning and data analysis for a long time. Various methods were introduced for splitting data into groups, including model-based, distance-based, spectral, fuzzy, and hierarchical methods (see~\cite{jain1999data, aggarwal2013data} for a survey).



Adding to this diversity of techniques, a large number of specialized types of clustering have been developed. 
One example is multi-view clustering, which considers gathering information coming from different domains \cite{jiang2015collaborative}. As another example, complementary or alternative clustering aims at finding groups which provide a perspective on the data that expands on what can be inferred from previous clusterings~\cite{gondek2007non}. 
Finally, semi-supervised clustering -- the problem investigated in this work -- makes use of side information to achieve better clustering results or to provide robustness against noisy side information~\cite{basu2008constrained}.

The traditional approach to incorporate side information into clustering is based on pairwise constraints. The authors of~\cite{Asafi_ConstraintsFeatures} suggested reducing distances between data points with a must-link constraint and adding a dimension for each cannot-link constraint. After updating all other distances to, e.g., satisfy the triangle inequality, the thus obtained pairwise distance matrix can be used for unsupervised clustering. Kamvar et al.~\cite{Kamvar_SpectralLearning} considered a similar procedure, taking the pairwise affinity matrix and setting must-links and cannot-links to predefined maximum and minimum values, respectively. Instead of clustering, they applied eigenvector-based classification taking the labeled data as training set. Another spectral technique, proposed in~\cite{Wang_FCSC}, relies on solving a generalized eigenvalue problem. Qian et al. \cite{qian2016affinity} developed a framework for spectral clustering that allows using side information in the form of pairwise constraints, partial labeling, and grouping information. An information-theoretic cost function, squared mutual information, was proposed for semi-supervised clustering in~\cite{Calandriello_InfoMaxClust}. Also clustering techniques based on non-negative matrix or concept factorization can incorporate pairwise constraints as regularizers~\cite{lu2016semi}.

As mentioned in the introduction, partition-level side information refers to a partial labeling of the data points that need not necessarily consider all classes -- the categories provided as side information may be only a subset of classes, or, as in Figure~\ref{fig:hierarchyDiagram}, be of a hierarchical nature. In consequence, clustering with partition-level side information differs significantly from a typical semi-supervised classification task, as the clustering algorithm should detect clusters within categories and/or within unlabeled data points. A recent paper using partition-level side information is~\cite{liu2015clustering}, where the authors add additional dimensions to feature vectors and propose a modification of k-means to cluster data points. In~\cite{basu2003semi}, partition-level side information was used to design a better initialization strategy for k-means. Similarly, partition-level side information was used to propose a semi-supervised version of fuzzy c-means~\cite{pedrycz1997fuzzy, pedrycz2008fuzzy}. The authors added a regularization term to the fuzzy c-means cost function that penalizes fuzzy clusterings that are inconsistent with the side information. This technique was later combined with feature selection methods~\cite{lai2013improving}. Finally, partition-level side information can be used in density-based clustering such as DBSCAN. Specifically, in~\cite{lelis2009semi} the authors proposed an algorithm that sets the parameter defining the neighborhood radius of a data point based on partial labeling.

GMMs can be easily adapted to make us of partition-level side information by combining the classical unsupervised GMM with a supervised one \cite{ambroise2001learning, zhu2009introduction}. This approach can be extended to labels with reliability information \cite{come2009learning, hullermeier2005learning, bouveyron2009robust}. Various statistical and machine learning libraries, such as mixmod \cite{lebret2015rmixmod} or bgmm \cite{biecek2012r}, provide implementations of GMMs with partition-level side information.

Also pairwise constraints can be incorporated into GMMs, where dependencies between the hidden cluster indicator variables are then usually modeled by a hidden Markov random field. This procedure was adopted, for example, in~\cite{shental2004computing} to account for cannot-link constraints. Must-link constraints were considered by treating all involved data points as a single data point with a higher weight. The parameters of the GMM, which was used for hard or soft clustering, are obtained by a generalized expectation-maximization procedure that requires simplifications or approximations~\cite{Lu_PPC, basu2004probabilistic, lange2005learning}. An overview of GMM-based methods with pairwise constraints can be found in~\cite{Nelson_RevisitingModels}. 

In contrast to most GMM approaches, our method does not require knowledge of the correct number of clusters; initialized with any (larger) number, {\cmc} reduces the number of clusters for an optimal trade-off between model accuracy, model complexity (i.e., number of clusters), and consistency with the side information. 
%

Our method is closely related to the information bottleneck method, which focuses on lossy compression of data preserving the information of a stochastically related random variable \cite{tishby2000information, slonim2002information}. Modifications of IB were used in consensus clustering \cite{topchy2003combining} or alternative clustering \cite{gondek2007non}. The mutual information between data points and its clusters, which describes the cost of (lossy) data compression in IB, is replaced in our model by the cross-entropy -- see Section~\ref{sec:model:IB} for more details. Thus, while IB focuses model simplicity and consistency with side information, {\cmc} adds model accuracy to the cost.


\section{Cross-Entropy Clustering with an Information Bottleneck Constraint}\label{sec:model}
\newcommand{\Prob}[1]{\mathbf{P}(#1)}

We now pave the way for our {\cmc} method. Since our model is related to CEC, we first review its basics in Section~\ref{sec:model:cec}. For completely labeled data, i.e., for the case where all data points are labeled, we then introduce our {\cmc} model based on ideas from IB in Section~\ref{sec:model:IB}. Section~\ref{sec:model:partial} extends the analysis to deal with the case where only some data points are labeled. We conclude this section by presenting and analyzing a clustering algorithm that finds a local optimum of our {\cmc} cost function.

\subsection{Cross-entropy clustering}\label{sec:model:cec}

CEC is a model-based clustering method that minimizes the empirical cross-entropy between a finite data set $X \subset \R^N$ and a parametric mixture of densities \cite{tabor2014cross}. This parametric mixture is a subdensity\footnote{i.e., $f(x)\ge 0$ and $\int_{\R^N} f(x)dx \le 1$.} given by
$$
f=\max(p_1 f_1, \ldots, p_k f_k)
$$
where $p_1$ through $p_k$ are non-negative weights summing to one and where $f_1$ through $f_k$ are densities from the Gaussian family $\F$ of probability distributions on $\R^N$. The empirical cross-entropy between $X$ and subdensity $f$ is
$$
H^\times(X \| f) = - \frac{1}{|X|} \sum_{x \in X} \log f(x) =-\frac{1}{|X|} \sum_{i=1}^k \sum_{x \in Y_i} \log(p_i f_i(x))
$$
where 
\begin{equation}\label{eq:partition}
  \Y = \{Y_1,\ldots,Y_k\},\quad Y_i :=\{x\in X{:}\ p_if_i (x)  = \max_{j} p_j f_j(x)\}
\end{equation}
is a partition of $X$ induced by the subdensity $f$. Letting 
\begin{flalign*}
& \mu_{Y_i} = \frac{1}{|Y_i|} \sum_{x \in Y_i} x,\\
& \Sigma_{Y_i} = \frac{1}{|Y_i|} \sum_{x \in Y_i} (x - \mu_{Y_i})(x - \mu_{Y_i})^T
\end{flalign*}
be the sample mean vector and sample covariance matrix of cluster $Y_i$, we show in~\ref{app:cec} that CEC looks for a clustering $\Y$ such that the following cost is minimized:
\begin{equation} \label{eq:cec}
H^\times(X \| f) =  H(\Y) + \sum_{i=1}^k  \frac{|Y_i|}{|X|} H(\N(\mu_{Y_i},\Sigma_{Y_i})),
\end{equation}
where the model complexity is measured by the Shannon entropy of the partition $\Y$,
$$
H(\Y) := - \sum_{i=1}^k \frac{|Y_i|}{|X|} \log \frac{|Y_i|}{|X|},
$$
and where the model accuracy, i.e., accuracy of density estimation in cluster $Y_i$, is measured by the differential entropy of the Gaussian density $f_i$,
$$
H(\N(\mu_{Y_i},\Sigma_{Y_i})) = \tfrac{N}{2}\ln(2\pi e)+\tfrac{1}{2}\ln \det (\Sigma_{Y_i}) = \min_{f_i\in\F} H^\times(Y_i \| f_i).
$$

The main difference between CEC and GMM-based clustering lies in substituting a mixture density $f=p_1 f_1 + \cdots + p_k f_k$ by a subdensity $f=\max(p_1 f_1, \ldots,p_k f_k)$. This modification allows to obtain a closed form solution for the mixture density given a fixed partition $\Y$, while for a fixed mixture density the partition $\Y$ is given in~\eqref{eq:partition}. This suggests a heuristic similar to the k-means method. In consequence, CEC might yield a slightly worse density estimation of data than GMM, but converges faster (see Section~\ref{sec:model:algo} and Appendix~\ref{app:optimizationTime}) while the experimental results show that the clustering effects are similar.

\subsection{CEC-IB with completely labeled data}\label{sec:model:IB}

We now introduce {\cmc} for completely labeled data (i.e., all data points are labeled) by combining the ideas from model-based clustering with those from the information bottleneck method. We also show that under some assumptions {\cmc} admits an alternative derivation based on conditional cross-entropy given the side information.
 
\begin{definition}
Let $X$ be a finite data set and let $X_\ell\subseteq X$ denote the set of labeled data points. The \emph{partition-level side information} is a partition $\Z=\{Z_1,\ldots,Z_m\}$ of $X_\ell$, where every $Z_j \in \Z$ contains all elements of $X_\ell$ with the same label.
\end{definition}

To make this definition clear, suppose that $\X=\{X_1,X_2,\dots,X_l\}$ is the \emph{true} partition of the data that we want to recover, i.e., we want to obtain $\Y=\X$. The partition-level side information $\Z$ can take several possible forms, including:
\begin{itemize}
 \item $|\Z|=l$, and $Z_j\subseteq X_j$ for $j=1,\dots,l$. This is equivalent to the notion of partial labeling in semi-supervised classification.
 \item $|\Z|=m<l$ and for every $j=1,\dots,m$ there is a different $i$ such that $Z_j\subseteq X_i$. This is the case where only some of the true clusters are labeled.
 \item $|\Z|=m<l$ and there are $m$ disjoint sets $I_j\subset\{1,\dots,l\}$ such that $Z_j\subset\bigcup_{i\in I_j} X_i$. This is the case where the labeling is derived from a higher level of the hierarchical true clustering (cf.~Figure~\ref{fig:hierarchyDiagram}).
\end{itemize}

For the remainder of this subsection, we assume that the side information is complete, i.e., that each data point $x\in X$ is labeled with exactly one category. In other words, $X_\ell = X$ and $\Z$ is a partition of $X$. We drop this assumption in Section~\ref{sec:model:partial}, where we consider partial labeling, i.e., $X_\ell\subseteq X$.

Our effort focuses on finding a partition that is consistent with side information:
\begin{definition}\label{def:consistency}
Let $X$ be a finite data set and let $X_\ell \subseteq X$ be the set of labeled data points that is partitioned into $\Z=\{Z_1,\ldots,Z_m\}$. We say that a partition $\Y=\{Y_1,\dots,Y_k\}$ of $X$ is \emph{consistent} with $\Z$, if for every $Y_i$ there exists at most one $Z_j$ such that $Z_j \cap Y_i\neq \emptyset$.
\end{definition}
The definition of consistency generalizes the refinement relation between partitions of the same set. If, as in this section, $X_\ell=X$, then $\Y$ is consistent with $\Z$ if and only if $\Y$ is a refinement of $\Z$. In other words, a clustering $\Y$ is consistent with $\Z$ if every $Y_i \in \Y$ contains elements from at most one category $Z_j \in \Z$. Mathematically, for a clustering $\Y$ consistent with $\Z$ we have
\begin{equation}\label{eq:consistent}
 \forall Y_i\in\Y{:}\ \exists! j'=j'(i){:}\ Z_j \cap Y_i =\begin{cases}  Y_i & j=j'\\0 & \text{else}.\end{cases}
\end{equation}
Thus, for a consistent clustering $\Y$ the conditional entropy $H(\Z|\Y)$ vanishes:\begin{multline*}
H(\Z|\Y) 
= \sum_{i=1}^k \frac{|Y_i|}{|X|} H(\Z | Y_i)
= - \sum_{i=1}^k \sum_{j=1}^m \frac{|Z_j \cap Y_i|}{|X|} \log \left( \frac{|Z_j \cap Y_i|}{|Y_i|} \right)\\
\stackrel{(a)}{=} - \sum_{i=1}^k \frac{|Z_{j'} \cap Y_i|}{|X|} \log \left( \frac{| Y_i|}{|Y_i|} \right) =0
\end{multline*}
where $(a)$ is due to~\eqref{eq:partition}.


The conditional entropy $H(\Z|\Y)$ therefore is a measure for consistency with side information: the smaller the conditional entropy, the higher is the consistency. We thus propose the following cost function for {\cmc}  in the case of complete side information, i.e., when $X_\ell = X$ and $\Z$ is a partition of $X$:
\begin{equation}\label{eq:costCond}
 \E_\beta(X, \Z; \Y) :=H(\Y) + \sum_{i=1}^k \frac{|Y_i|}{|X|} \crossfi + \beta H(\Z| \Y) \text{, where } \beta \geq 0.
\end{equation}
The first two terms are the CEC cost function~\eqref{eq:cec}, and the last term $H(\Z|\Y)$ penalizes clusterings $\Y$ that are not consistent with the side information $\Z$. Thus {\cmc} aims at finding the minimal number of clusters needed to model the data set distribution and to preserve the consistency with the side information. The weight parameter $\beta$ trades between these objectives; we will analyze rationales for selecting this parameter in Section~\ref{sec:weight}. 

Our cost function~\eqref{eq:costCond} is intricately connected to the IB and related methods. In the notation of this work, i.e., in terms of partitions rather than random variables, the IB cost function is given as \cite{tishby2000information}
$$
I(X; \Y) - \beta I(\Y; \Z) = H(\Y) - H(\Y|X) - \beta H(\Z) + \beta H(\Z|\Y).
$$
Noticing that $H(\Z)$ does not depend on the clustering $\Y$, the main difference between IB and {\cmc} is that {\cmc} incorporates a term accounting for the modeling accuracy in each cluster, while IB adds a term related to the ``softness'' of the clustering: Since $H(\Y|X)$ is minimized for deterministic, i.e., hard clusters, IB implicitly encourages soft clusters. A version of IB ensuring deterministic clusters was recently introduced in~\cite{Strouse_DIB}. The cost function of this method dispenses with the term related to the softness of the clusters leading to a clustering method minimizing 
$$
H(\Y)+ \beta H(\Z|\Y).
$$
Our {\cmc} method can thus be seen as deterministic IB with an additional term accounting for model accuracy. {\cmc} can therefore be considered as a model-based version of the information bottleneck method. 

We end this subsection by showing that under some assumptions, one can arrive at the {\cmc} cost function (with $\beta=1$) in a slightly different way, by minimizing the conditional cross-entropy function:
\begin{theorem}\label{thm:condcross}
Let $X$ be a finite data set that is partitioned into $\Z=\{Z_1,\ldots,Z_m\}$. Minimizing the {\cmc} cost function~\eqref{eq:costCond}, for $\beta=1$, is equivalent to minimizing the conditional cross-entropy function:
 $$
H^\times((X \| f) |\Z) := \sum_{j=1}^m \frac{|Z_j|}{|X|} H^\times(Z_j \| f_{|j}),
$$
where 
$$
f_{|Z_j}:=f_{|j} = \max(p_1(j) f_1,\ldots,p_k(j) f_k)
$$
is the conditional density $f$ given $j$-th category and where $p_1(j),\ldots,p_k(j)$ are non-negative weights summing to one.
\end{theorem}

The proof of this theorem is given in~\ref{app:derivation}. It essentially states that our cost function ensures that each category $Z_j$ is modeled by a parametric mixture of densities that is both simple and accurate. We believe that this view on the problem can lead to the development of a clustering algorithm slightly different from what is presented in this paper. 

\subsection{CEC-IB with partially labeled data}\label{sec:model:partial}
\newcommand{\Lab}{\mathcal{L}}

The previous section assumed that all data points in $X$ were labeled, i.e., the partition-level side information $\Z=\{Z_1,\ldots,Z_m\}$ was a partition of $X$. In this section, we relax this assumption and assume that only a subset $X_\ell \subseteq X$ is labeled. In this case $\Z$ is a partition only of $X_\ell$, and in consequence, the conditional entropy $H(\Z|\Y)$ from the previous subsection is undefined.

To deal with this problem, let $\Lab = \{X_\ell, X \setminus X_\ell\}$ denote the partition of $X$ into labeled and unlabeled data. We decompose the conditional entropy of $\Z$ given partitions $\Y$ and $\Lab$ as
\begin{equation}\label{eq:totalcost}
 H(\Z | \Y, \Lab) = \frac{|X_\ell|}{|X|}  H(\Z | \Y, X_\ell) + \frac{|X\setminus X_\ell|}{|X|} H(\Z | \Y, X\setminus X_\ell),
\end{equation}
where
\begin{align*} 
 H(\Z | \Y, X_\ell) 
 &=\sum_{i=1}^k  \frac{|Y_i \cap X_\ell|}{|X_\ell|} H(\Z | Y_i\cap X_\ell) \notag\\
 &=\sum_{i=1}^k  \frac{|Y_i \cap X_\ell|}{|X_\ell|} \sum_{j=1}^m \frac{|Y_i \cap Z_j|}{|Y_i \cap X_\ell|} \left(-\log \frac{|Y_i \cap Z_j|}{|Y_i \cap X_\ell|}\right).
\end{align*}

Let us now assume that the partition-level side information is a representative sample of true categories. In other words, assume that the probability for a category of an unlabeled data point given the cluster equals the empirical probability of this category for labeled data points in this cluster. To formalize this reasoning, we view the partition $\Z$ as a random variable that takes values in $\{1,\dots,m\}$. Our labeled data set $X_\ell$ corresponds to realizations of this random variable, i.e., for every $x \in X_\ell$, the corresponding random variable $\Z$ assumes the value indicated by the labeling. Since the side information was assumed to be representative, the relative fraction of data points in cluster $Y_i$ assigned to category $Z_j$ gives us an estimate of the true underlying probability; we extrapolate this estimate to unlabeled data points and put
$$
\Prob{\Z=j|Y_i \cap (X\setminus X_\ell)}=\Prob{\Z=j|Y_i \cap X_\ell}=\frac{|Y_i \cap Z_j|}{|Y_i \cap X_\ell|} = \Prob{\Z=j|Y_i}.
$$
Hence, $H(\Z | Y_i\cap (X\setminus X_\ell)) = H(\Z | Y_i\cap X_\ell)$ for every $Y_i$, and we get for \eqref{eq:totalcost}:
\begin{align}
 H(\Z | \Y)&=H(\Z | \Y, \Lab)\\
 & = \frac{|X_\ell|}{|X|}  H(\Z | \Y, X_\ell) + \frac{|X\setminus X_\ell|}{|X|} H(\Z | \Y, X\setminus X_\ell)\\[0.4ex]
& = \frac{|X_\ell|}{|X|}  \sum_{i=1}^k  \frac{|Y_i \cap X_\ell|}{|X_\ell|} H(\Z | Y_i\cap X_\ell) \notag \\ 
&{}+  \frac{|X\setminus X_\ell|}{|X|}  \sum_{i=1}^k \frac{|Y_i \cap (X\setminus X_\ell)|}{|X\setminus X_\ell|} H(\Z | Y_i \cap X_\ell)\\
& =  \sum_{i=1}^k \frac{|Y_i|}{|X|} H(\Z | Y_i \cap X_\ell).\label{eq:condEntPartial}
\end{align}
where the first equality follows because the conditional entropy $H(\Z | \Y, \Lab)$ does not depend on the partition $\Lab$.

With this, we define the {\cmc} cost function for a model with partition-level side information:
\begin{definition} ({\cmc} cost function)
Let $X$ be a finite data set and let $X_\ell \subseteq X$ be the set of labeled data points that is partitioned into $\Z=\{Z_1,\ldots,Z_m\}$. The cost of clustering $X$ into the partition $\Y=\{Y_1,\ldots,Y_k\}$ for a given parameter $\beta \geq 0$ equals
\begin{equation}\label{eq:cost}
\E_\beta(X, \Z; \Y) :=H(\Y) + \sum_{i=1}^k \frac{|Y_i|}{|X|} \left( \crossfi + \beta H(\Z | Y_i \cap X_\ell)\right).
\end{equation}
To shorten the notation we sometimes write $\E_\beta(\Y)$ assuming that $X$ and $\Z$ are fixed.
\end{definition}
Note that for a complete side information, i.e., for $X_\ell=X$, we get precisely the cost function~\eqref{eq:costCond} obtained in the previous subsection.

\subsection{Optimization algorithm}\label{sec:model:algo}

The optimization of {\cmc} cost function can be performed similarly as in the classical CEC method, in which the Hartigan approach \cite{hartigan1979algorithm} is used. 

Let $X$ be a finite data set and let $X_\ell \subseteq X$ be the set of labeled data points that is partitioned into $\Z=\{Z_1,\ldots,Z_m\}$. The entire procedure consists of two steps: initialization and iteration. In the initialization step, a partition $\Y$ is created randomly; $f_i = \N(\mu_{Y_i}, \Sigma_{Y_i})$ are Gaussian maximum likelihood estimators on $Y_i$. In the iteration stage, we go over all data points and reassign each of them to the cluster that decreases the {\cmc} cost~\eqref{eq:cost} the most. After each reassignment, the clusters densities $f_i$ are re-parameterized by the maximum likelihood estimators of new clusters and the cardinalities of categories $Z_j \cap Y_i$ are recalculated. If no cluster membership changed then the method terminates with a partition $\Y$.

Note that this procedure automatically removes unnecessary clusters by introducing the term $H(\Y)$, which is the cost of cluster identification. If the method is initialized with more clusters than necessary, some clusters will lose data points to other clusters in order to reduce $H(\Y)$, and the corresponding clusters may finally disappear altogether (e.g., by the number of data points contained in this cluster falling below a predefined threshold). 

To describe the algorithm in detail, let us denote the cost of a single cluster $Y \subset X$ by
\begin{equation}\label{eq:partialCost}
 \E_\beta(Y) :=  \frac{|Y|}{|X|} \left( -\ln \frac{|Y|}{|X|} + H(\N(\mu_Y, \Sigma_Y)) + \beta H(\Z | Y \cap X_\ell)\right),
\end{equation}
assuming that $X$, $\Z$, and $\beta$ are fixed. Then, for a given partition $\Y$ of $X$ the minimal value of {\cmc} cost function equals:
$$
\E_\beta(X,\Z;\Y) = \sum_{i=1}^k \E_\beta(Y_i).
$$
Making use of the above notation, the algorithm can be written as follows:

\begin{flushleft}
\footnotesize
\begin{algorithmic}[1] 
\STATE \textbf{INPUT:}
	\STATE $X \subset \R^N$ -- data set
	\STATE $\Z = \{Z_1,\ldots,Z_m\}$ -- partition-level side information
	\STATE $k$ -- initial number of clusters
	\STATE $\beta$ -- weight parameter
	\STATE $\varepsilon > 0$ -- cluster reduction parameter
\STATE \textbf{OUTPUT:}\\
	\STATE Partition $\Y$ of $X$\\
\STATE \textbf{INITIALIZATION:}\\
		\STATE $\Y = \{Y_1,\ldots,Y_k\}$ -- random partition of $X$ into $k$ groups 
\STATE \textbf{ITERATION:}\\
	\REPEAT
		\FORALL{$x \in X$}
			\STATE $Y_x \leftarrow$ get cluster of $x$
			\STATE $Y \leftarrow \mathrm{arg} \max\limits_{Y \in \Y} \{ \E_\beta(Y_x) + \E_\beta(Y) - \E_\beta(Y_x \setminus \{x\}) - \E_\beta(Y \cup \{x\}) \}$ \label{pseudo:rec1}
			\IF{$Y \neq Y_x$}
				\STATE move $x$ from $Y_x$ to $Y$
				\STATE update density models of $Y_x$ and $Y$\label{pseudo:rec2}
				\IF{$|Y_x| < \varepsilon \cdot |X|$}
					\STATE delete cluster $Y_x$ and assign its elements to these clusters which minimize the {\cmc} cost function
				\ENDIF
			\ENDIF
		\ENDFOR
	\UNTIL{no switch for all subsequent elements of $X$}
\end{algorithmic}
\end{flushleft}
The outlined algorithm is not deterministic and its results depend on the randomly chosen initial partition. Therefore, the algorithm can be restarted multiple times to avoid getting stuck in bad local minima.

One may think that the recalculation of the models and the evaluation of the cost in lines \ref{pseudo:rec1} and \ref{pseudo:rec2} is computationally complex. Looking at~\eqref{eq:partialCost}, one can see that evaluating the cost for a given cluster requires recomputing the sample mean vector and sample covariance matrix, which, according to \cite[Theorem 4.3.]{tabor2014cross}, has a complexity quadratic in the dimension $N$ of the dataset. Computing the determinant of the sample covariance matrix can be done with cubic complexity. Moreover, computing the conditional entropy of $\Z$ given the current cluster $Y$ is linear in the number $m$ of categories; if the selected data point $x$ is not labeled, then there is no cost at all for computing these terms, since they cancel in the difference in line \ref{pseudo:rec1}. Since in each iteration, all data points have to be visited and, for each data point, all clusters have to be tested, one arrives at a computational complexity in the order of $\mathcal{O}(nk(N^3+m))$. In comparison, Lloyd's algorithm for k-means has a complexity of $\mathcal{O}(nkN)$ in each iteration and the expectation maximization (EM) algorithm to fit a GMM has a complexity of $\mathcal{O}(nkN^2)$~\cite[p.~232]{Redner_EM}. Note, however, that neither classical k-means nor EM is designed to deal with side information, hence, the complexity of semi-supervised algorithms is in general larger. In particular, the addition of cannot-link constraints to GMM can involve a high computational cost. Moreover, in \ref{app:optimizationTime} we provide experimental evidence that the proposed Hartigan algorithm converges faster than Lloyd's or EM, because the model is re-parametrized after each switch.

\section{Selection of the weight parameter}\label{sec:weight}

In this section we discuss the selection of the weight parameter $\beta$, trading between model complexity, model accuracy, and consistency with side information. Trivially, for $\beta=0$ we obtain pure model-based clustering, i.e., the CEC method while, for $\beta\to\infty$, model fitting becomes irrelevant and the obtained clustering is fully consistent with reference labeling. 

Our first theoretical result states that for $\beta=1$ the algorithm tends to create clusters that are fully consistent with the side-information. Before proceeding, we introduce the following definitions:

\begin{definition}\label{def:coarsening}
 Let $X$ be a data set and let $\Y=\{Y_1,\dots,Y_k\}$ be a partition of $X$. Let further $\Z=\{Z_1,\dots,Z_m\}$ be a partition of $X_\ell \subseteq X$. We say that $\Y$ is a \emph{coarsening} of $\Z$,  if for every $Z_j$ there exists a cluster $Y_i$ such that $Z_j \subseteq Y_i$. 
 
We say that the partition $\Y$ is \emph{proportional} to $\Z$, if the fraction of data points in each cluster equals the fraction of \emph{labeled} data points in this cluster, i.e., if $\frac{|Y_i|}{|X|} = \sum_{j=1}^ m\frac{|Z_j\cap Y_i|}{|X_\ell|}=\frac{|Y_i \cap X_\ell|}{|X_\ell|}$.
\end{definition}


Proportionality is required in the proofs below since it admits applying the chain rule of entropy to $H(\Z|\Y)$ even in the case where $X_\ell\subset X$. In other words, if $\Y$ is proportional to $\Z$, then (see~\ref{app:chainrule} for the proof):
$$
 H(\Z|\Y) + H(\Y) = H(\Z,\Y).
$$
Every coarsening of a proportional partition $\Y$ is proportional. Trivially, if $X_\ell=X$, then every partition $\Y$ is proportional to $\Z$. Note, however, that for finite data sets and if $X_\ell\subset X$, it may happen that there exists no partition $\Y$ proportional to the side information $\Z$ (e.g., if all but one data points are labeled). Nevertheless, the following theorems remain valid as guidelines for parameter selection.

Finally, note that consistency as in Definition~\ref{def:consistency} and coarsening as in Definition~\ref{def:coarsening} are, loosely speaking, opposites of each other. In fact, if $X_\ell=X$ and if $\Z$ is a partition of $X$, then $\Y$ is a coarsening of $\Z$ if and only of $\Z$ is consistent with $\Y$. Although we are interested in partitions $\Y$ consistent with $\Z$, we use the concept of a coarsening to derive results for parameter selection. Moreover, note that a partition $\Y$ can be both consistent with and a coarsening of the side information $\Z$. This is the case where every $Y_i$ contains exactly one $Z_j$ and every $Z_j$ is contained in exactly one $Y_i$ (i.e., $\Y$ has the same number of elements as $\Z$).

\begin{theorem}\label{thm:split}
Let $X \subset \R^N$ be a finite data set and $X_\ell \subseteq X$ be the set of labeled data points that is partitioned into $\Z=\{Z_1,\dots,Z_m\}$. Let $\Y=\{Y_1,\dots,Y_k\}$ be a proportional coarsening of $\Z$, and suppose that the sample covariance matrices $\Sigma_i$ of $Y_i$ are positive definite.

If $\tilde{\Y}=\{\tilde{Y}_1,\dots,\tilde{Y}_{k'}\}$ is a coarsening of $\Y$, then
\begin{equation}\label{eq:inquality}
\E_1(\tilde{\Y}) \geq  \E_1( \Y).
\end{equation}
\end{theorem}
\begin{proof}
See~\ref{app:proofsplit}.
\end{proof}




An immediate consequence of Theorem~\ref{thm:split} is that, for $\beta=1$, {\cmc} tends to put elements with different labels in different clusters. Note, however, that there might be partitions $\Y$ that are consistent with $\Z$ that have an even lower cost \eqref{eq:cost}: Since every consistent partition $\Y$ satisfies $H(\Z|\Y)=0$, any further refinement of $\Y$ reduces the cost whenever the cost for model complexity, $H(\Y)$, is outweighed by the modeling inaccuracy $\sum_{i=1}^k \frac{|Y_i|}{|X|} \crossfi$.

Theorem~\ref{thm:split} does not assume that the side information induces a partition of $X$ that fits our intuition of clusters: the $Z_j$  need not be a connected set, but could result from, say, a random labeling of the data set $X$. Then, for $\beta=1$, splitting $X$ into clusters that are consistent with the side information will be at least as good as creating a single cluster. Interestingly, if the labeling is completely random, any $\beta<1$ will prevent dividing the data set into clusters: 

\begin{remark}\label{rem:randomlabels} 
Suppose a completely random labeling for the setting of Theorem~\ref{thm:split}. More precisely, we  assume that the set of labeled data $X_\ell$ is divided into $\Z=\{Z_1,\ldots,Z_m\}$ and that the partition $\Y$ is a proportional coarsening of $\Z$. If sufficiently many data points are labeled, we may assume that the sample covariance matrix $\Sigma_{Y_i}$ of $Y_i$ is close to the covariance matrix $\Sigma_X$ of $X$, i.e. $\Sigma_{Y_i}\approx \Sigma_X$. For any coarsening $\tilde{\Y}$ of $\Y$, the cross-entropies for $\Y$ and $\tilde{\Y}$ are approximately equal:
$$
 \sum_i \frac{|Y_i|}{|X|} \crossfi \approx \sum_j \frac{|\tilde{Y}_j|}{|X|} H(\N(\mu_{\tilde{Y}_j},\Sigma_{\tilde{Y}_j})) 
 \approx H(\N(\mu_X,\Sigma_X))
$$
because the sample covariance matrices of $\tilde{Y} \in \tilde{\Y}$ are also close to $\Sigma_X$.

If we compare the remaining parts of cost function \eqref{eq:cost}, then with $\beta<1$ we obtain:
 \begin{multline}\label{eq:betacost}
  H(\Y) + \beta H(\Z|\Y) = (1-\beta)H(\Y) + \beta (H(\Y)+H(\Z|\Y)) \\= (1-\beta)H(\Y) + \beta (H(\tilde\Y)+H(\Z|\tilde\Y))
  > H(\tilde\Y)+\beta H(\Z|\tilde\Y).
 \end{multline}
The last inequality follows from the fact that $H(\Y)>H(\tilde\Y)$. Therefore, {\cmc} with $\beta< 1$ is robust on random labeling.
\end{remark}


Our second result is a critical threshold $\beta_0$, above which splitting a given cluster $\tilde{Y}_1$ into smaller clusters $Y_1,\ldots,Y_l$ reduces the cost. This threshold $\beta_0$ depends on the data set and on the side information. For example, as Remark~\ref{rem:randomlabels} shows, for a completely random labeling we get $\beta_0=1$. To derive the threshold in the general case, we combine the proof of Theorem~\ref{thm:split} with~\eqref{eq:betacost}:

\begin{theorem} \label{thm:limitBeta}
Let $X \subset \R^N$ be a finite data set and $X_\ell \subseteq X$ be the set of labeled data points that is partitioned into $\Z=\{Z_1,\dots,Z_m\}$. Let $\Y=\{Y_1,\dots,Y_k\}$ be a proportional coarsening of $\Z$, and suppose that the sample covariance matrices $\Sigma_i$ of $Y_i$ are positive definite. Suppose that $\tilde{\Y}=\{Y_1,\dots,Y_{k'-1},(Y_{k'}\cup\cdots\cup Y_k)\}$, for $1< k'< k$, is a coarsening of $\Y$, and let $\mu$ and $\Sigma$ be the sample mean vector and sample covariance matrix of $Y_{k'}\cup\cdots\cup Y_k$. Let $q_i=p_i$ for $i=1,\dots,k'-1$ and $q_k=\sum_{i=k'}^k p_i$.
We put 
\begin{equation}\label{eq:betaThresh}
 \beta_0 = 1+\frac{\sum_{i=k'}^k \frac{p_i}{2q_{k'}} \ln\left(\frac{\det\Sigma_i}{\det\Sigma}\right) }{H\left(\frac{p_{k'}}{q_{k'}},\dots,\frac{p_k}{q_{k'}}\right)}.
\end{equation}

If $\beta \geq \beta_0$, then 
$$
 \E_\beta ( \tilde{\Y}) \geq \E_\beta (\Y).
$$
\end{theorem}

\begin{proof}
See \ref{app:proof:lastthm}.
\end{proof}

We now evaluate a practically relevant instance of the above theorem, where the data follows a Gaussian distribution and the partition-level side information is ``reasonable'':
\begin{example} \label{thm:beta0Gauss}
Let $X \subset \R$ be a data set generated by a one-dimensional Gaussian distribution $f=\N(\mu,\sigma^2)$, and suppose that the data set is large enough such that the sample mean $\mu_X$ and sample variance $\sigma_X^2$ are close to $\mu$ and $\sigma^2$, respectively. A classical unsupervised model-based clustering technique, such as CEC or GMM, terminates with a single cluster.

Now suppose that $Z_1 \subset (-\infty,\mu)$ and $Z_2 \subset [\mu,+\infty)$ are equally-sized sets, which suggests that $\Y=\{Y_1,Y_2\}=\{(-\infty,\mu) \cap X, [\mu,+\infty) \cap X\}$ is the expected clustering. Consequently, on one hand, the data distribution indicates that a single cluster should be created while, on the other hand, the side information suggests splitting the data set into two clusters. At the threshold $\beta_0$  these two conflicting goals are balanced, while for $\beta > \beta_0$ a consistent clustering is obtained.

To calculate the critical value $\beta_0$, let $\Y=\{Y_1,Y_2\}$ be proportional to $\Z$, and let $f_i = \N(\mu_{Y_i},\sigma_{Y_i}^2)$ be the optimal fit for cluster $Y_i$. Since the data in $Y_i$ can be well approximated by a truncated Gaussian distribution, we can calculate:
$$
\sigma_{Y_1}^2 \approx \sigma_{Y_2}^2 \approx \sigma^2 \left(1-\frac{2}{\pi}\right).
$$
Making use of the previous theorem, $\E_{\beta}(\{X\}) = \E_{\beta}(\Y)$ for 
\begin{equation}\label{eq:beta0}
\beta=\beta_0 \approx 1+ \frac{\ln \sqrt{1-\frac{2}{\pi}}}{H(\frac{1}{2}, \frac{1}{2})}\approx 0.269.
\end{equation}

Continuing this example, in some cases the side information might be noisy, i.e., data points are labeled wrongly. Consider the labeling $\Z$ that satisfies $Z_1 \subset (-\infty, \mu+c)$ and $Z_2\subset [\mu-c,+\infty)$, for some $c>0$. In other words, the human experts did not agree on the labeling at the boundary between the clusters. If we choose $\Y$ proportional to this noisy side information $\Z$, then one has reason to suppose that the sample variances of $Y_1$ and $Y_2$ are larger than in the noiseless case, hence leading to a larger threshold $\beta_0$ according to Theorem~\ref{thm:beta0Gauss}. Setting $\beta$ to a value only slightly higher than the threshold $\beta_0$ for the noiseless case thus ensures a partition $\Y$ that is consistent with the noiseless labeling, but that is robust to noise. In summary, one should choose $\beta$ large (i.e., close to 1), if one believes that the side information is correct, but small if one has to expect noisy side information. 
\end{example}

\section{Experiments}\label{sec:experiments}

We evaluated our method in classical semi-supervised clustering tasks on examples retrieved from the UCI repository \cite{asuncion2007uci} and compared its results to state-of-the-art semi-supervised clustering methods. We evaluated performance in the case of only few classes being present in the partition-level side information and for noisy labeling, and investigated the influence of the parameter $\beta$ on the clustering results. We furthermore applied {\cmc} on a data set of chemical compounds \cite{warszycki2013linear} to discover subgroups based on partition-level side information derived from the top of a cluster hierarchy and illustrate its performance in an image segmentation task.

\subsection{Experimental setting}\label{sec:experiment:setting}

We considered five related semi-supervised clustering methods for comparison. The first is a classical semi-supervised classification method that is based on fitting a GMM to the data set taking partial labeling into account. Since it is a classification method, it only works if all true classes are present in the categorization $\Z$. We used the R implementation Rmixmod~\cite{lebret2015rmixmod} with default settings; we refer to this method as ``mixmod''.

The second method incorporates pairwise constraints as side information for a GMM-based clustering technique~\cite{shental2004computing}. To transfer the partition-level side information to pairwise constraints, we went over all pairs of labeled data points in $X_\ell$ and generated a must-link constraint if they were in the same, and a cannot-link constraint if they were in different categories. We used the implementation from one of the authors' website\footnote{\url{http://www.scharp.org/thertz/code.html}} and refer to this method as ``c-GMM''. We ran c-GMM in MultiCov mode, i.e. every cluster was characterized by its own covariance matrix.

We also applied an extension of k-means to accept partition level-side information~\cite{liu2015clustering}. The method requires setting a weight parameter $\lambda$, that places weight on the features derived from the side information. The authors suggested $\lambda=100$, but we found that the method performs more stable for $\lambda=100\cdot \mathrm{tr}(\Sigma)$, i.e., for $\lambda$ being proportional to the trace of the sample covariance matrix of the data set $X$. We refer to this method as ``k-means''.

Moreover, we considered a semi-supervised variant of fuzzy c-means~\cite{pedrycz1997fuzzy, pedrycz2008fuzzy}, which incorporates partition-level side information. We used the Euclidean distance, set the fuzzifier parameter to 2, and chose a trade-off parameter $\alpha = \frac{|X|}{|X_\ell|}$ as suggested by the authors. To obtain a ``hard'' clustering $\Y$ from a fuzzy partition we assigned every point to its most probable cluster. This technique will be referred to as ``fc-means''.

Finally, we used a semi-supervised version of spectral clustering~\cite{qian2016affinity} (referred to as ``spec''), which was claimed to achieve state-of-the-art performance among spectral algorithms. The method accepts pairwise constraints and operates on the affinity (similarity) matrix of the data set. The authors of~\cite{qian2016affinity} suggested setting the similarity between data points $x_i$ and $x_j$ to $e^{-\Vert x_i - x_j\Vert^2/2\rho^2}$, where $\Vert\cdot\Vert$ is the Euclidean distance and where $\rho > 0$ is called affinity parameter. In order to account for different variances in different dimensions, we used
\begin{equation}
 e^{-\sum_{\ell=1}^N \frac{\vert x_i^{(\ell)}-x_j^{(\ell)}\vert^2}{2\rho^2\sigma^2_{(l)}}},
\end{equation}
where $x_i^{(\ell)}$ is the value of the $\ell$-th coordinate of $x_i$ and where $\sigma^2_{(\ell)}$ is the variance of the $\ell$-th coordinate of $X$. The method can be tuned with two parameters: affinity parameter $\rho$ and trade-off factor $\eta$. The authors suggest to find the best possible combination of these parameters using a grid-search strategy. Since we did not allow for tuning any parameters of other methods (including $\beta$ in {\cmc}), for a fair comparison we decided to fix these two parameters. Specifically, we put $\eta = 0.7$ analyzing the results reported in~\cite{qian2016affinity}. We moreover set $\rho = 1$ based on the fact that the Euclidean distances are already normalized according to the variances of the respective dimensions and since~\cite{qian2016affinity} reports little influence of the selection of $\rho$. We generated must-link and cannot-link constraints as we did for c-GMM; moreover, the entries of the affinity matrix were set to one for must-link, and to zero for cannot-link constraints. 

Since {\cmc} automatically determines an appropriate number of clusters by removing clusters with too few data points, we initialized {\cmc} with twice the correct number of clusters. In contrast, other methods were run with the correct numbers of clusters. In a semi-supervised clustering task with correct labels from all classes, the competing methods can thus be expected to perform better than {\cmc}.

To better illustrate the effect of the weight parameter $\beta$, we used two parameterization of {\cmc}, using $\beta=1$ and $\beta=\beta_0 \approx 0.269$ given by \eqref{eq:beta0}. We refer to these two variants as $\cmc_1$ and $\cmc_0$, respectively. 

\begin{table}[t]\footnotesize
\caption{Summary of UCI datasets used in the experiments.}
\label{tab:data1}	
\medskip
\setlength{\arrayrulewidth}{0.1mm}
\setlength{\tabcolsep}{3pt}
\centering
\begin{tabular}{lccc}
\bf Data set & \bf \# Instances & \bf \# Features & \bf \# Classes \\ \hline
Ecoli$^+$ & 327 & 5 & 5\\
Glass & 214 & 9 & 6\\
Iris & 150 & 4 & 3\\
Segmentation$^+$ & 210 & 5 & 7\\
User Modeling & 403 & 5 & 4\\
Vertebral & 310 & 6 & 3\\
Wine & 178 & 13 & 3\\
\hline
\end{tabular}
\\ \vspace{0.1in}
\footnotesize{$^+$: PCA was used to reduce a dimensionality of the data set and remove dependent attributes}
\end{table}

The similarity between the obtained clusterings and the ground truth was evaluated using Normalized Mutual Information (NMI) \cite{ana2003robust}. For a reference grouping $\X$ and a clustering $\Y$ it is defined by
$$
\mathrm{NMI}(\Y, \X) = \frac{2I(\Y; \X)}{H(\Y) + H(\X)}.
$$ 
Since $I(\Y;\X) \leq \min\{H(\Y), H(\X)\}$, NMI is bounded from above by 1, which is attained for identical partitions. If $\Y$ and $\X$ contain different numbers of clusters, then NMI is always below 1.

\subsection{Semi-supervised clustering}\label{sec:experiment:semi}

We evaluated the proposed method in a classical semi-supervised clustering task, in which we aim to recover a reference partition based on a small sample of labeled data. 

We used seven UCI data sets, which are summarized in Table~\ref{tab:data1}. The partition-level side information was generated by choosing 0\%, 10\%, 20\%, and 30\% of the data points and labeling them according to their class. To remove effects from random initializations, we generated 10 different samples of side information for each percentage and averaged the resulting NMI values. 

\begin{figure}
\begin{center}
\includegraphics[width=3.7in]{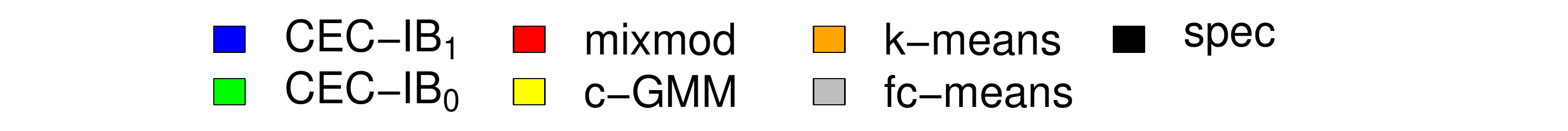}

\subfigure[Ecoli]{\includegraphics[width=1.7in]{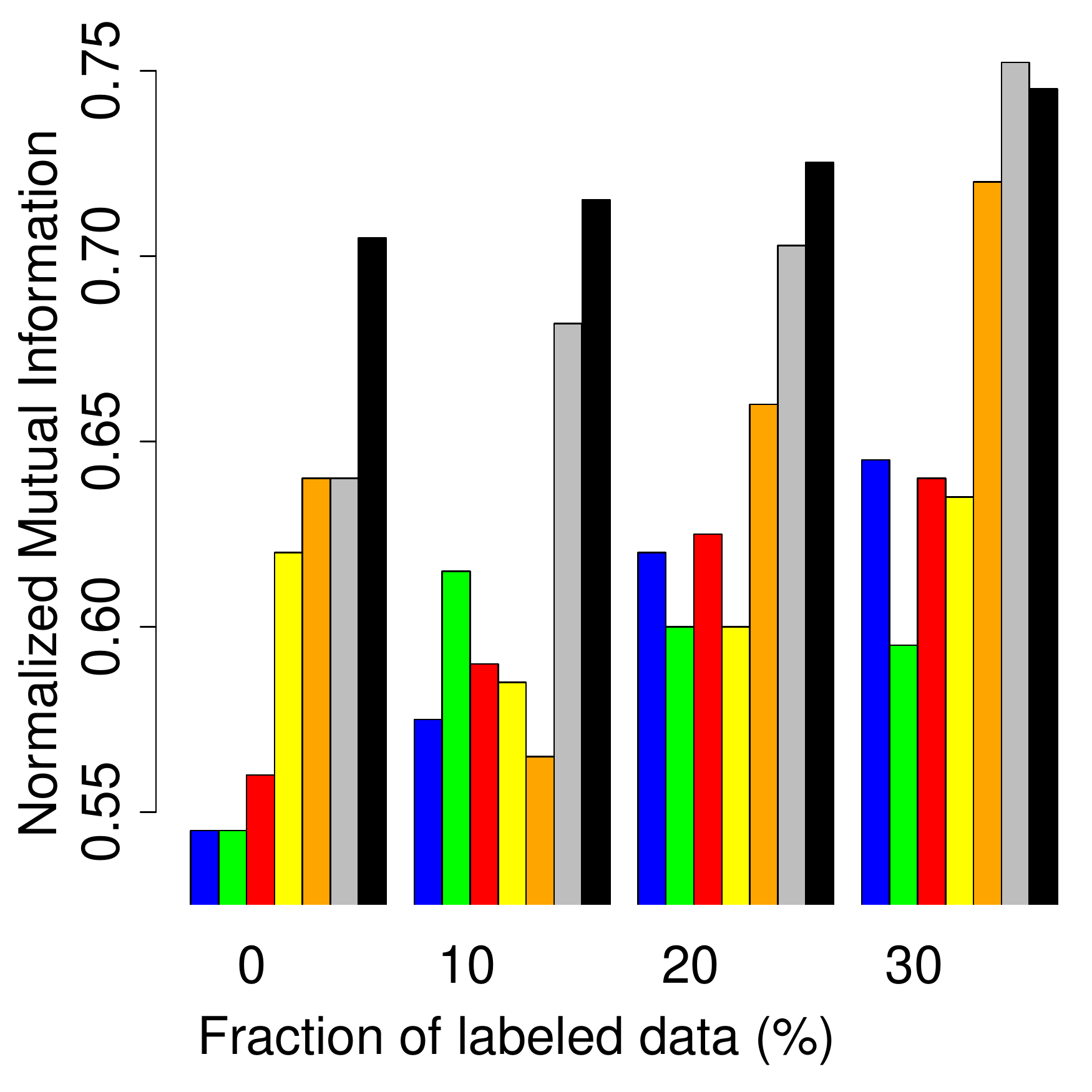}}
\subfigure[Glass]{\includegraphics[width=1.7in]{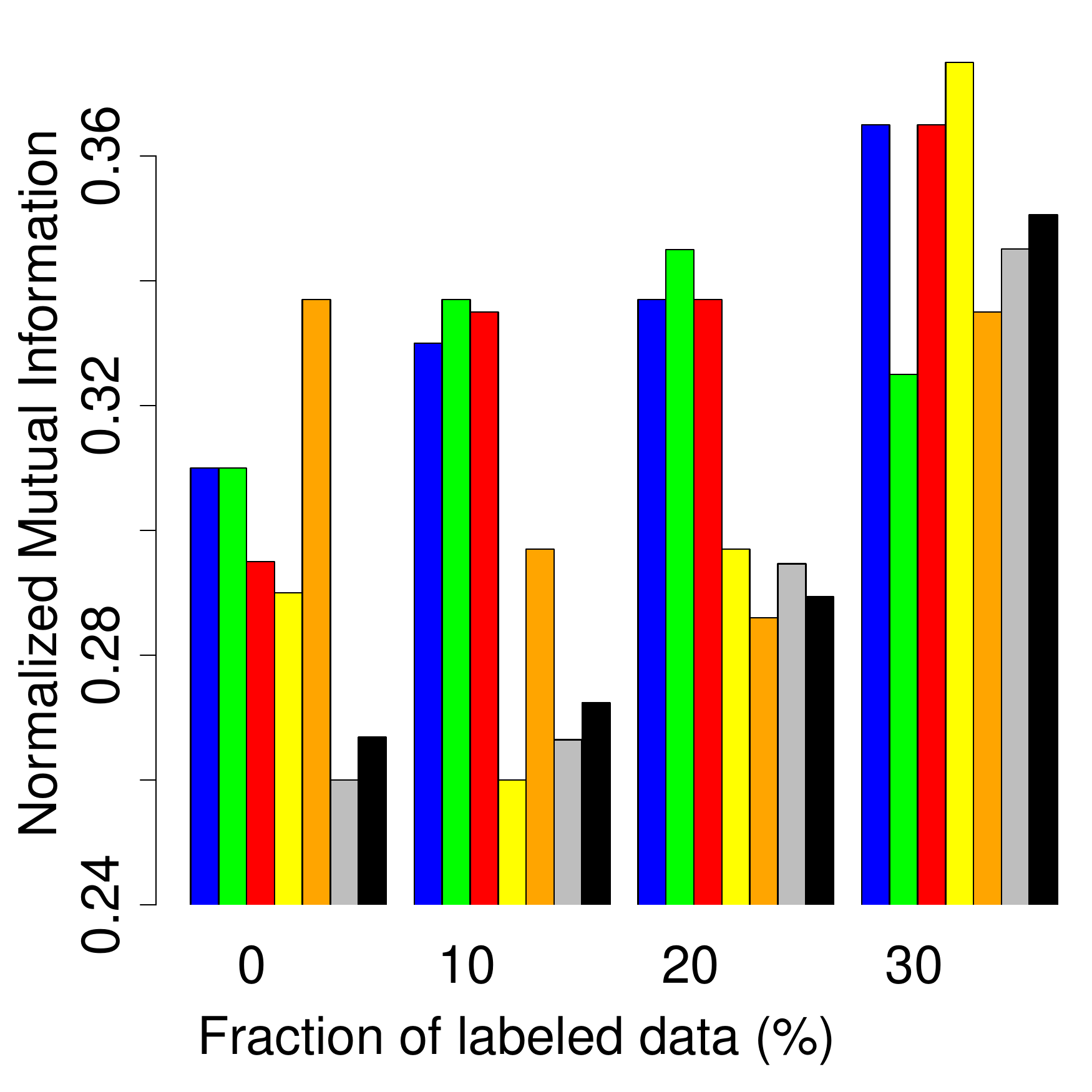}}
\subfigure[Segmentation]{\includegraphics[width=1.7in]{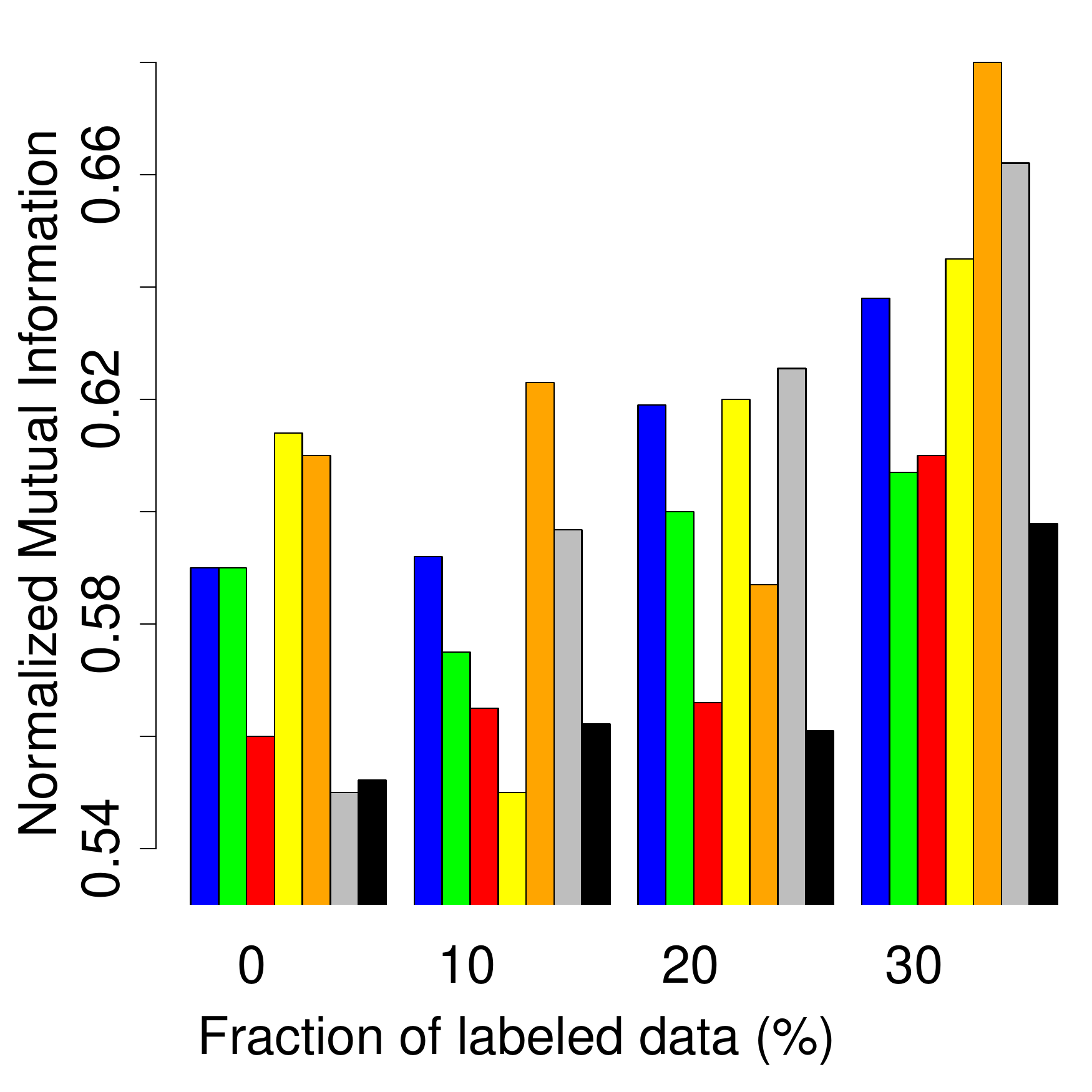}}
\subfigure[User Modeling]{\includegraphics[width=1.7in]{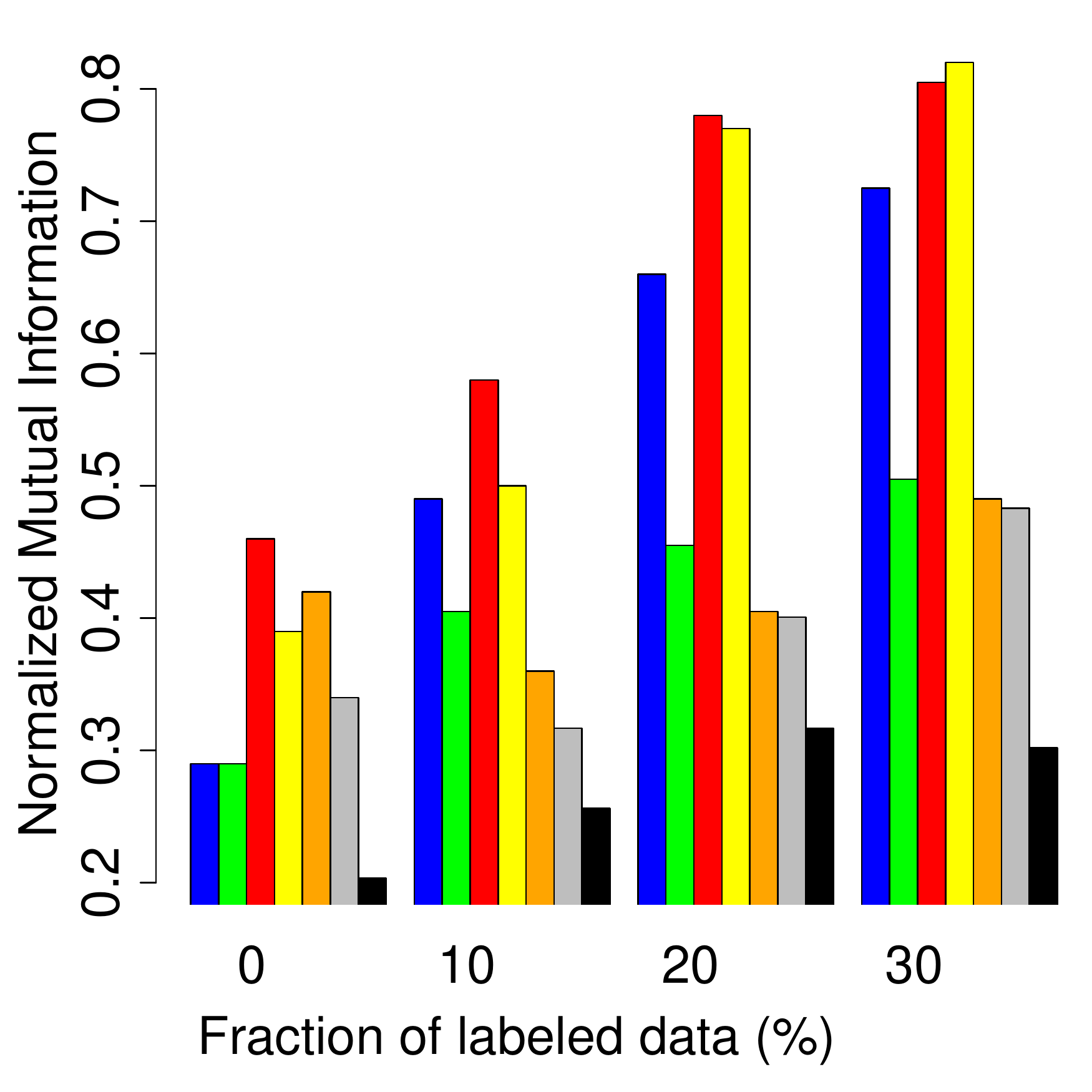}}
\subfigure[Vertebral]{\includegraphics[width=1.7in]{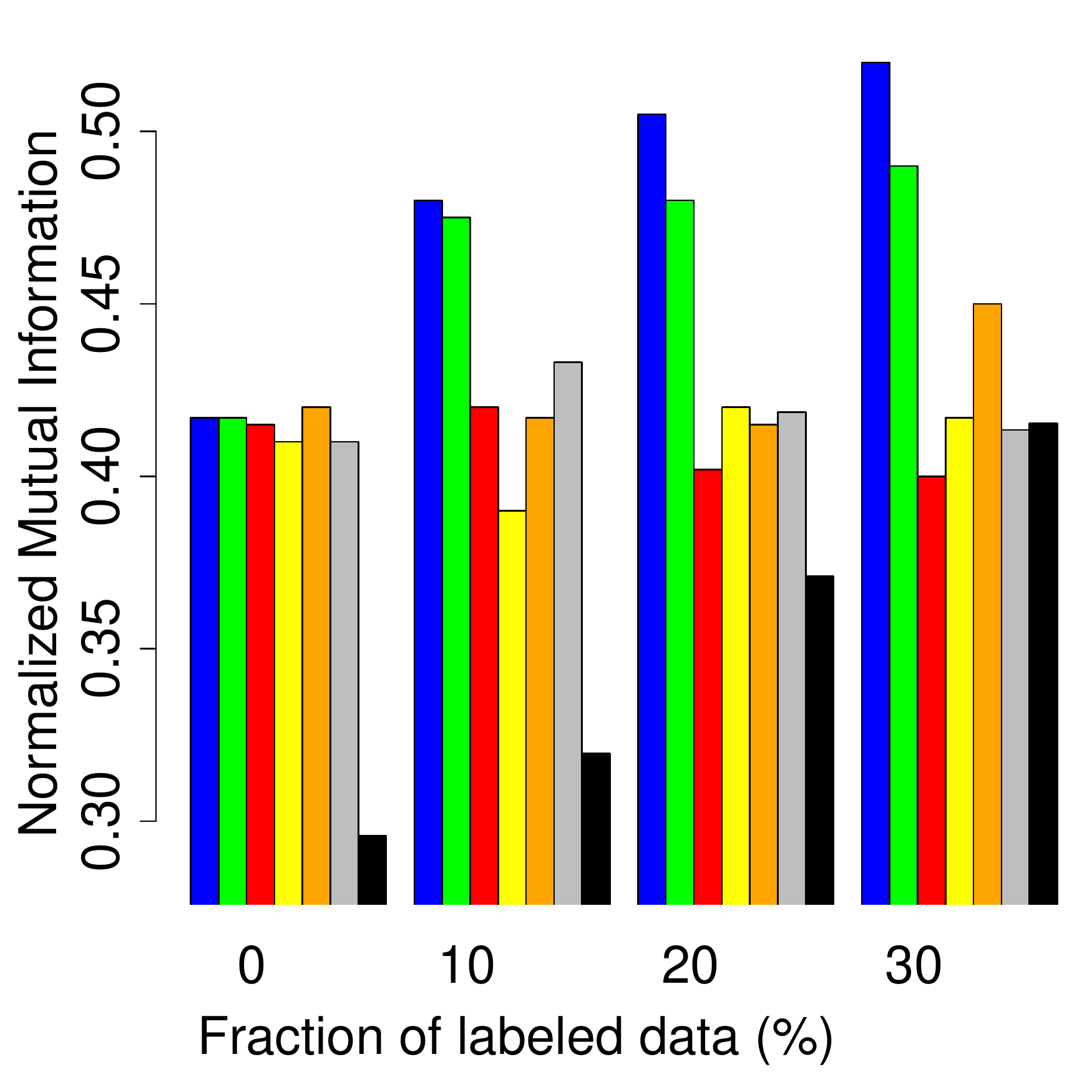}}
\subfigure[Wine]{\includegraphics[width=1.7in]{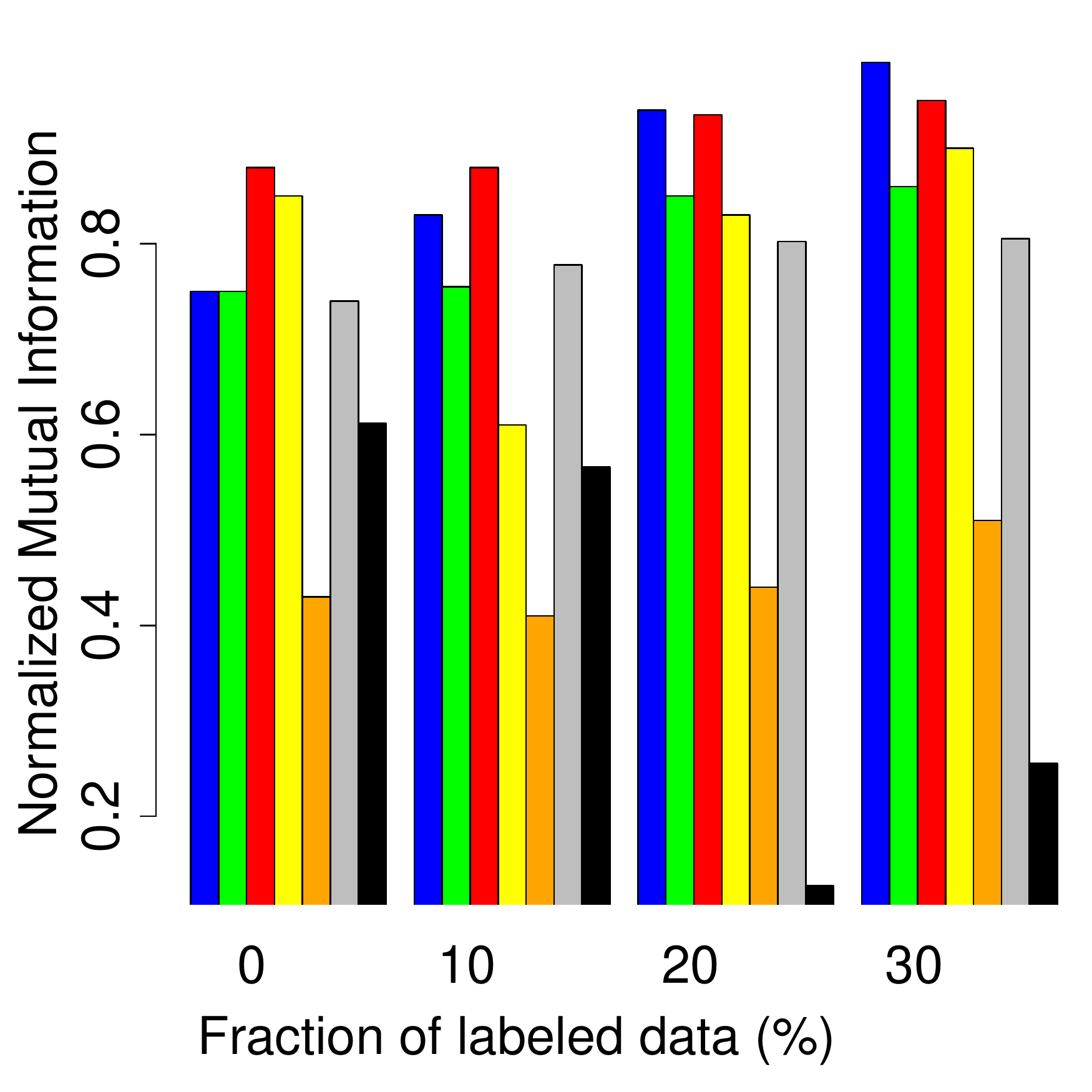}}
\subfigure[Iris]{\includegraphics[width=1.7in]{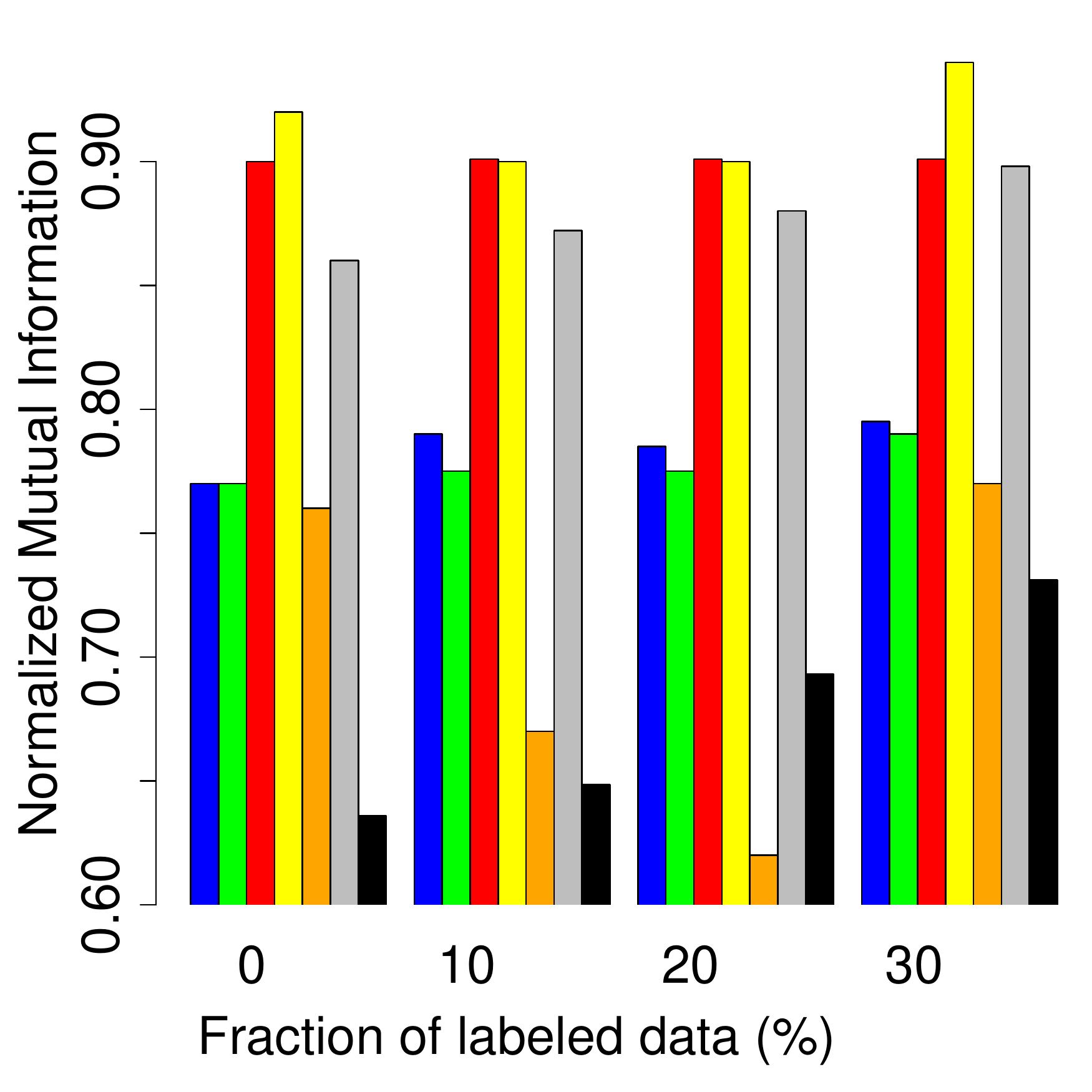}}
\subfigure[Iris$^+$]{\label{fig:iris3}\includegraphics[width=1.7in]{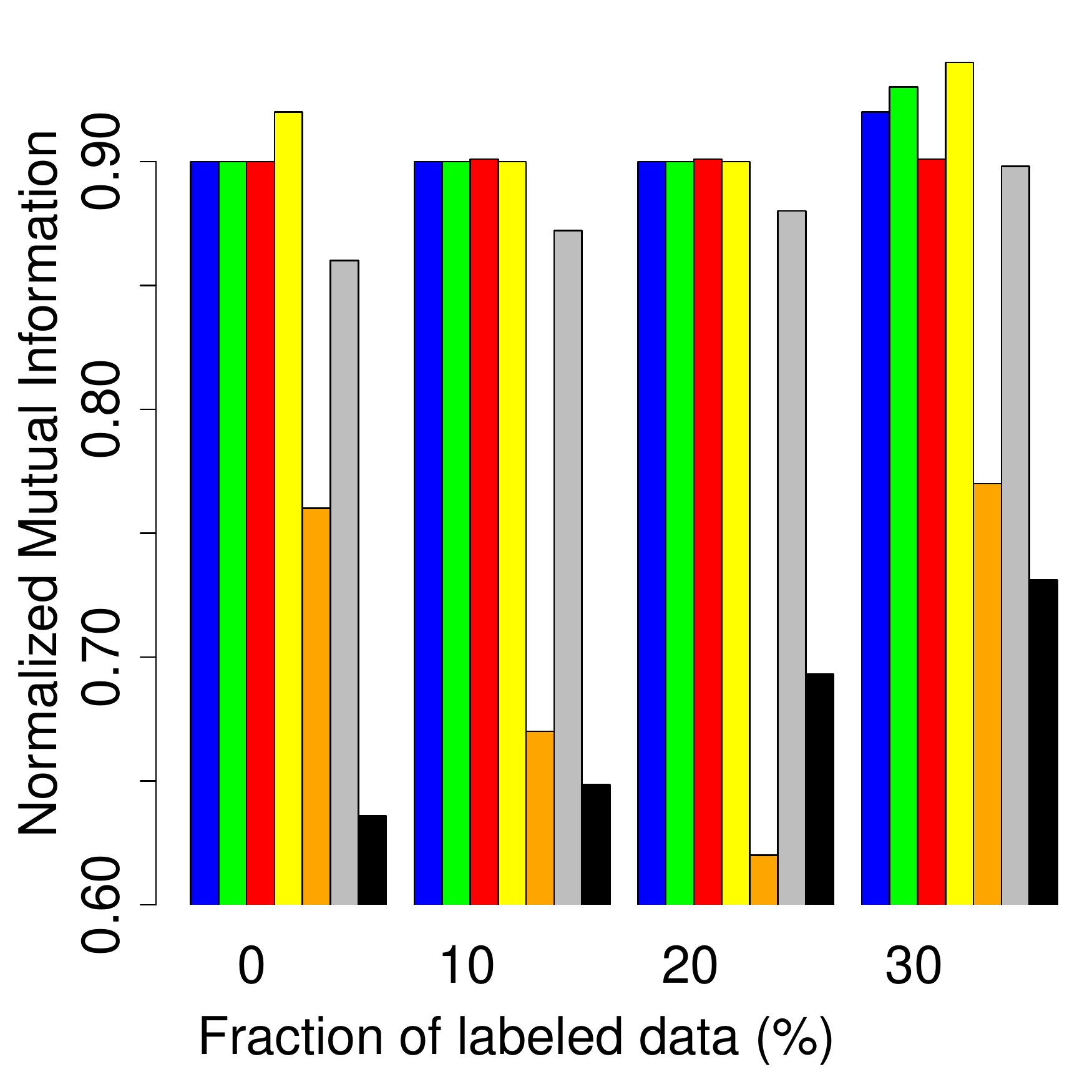}}
\end{center}
\caption{Normalized Mutual Information of examined methods evaluated on UCI datasets. {\cmc} was initialized with twice the true number of clusters while other methods used the correct number of clusters.}
\label{fig:uci1}
\footnotesize{$^+$ {\cmc} was initialized with the correct number of clusters}
\end{figure}

The clustering results presented in Figure \ref{fig:uci1} show that $\cmc_1$ usually achieved a higher NMI than $\cmc_0$: Since the partition-level side information is noise-free, i.e., agrees with the reference grouping, a larger weight parameter $\beta$ leads to better performance. In general, $\cmc_1$ produced results similar to the two other GMM-based techniques, c-GMM and mixmod. Notable differences can be observed on Vertebral dataset, where {\cmc} performed significantly better, and on Iris and User Modeling, where the competing methods gave higher NMI. This is most likely caused by the fact that {\cmc} failed to determine the correct number of clusters (see Table \ref{tab:clusters}), while the GMM implementations were given this correct number of clusters as side information. As it can be seen in Table \ref{tab:clusters}, in all other cases, {\cmc} terminated with a number of clusters very close to the true value. Initializing {\cmc} with the correct number of clusters for the Iris data set, we get results that are comparable to those of mixmod and c-GMM (see Figure \ref{fig:iris3}).

\begin{table}[t]\footnotesize
\caption{Number of clusters returned by {\cmc} for a given amount of labeled data.}
\label{tab:clusters}	
\medskip
\setlength{\arrayrulewidth}{0.1mm}
\setlength{\tabcolsep}{3pt}
\centering
\begin{tabular}{lccccc}
\bf Data set & \bf \# Classes & \bf 0\% & \bf 10\% & \bf 20\% & \bf 30\% \\ \hline
Ecoli & 5&7 & 6 & 6 & 6 \\
Glass & 6& 5 & 6 & 6 & 6 \\
Iris & 3& 5 & 5 & 5 & 5 \\
Segmentation & 7&8 & 8 & 7 & 8 \\
User & 4&7 & 6 & 6 & 6 \\
Vertebral & 3&4 & 4 & 4 & 4 \\
Wine & 3&3 & 3 & 3 & 3 \\
\hline
\end{tabular}
\end{table}


Observe that k-means gave slightly lower NMI than fc-means. Nevertheless, both algorithms performed worse than the GMM-based methods, except for the Ecoli and Segmentation data sets. The difference in the results can be explained by the fact that fc-means and k-means are distance-based methods and therefore perform differently from model-based approaches. Although the performance of spec usually increases with more labeled examples, its results are worse than the other methods.

\subsection{Few labeled classes} \label{sec:experiment:few}

\begin{figure}
\begin{center}
\includegraphics[width=3.7in]{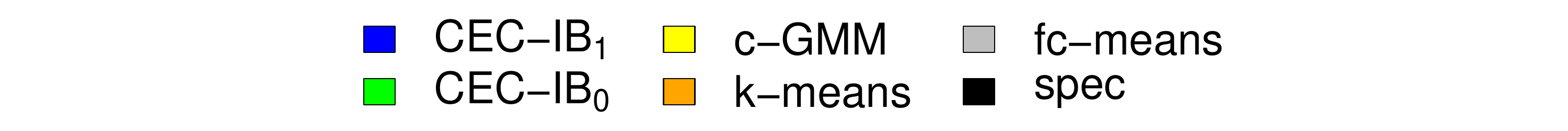}

\subfigure[Ecoli]{\includegraphics[width=1.7in]{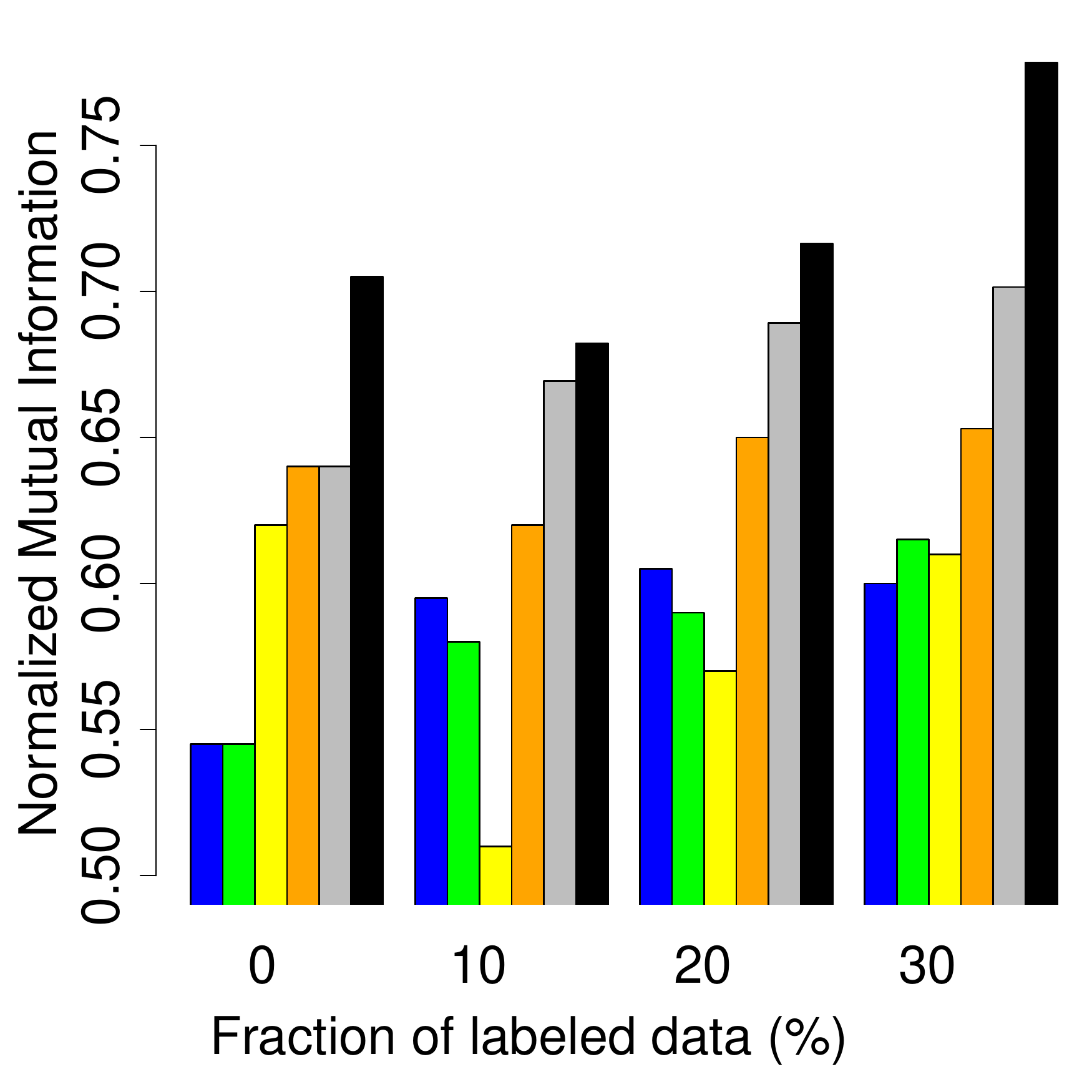}}
\subfigure[Glass]{\includegraphics[width=1.7in]{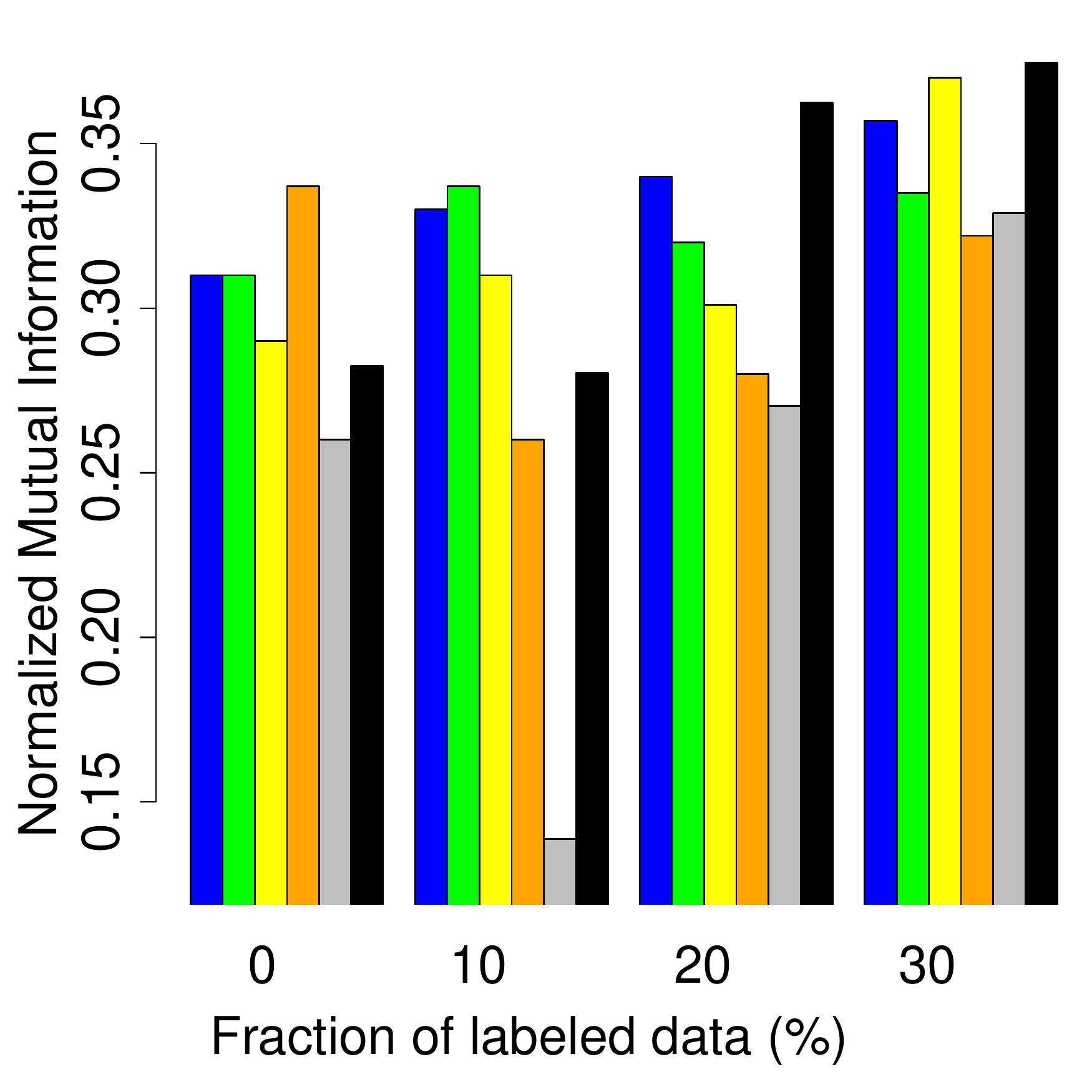}}
\subfigure[Segmentation]{\includegraphics[width=1.7in]{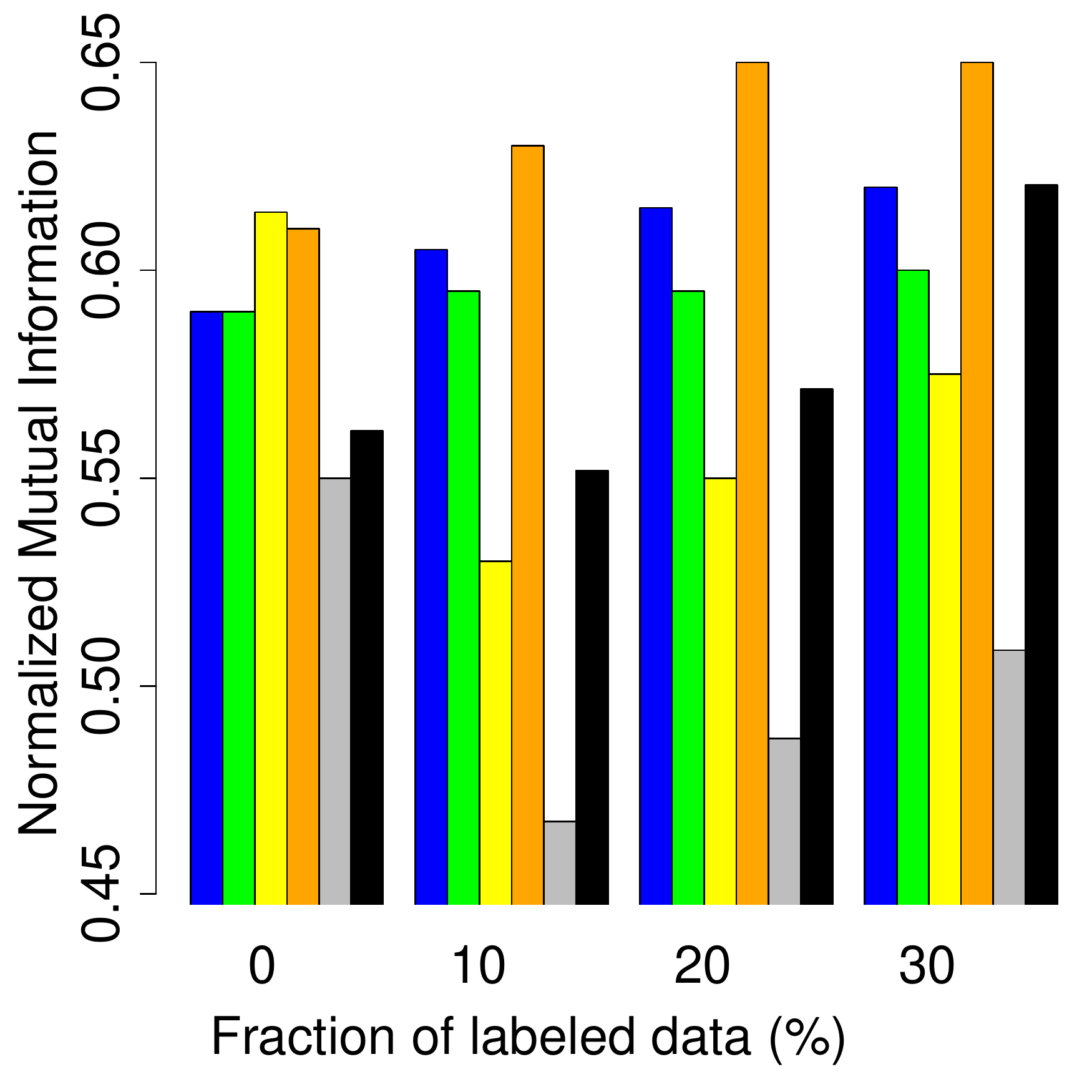}}
\subfigure[User]{\includegraphics[width=1.7in]{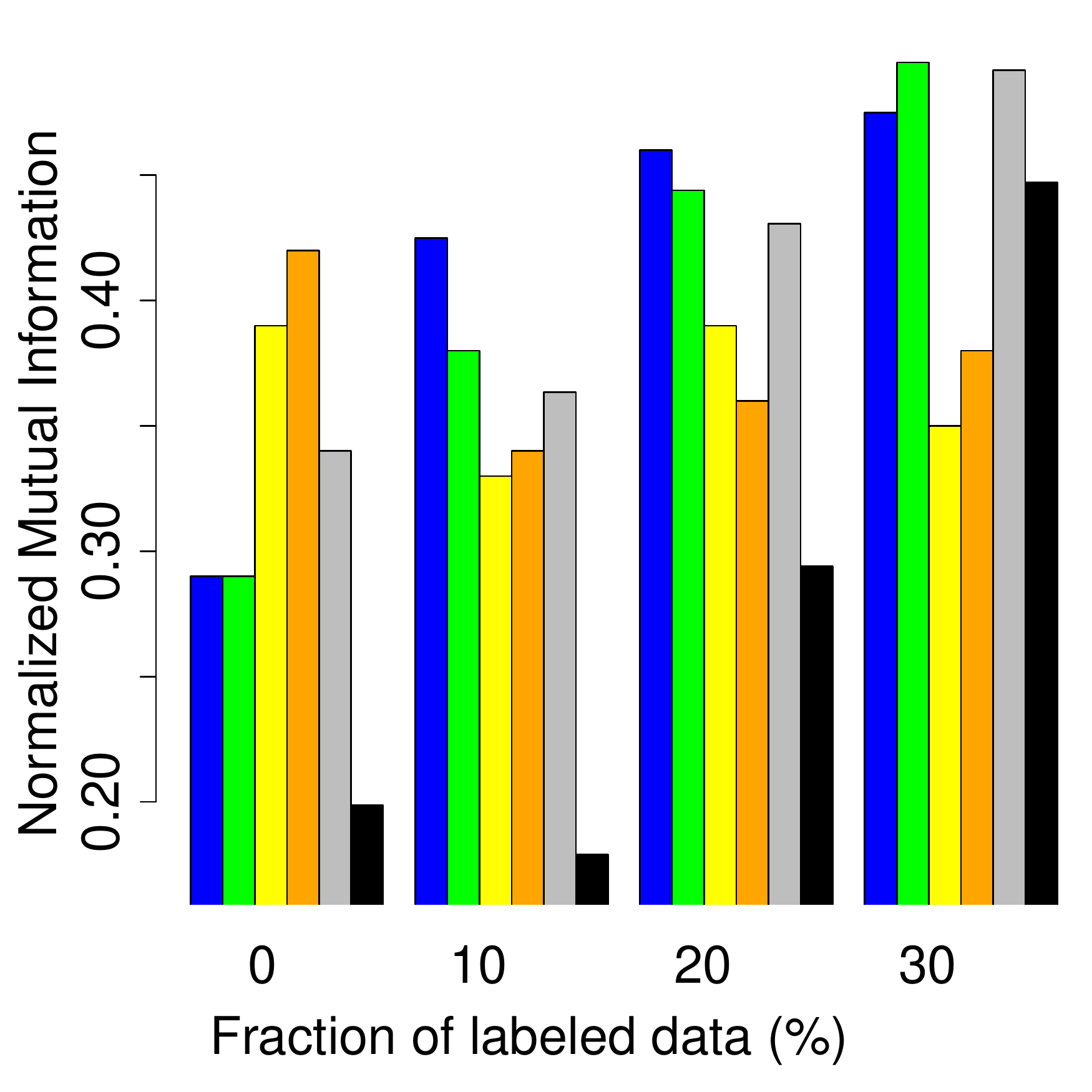}}
\subfigure[Vertebral]{\includegraphics[width=1.7in]{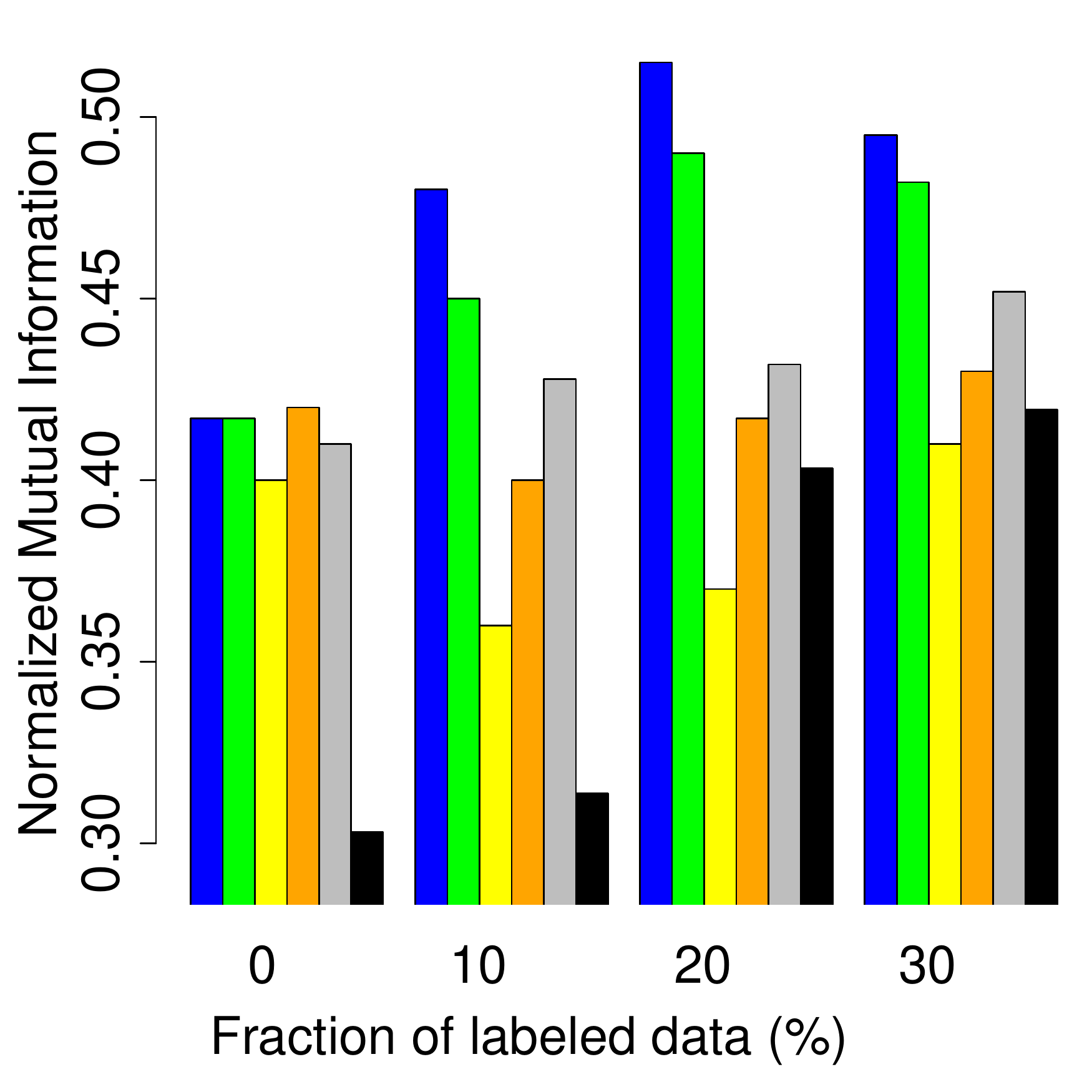}}
\subfigure[Wine]{\includegraphics[width=1.7in]{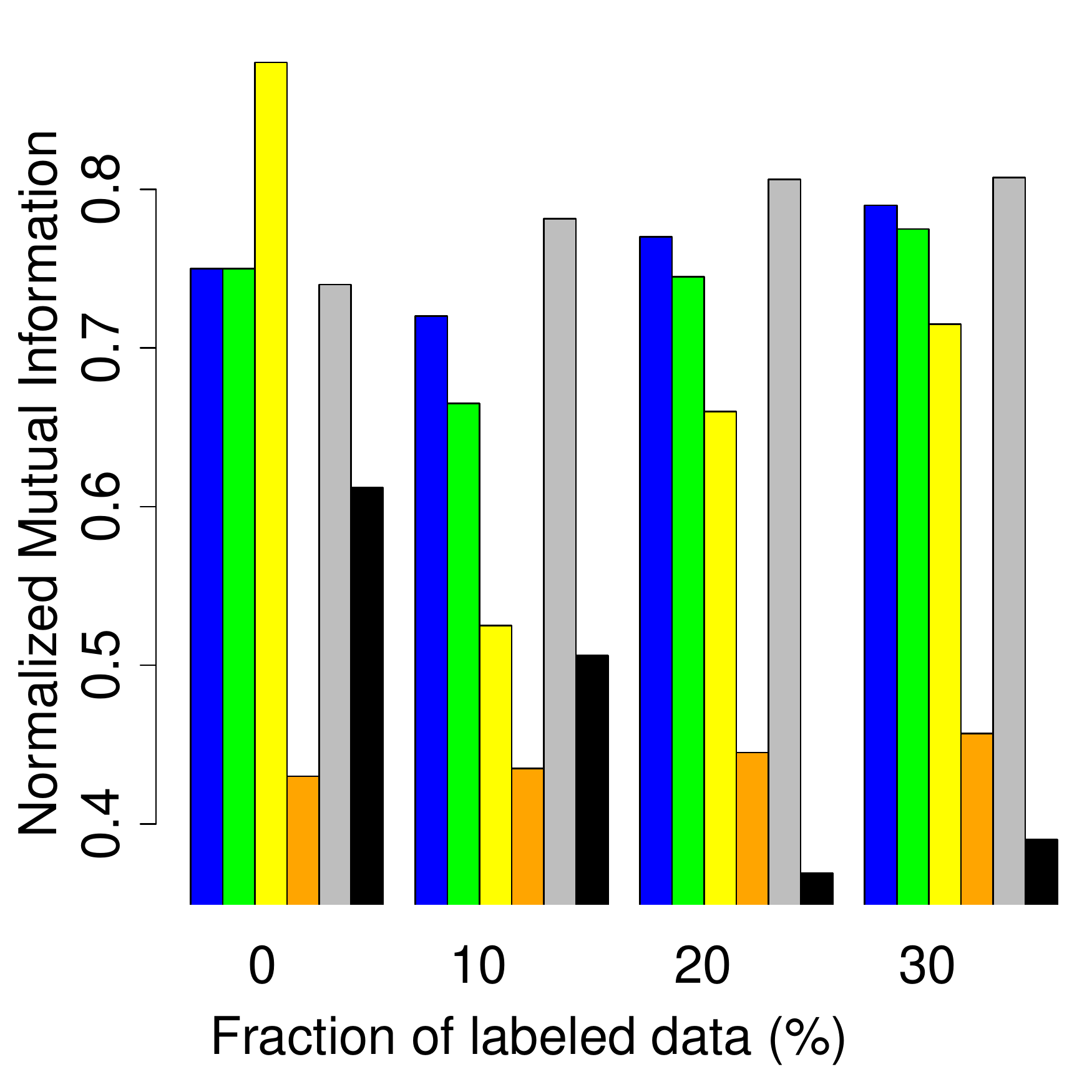}}
\subfigure[Iris]{\includegraphics[width=1.7in]{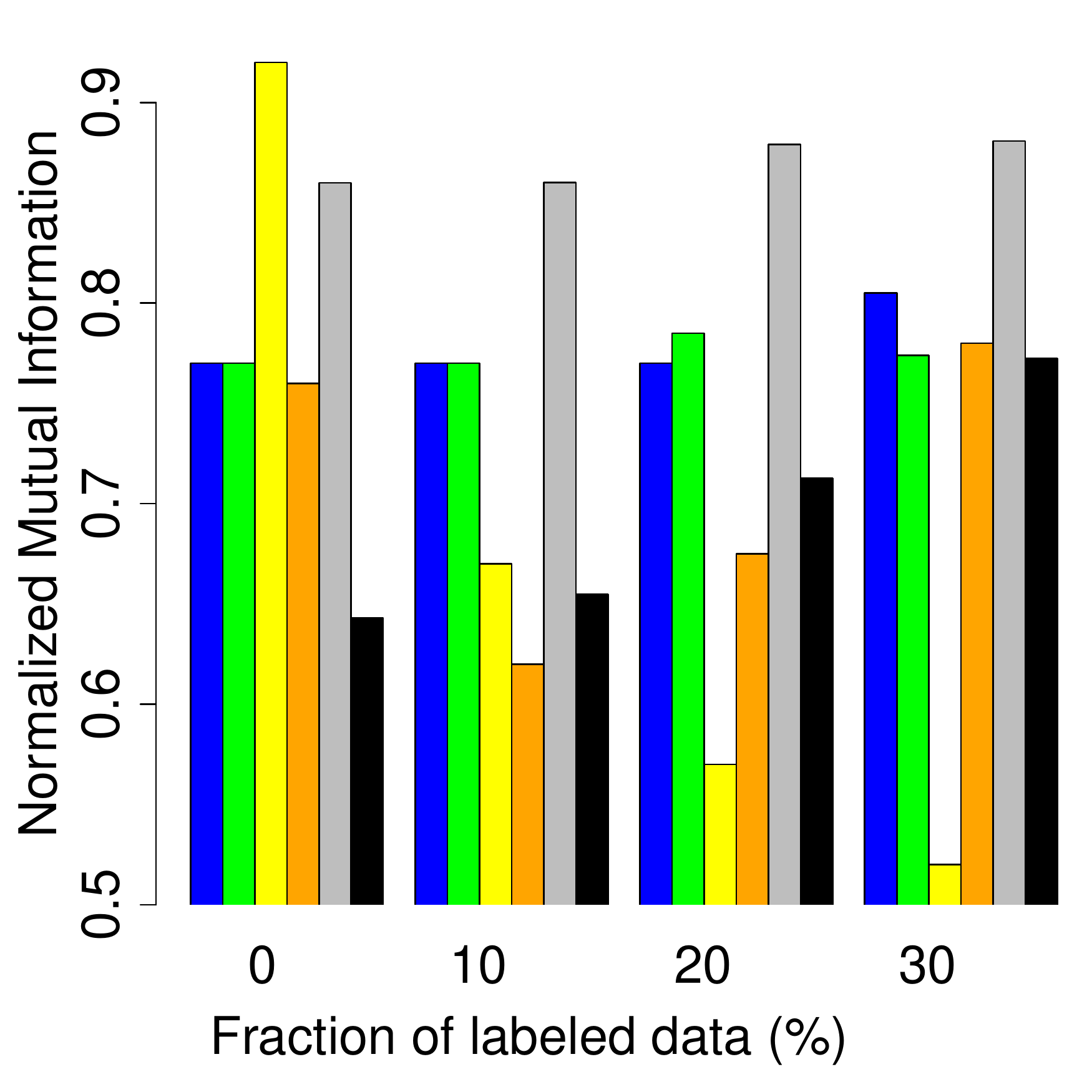}}
\end{center}
\caption{Normalized Mutual Information of examined methods evaluated on UCI datasets when the partition-level side information covered only two classes. }
\label{fig:fewClasses}
\end{figure}

In a (semi-)supervised classification task, the model learns classes from a set of labeled data and applies this knowledge to unlabeled data points. More specifically, the classifier cannot assign class labels that were not present in a training set. In contrast, clustering with partition-level side information, can detect clusters within a labeled category or within the set of unlabeled data points.

In this section, we apply {\cmc} to a data set for which the partition-level side information contains labels of only two classes from the reference grouping. As before we considered 0\%, 10\%, 20\% and 30\% of labeled data. For each of the 10 runs we randomly selected two classes from a reference grouping that covered at least 30\% of data in total and generated the partition-level side information from these two categories (the same classes were used for all percentages of side information). It was not possible to run mixmod in this case because this package does not allow to use a number of clusters different from the categories given in the side information.

Figure \ref{fig:fewClasses} shows that {\cmc} was able to consistently improve its clustering performance with an increasing size of the labeled data set\footnote{Although the results for 0\% of labeled data should be identical with the ones reported in Section \ref{sec:experiment:semi}, some minor differences might follow from a random initialization of the methods, see Section \ref{sec:model:algo}.}. Surprisingly, c-GMM sometimes dropped in performance when adding side information. This effect was already visible in Figure~\ref{fig:uci1}, but seems to be more pronounced here. While a deeper analysis of this effect is out of scope of this work, we believe that it is due to the simplification made in~\cite{shental2004computing} to facilitate applying a generalized EM scheme. This simplification is valid if pairs of points with cannot-link constraints are disjoint, an assumption that is clearly violated by the way we generate cannot-link constraints (see Section~\ref{sec:experiment:setting}).

Contrary to c-GMM, the results of fc-means and k-means were far more stable. In most cases both algorithms increased their performance having access to more labeled data. Interestingly, spec performed in general better when only two classes were labeled than in the previous experiment where all classes were labeled. In consequence, its results were often comparable to or sometimes even better than other methods.

\subsection{Noisy side information}\label{sec:experiment:noisy}

\begin{figure}
\begin{center}
\includegraphics[width=3.7in]{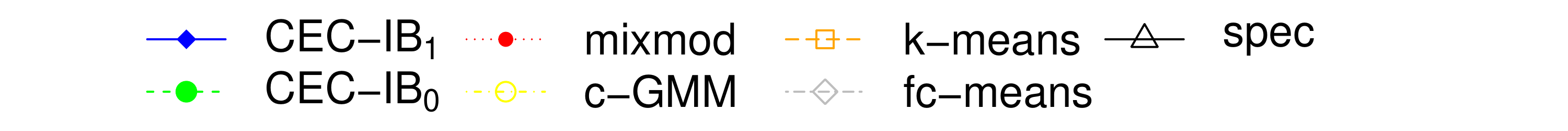}

\subfigure[Ecoli]{\includegraphics[width=1.7in]{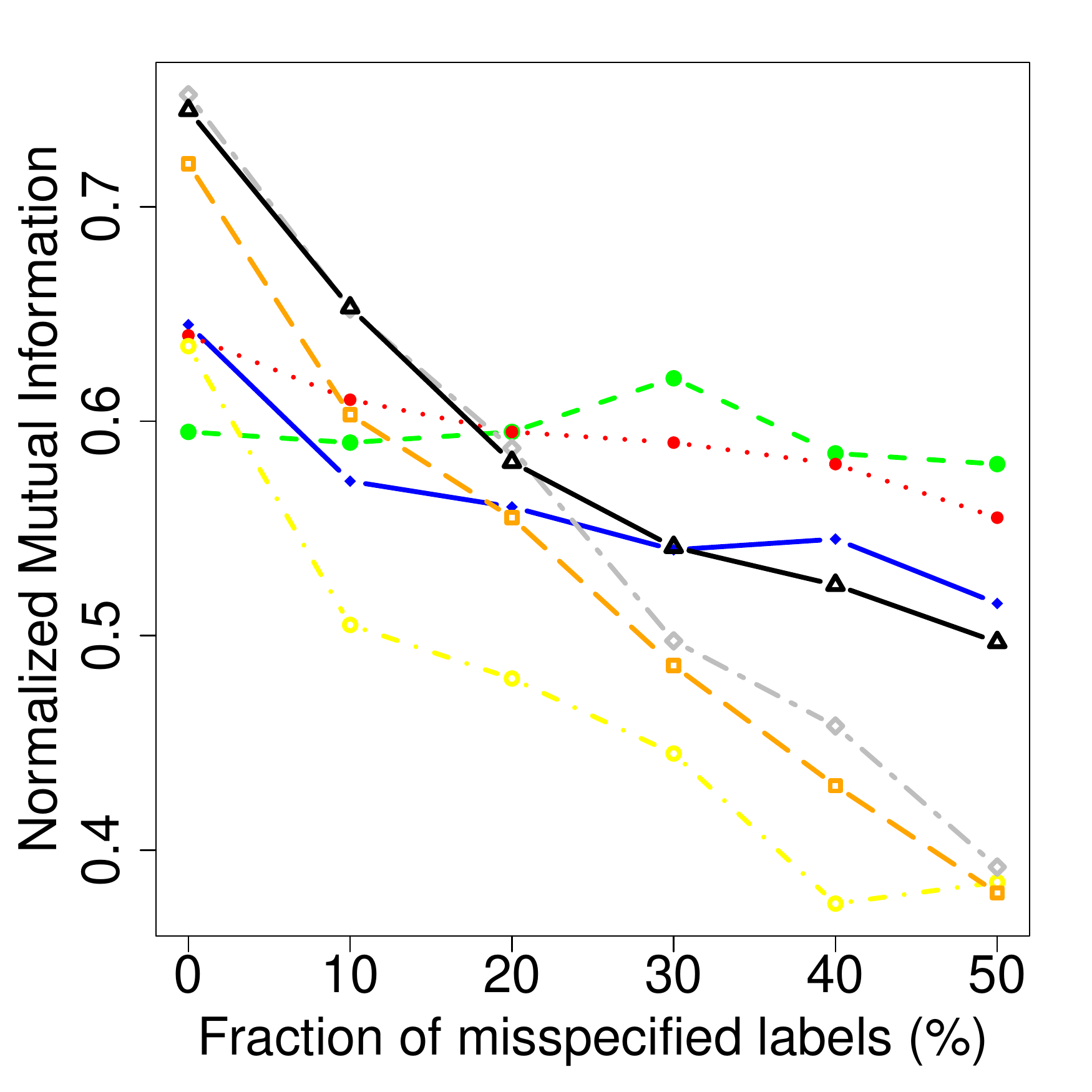}}
\subfigure[Glass]{\includegraphics[width=1.7in]{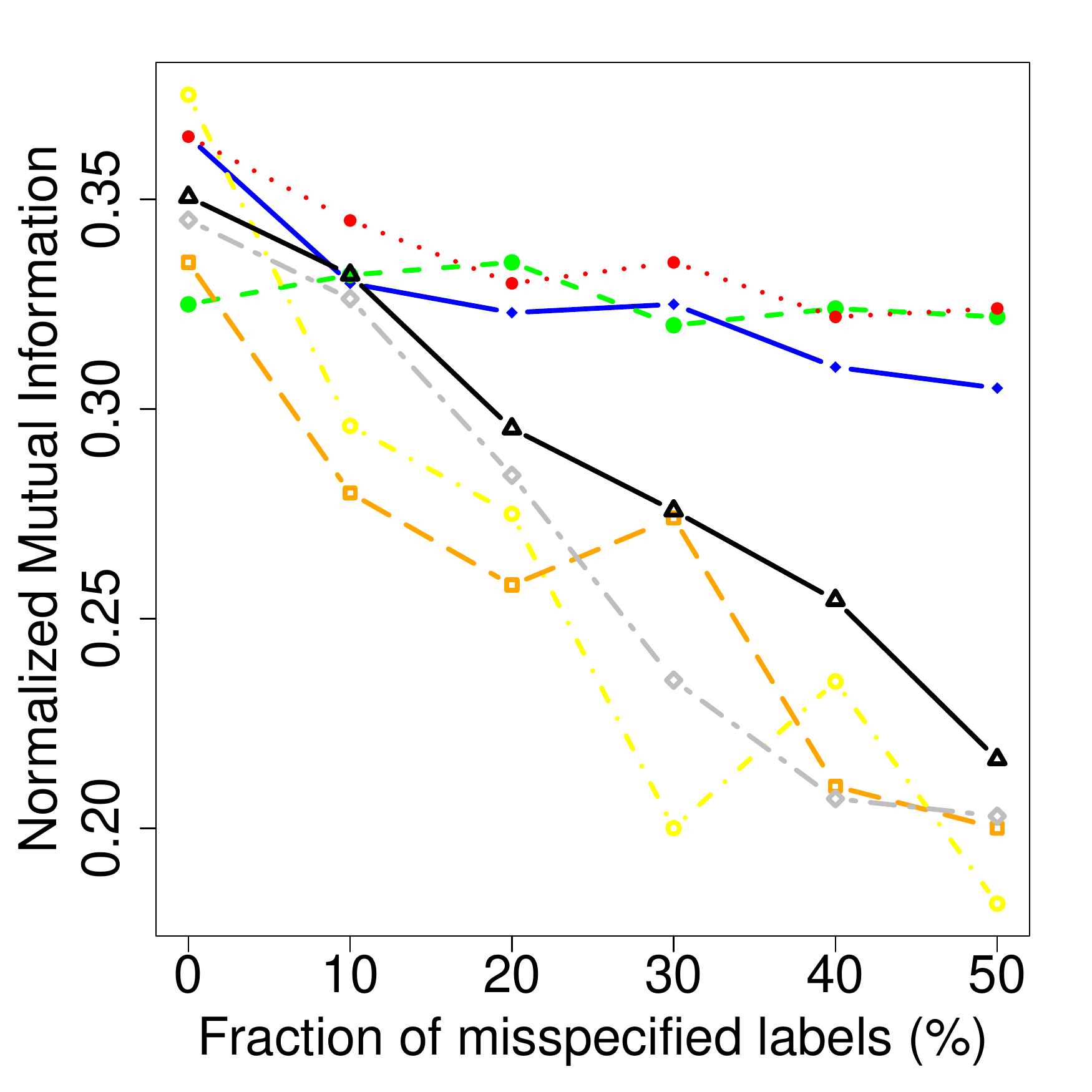}}
\subfigure[Segmentation]{\includegraphics[width=1.7in]{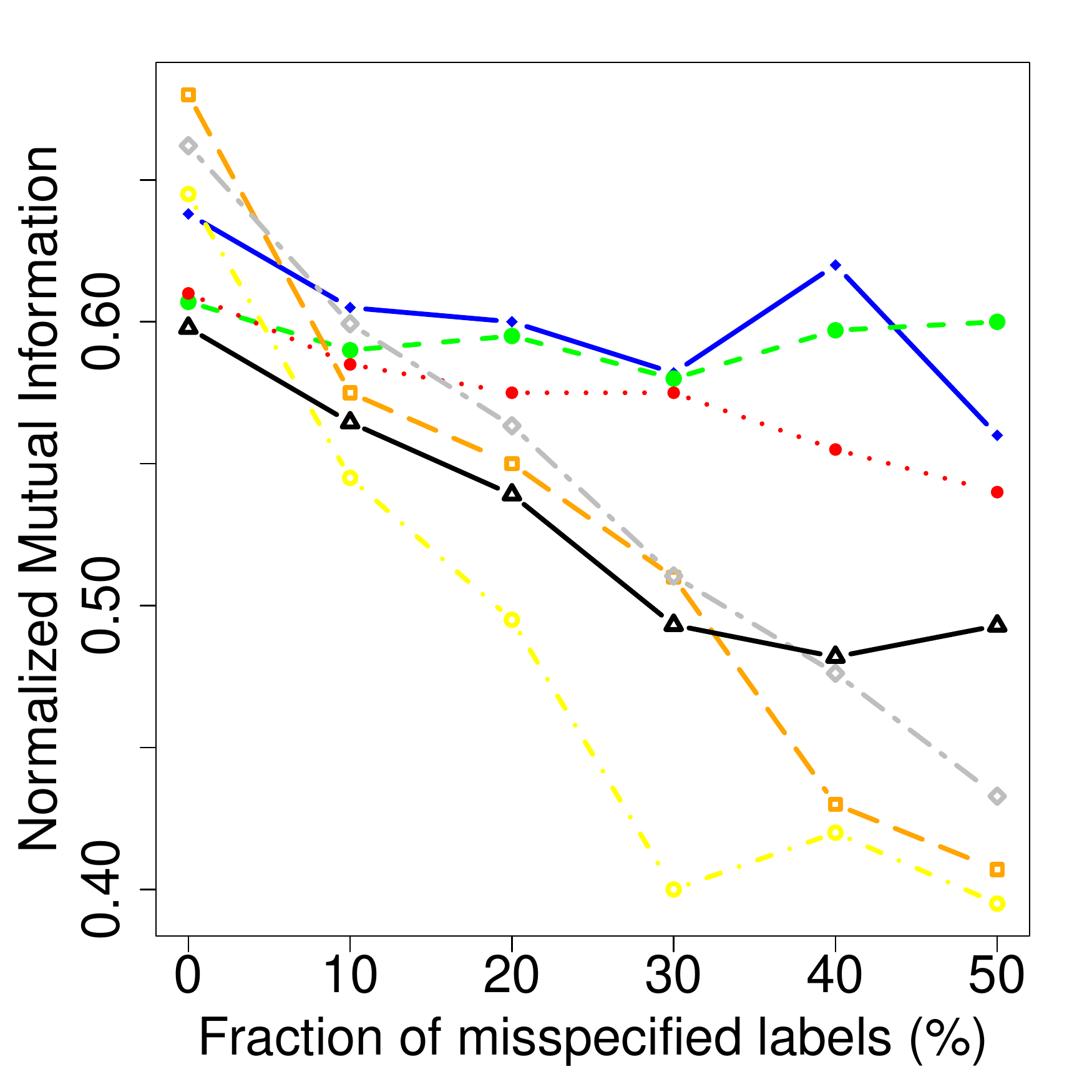}}
\subfigure[User Modeling]{\includegraphics[width=1.7in]{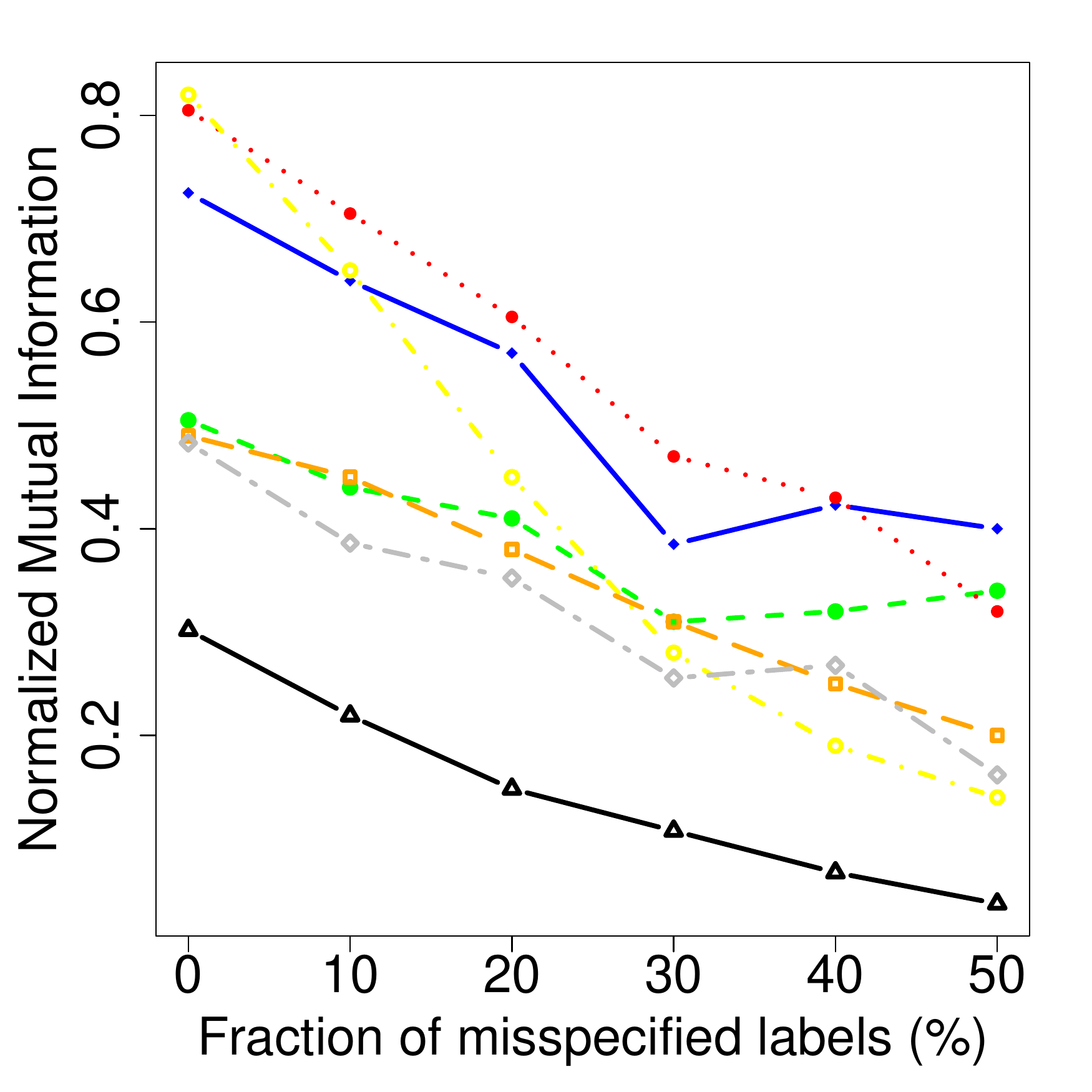}}
\subfigure[Vertebral]{\includegraphics[width=1.7in]{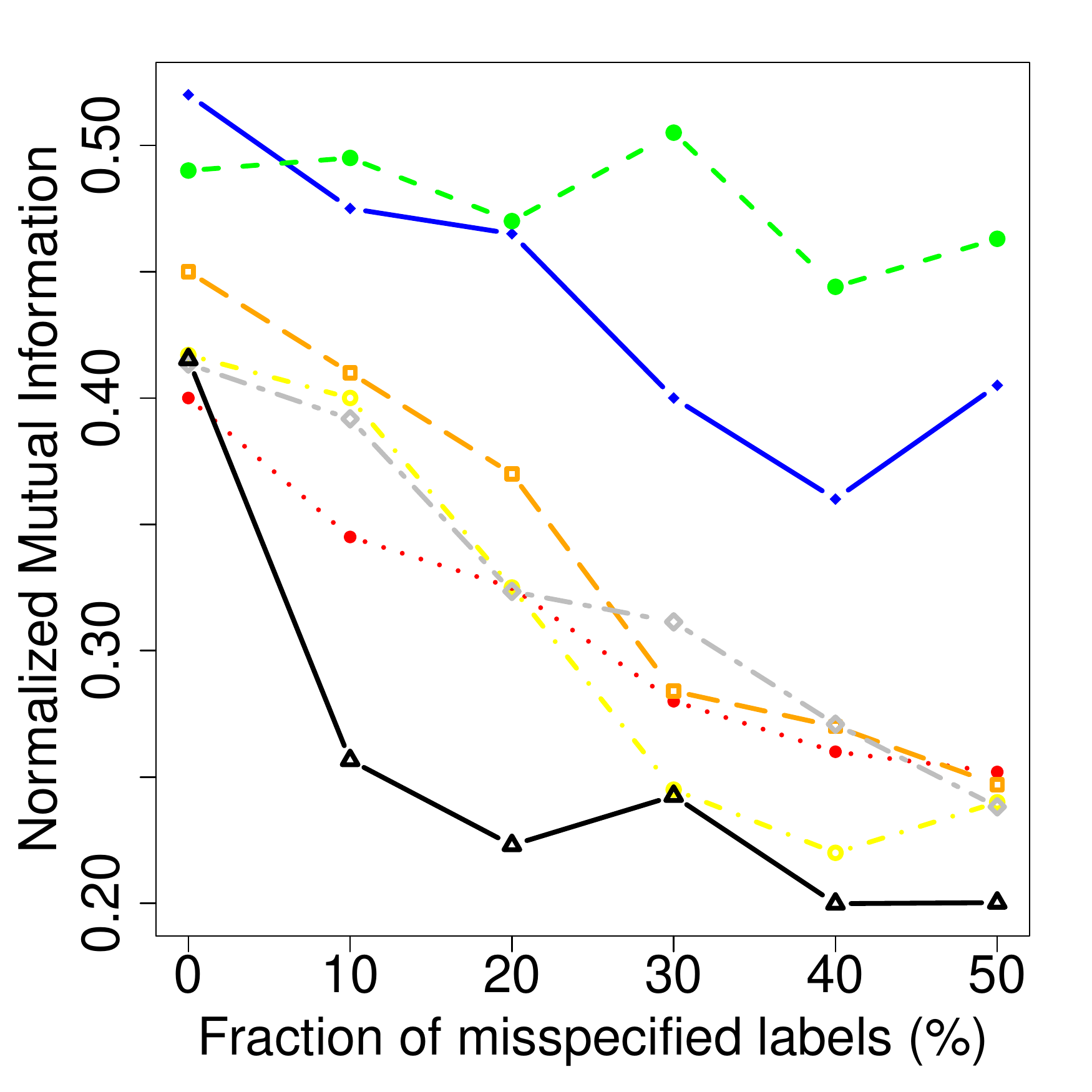}}
\subfigure[Wine]{\includegraphics[width=1.7in]{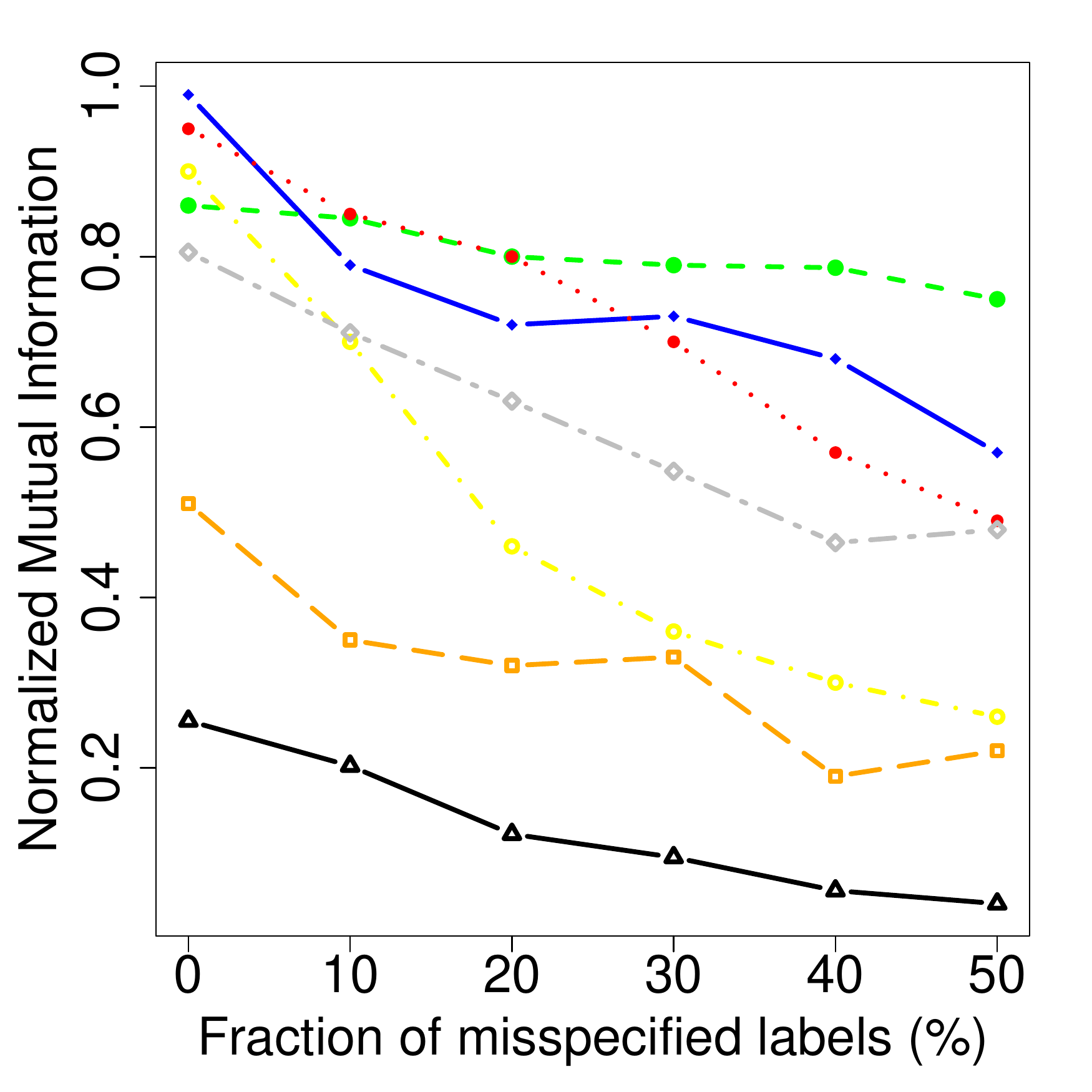}}
\subfigure[Iris]{\includegraphics[width=1.7in]{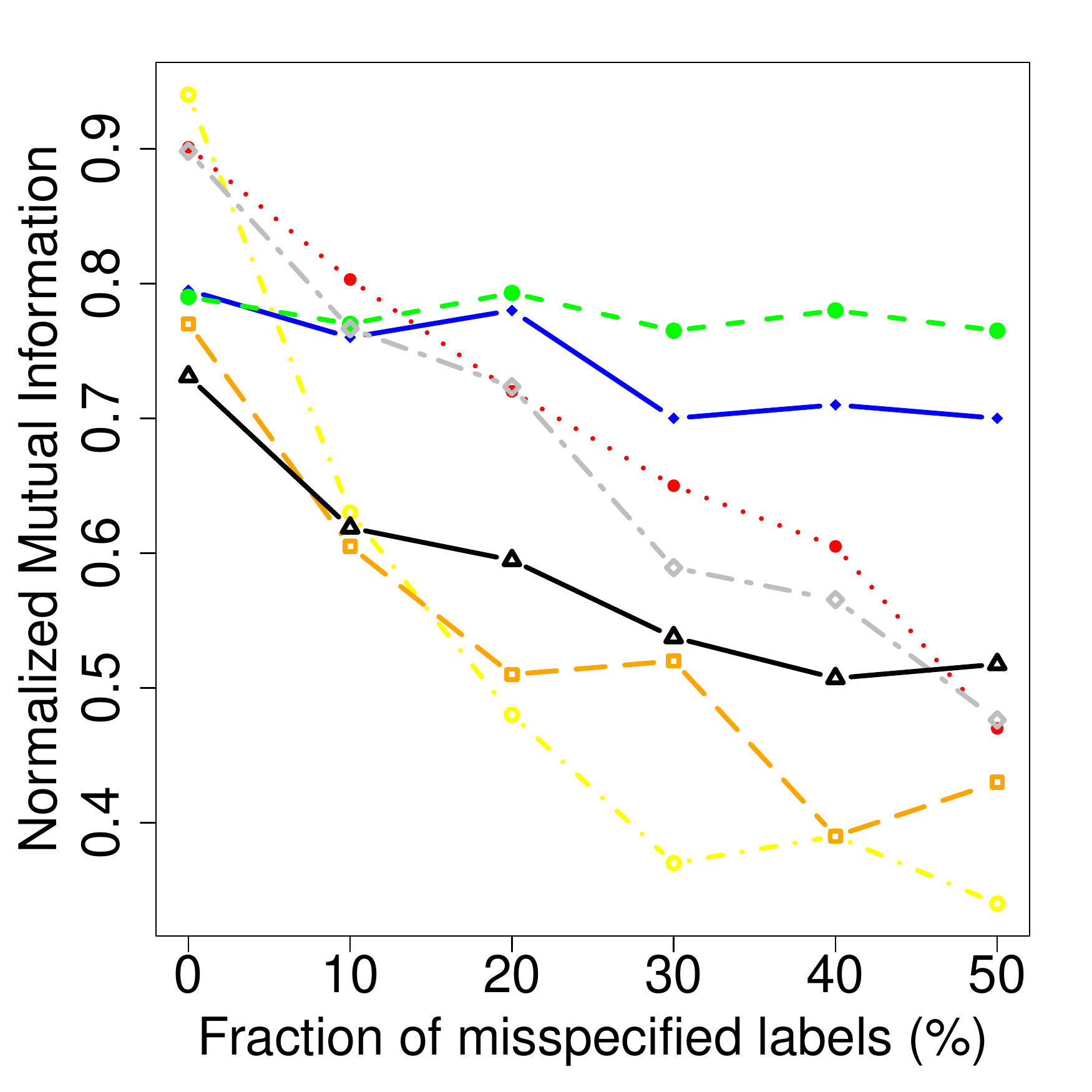}}
\end{center}
\caption{Influence of misspecified labels on the clustering results. {\cmc} was run with twice the true number of clusters.}
\label{fig:uci-noise}
\end{figure}

In real-world applications, the side information usually comes from human experts, who label training samples. Depending on the expertise of these workers, some part of this side information might be noisy or erroneous. Therefore, the clustering algorithm needs to be robust w.r.t.\ noisy side information.

To simulate the above scenario, we randomly selected 30\% of the data points as side information, as in Section \ref{sec:experiment:semi}, and assign incorrect labels for a fixed percentage of them (0\%, 10\%, 20\%, 30\%, 40\%, 50\% of misspecified labels). All methods were run in the same manner as in the previous experiments.

One can see in Figure \ref{fig:uci-noise} that $\cmc_0$ showed the highest robustness to noisy labels among all competing methods, i.e., the NMI deteriorated the least with increasing noise. Although $\cmc_1$ achieved higher NMI than $\cmc_0$ for correctly labeled data (without noise), its performance is usually worse than $\cmc_0$ when at least 30\% of labels are misspecified. The robustness of mixmod and spec is acceptable; their results vary with the used data set, but on average they cope with incorrect labels comparably to $\cmc_1$. In contrast, c-GMM, k-means and fc-means are very sensitive to noisy side information. Since their performance falls drastically below the results returned for strictly unsupervised case, they should not be used if there is a risk of unreliable side information.

\subsection{Influence of weight parameter}

From Figure~\ref{fig:uci1} it can be seen that $\beta = \beta_0$ often seems to be too small to benefit sufficiently from partition-level side information, although it provides high robustness to noisy labels. In this experiment, we investigate the dependence between the value of $\beta$ and the size of the labeled data set and the fraction of noisy labels, respectively.

\begin{figure}
\begin{center}
\subfigure[Ecoli]{\includegraphics[width=1.7in]{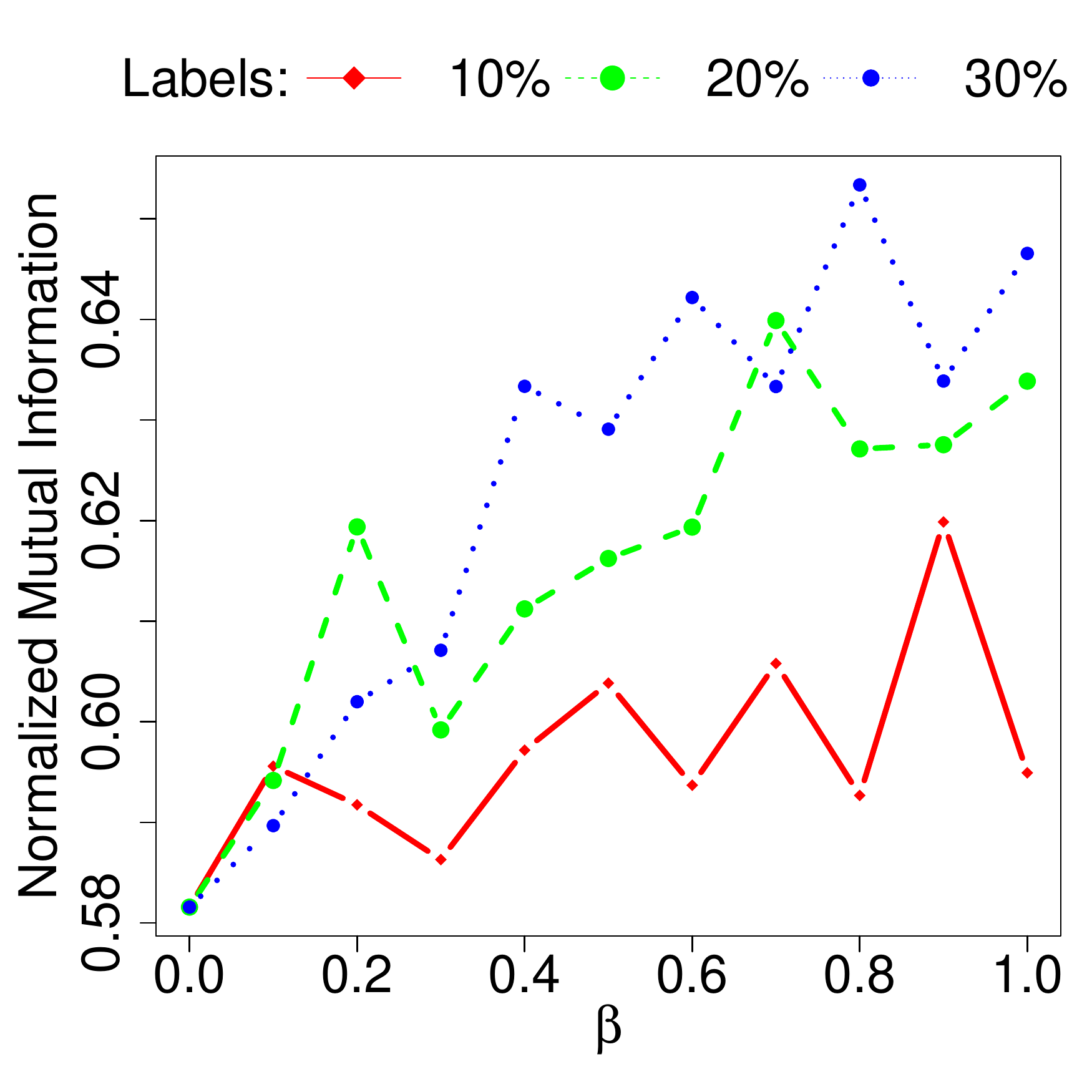}}
\subfigure[Glass]{\includegraphics[width=1.7in]{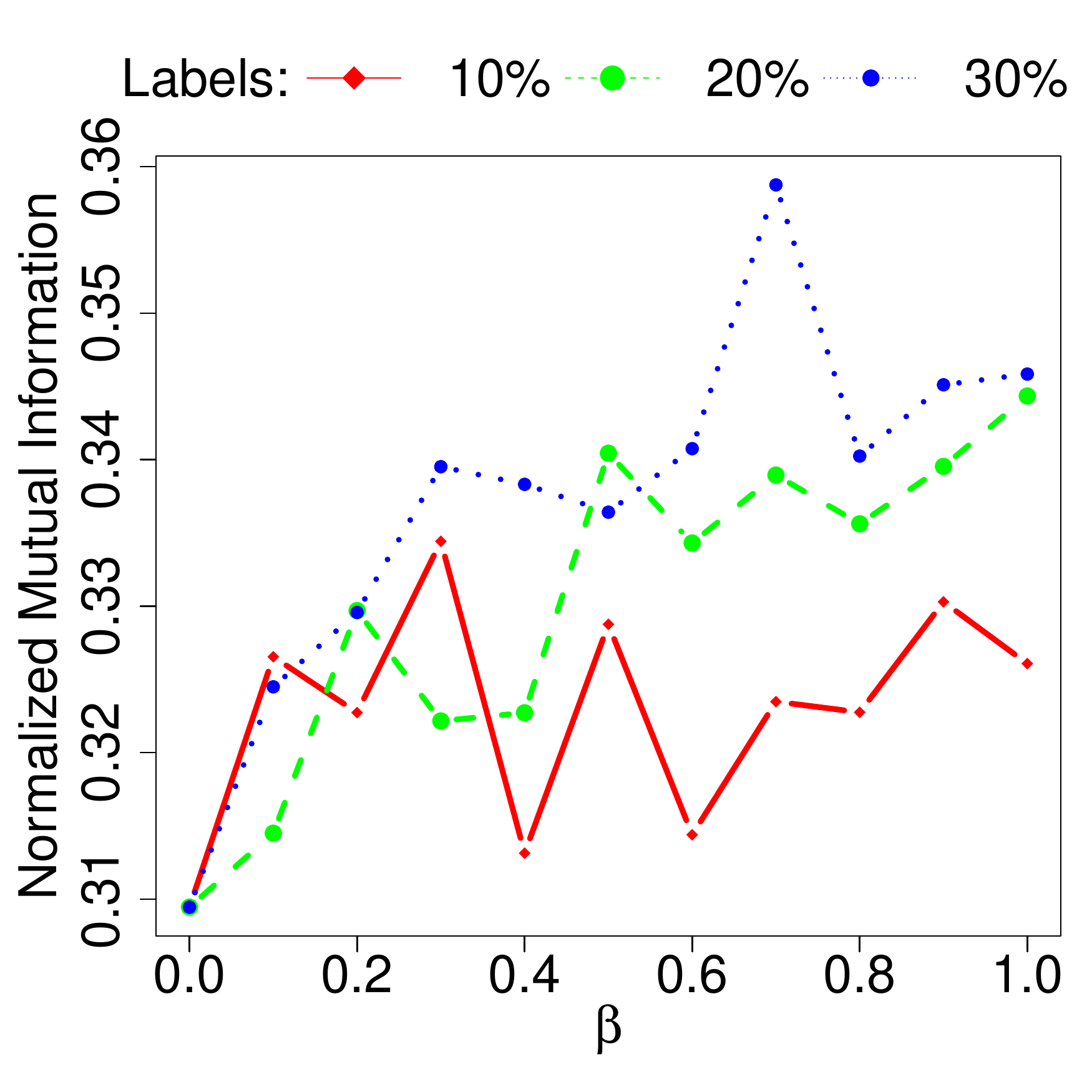}}
\subfigure[Segmentation]{\includegraphics[width=1.7in]{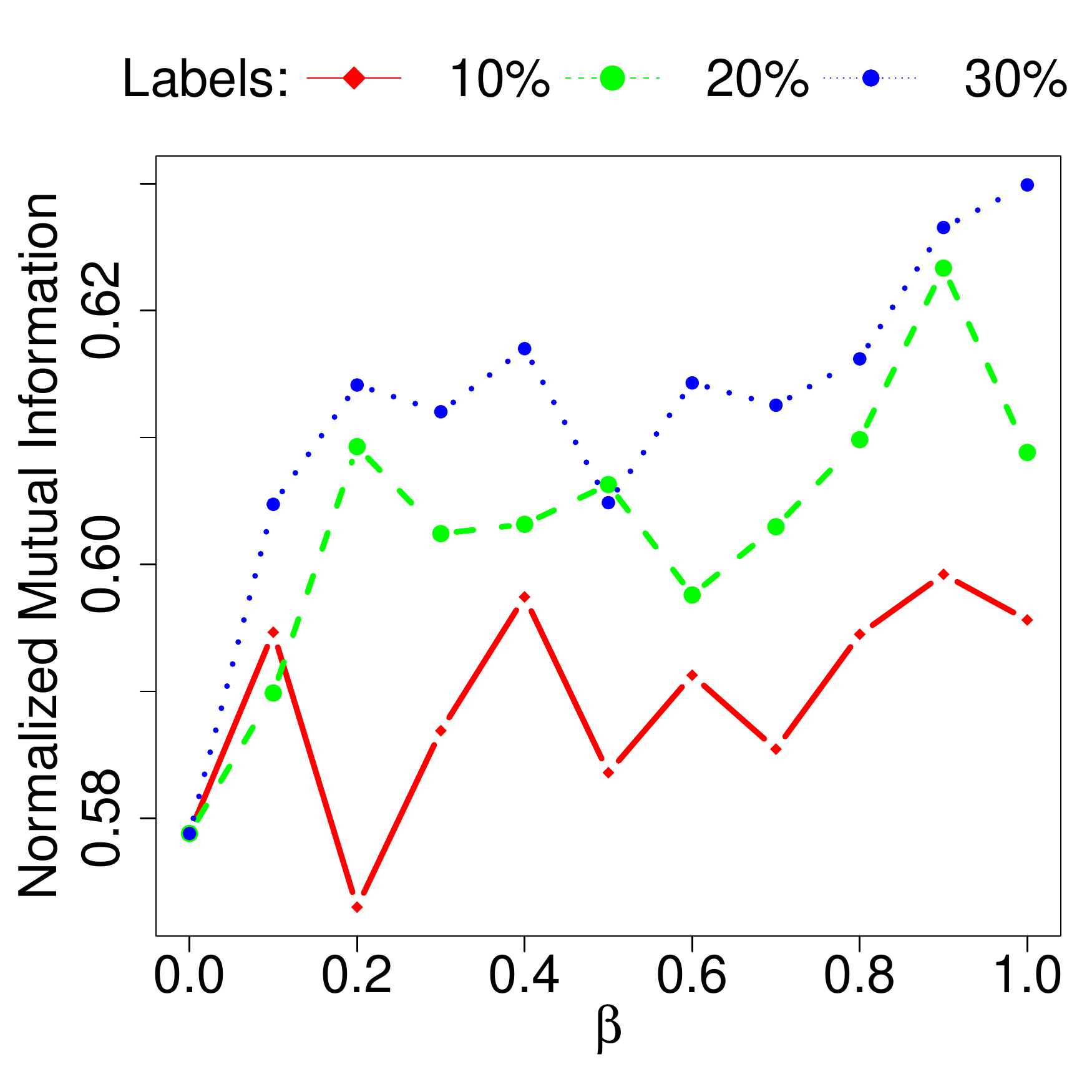}}

\subfigure[User Modeling]{\includegraphics[width=1.7in]{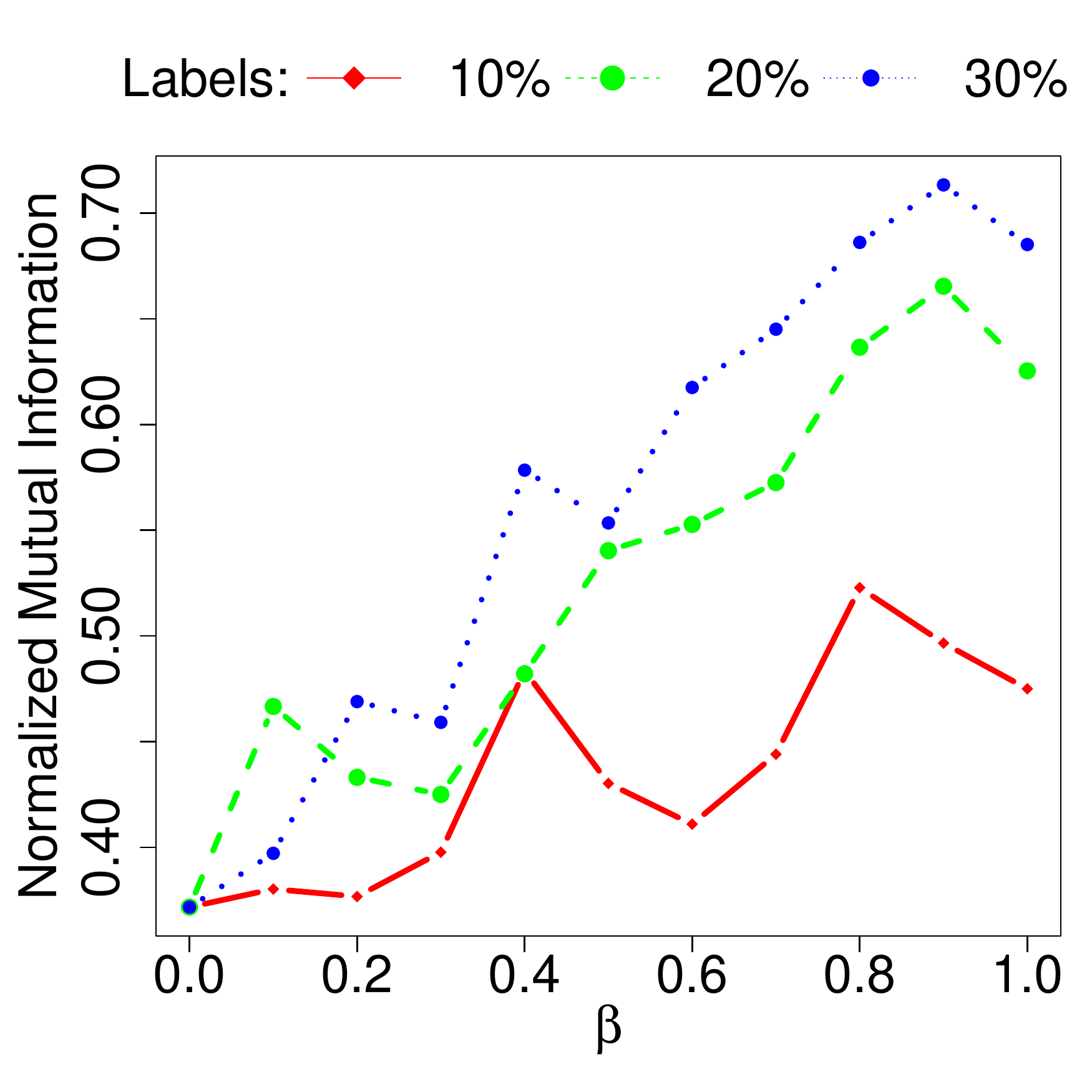}}
\subfigure[Vertebral]{\includegraphics[width=1.7in]{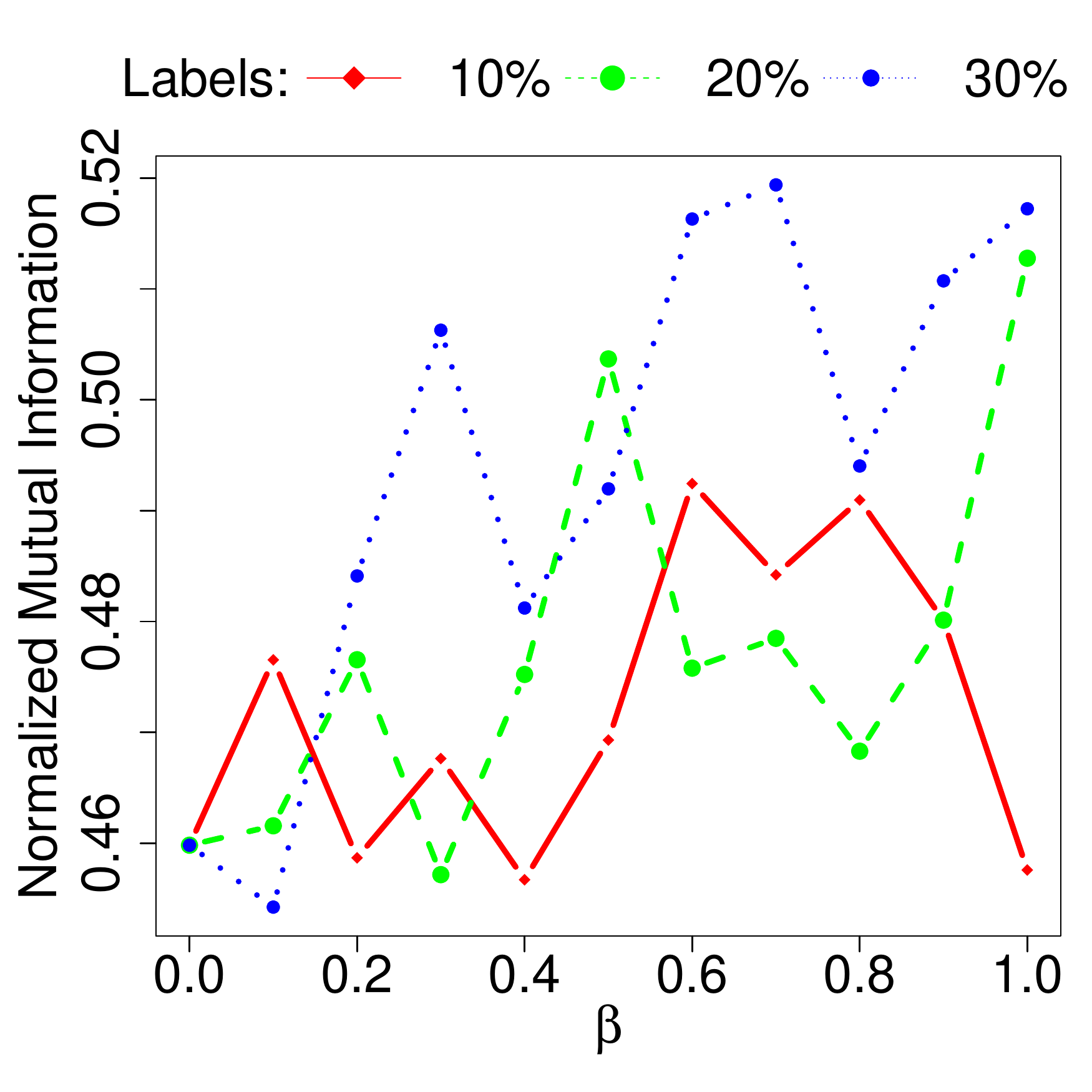}}
\subfigure[Wine]{\includegraphics[width=1.7in]{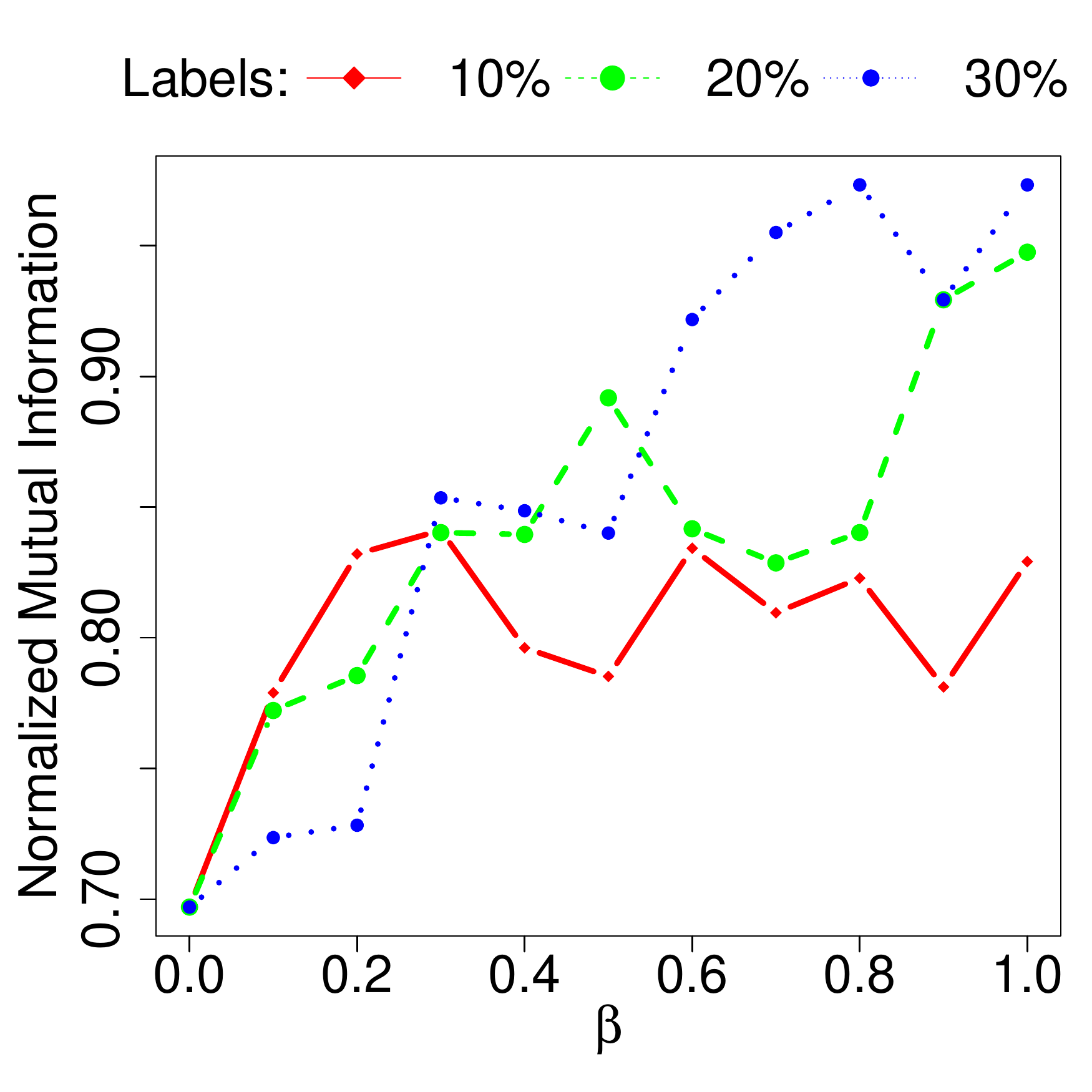}}

\subfigure[Iris]{\includegraphics[width=1.7in]{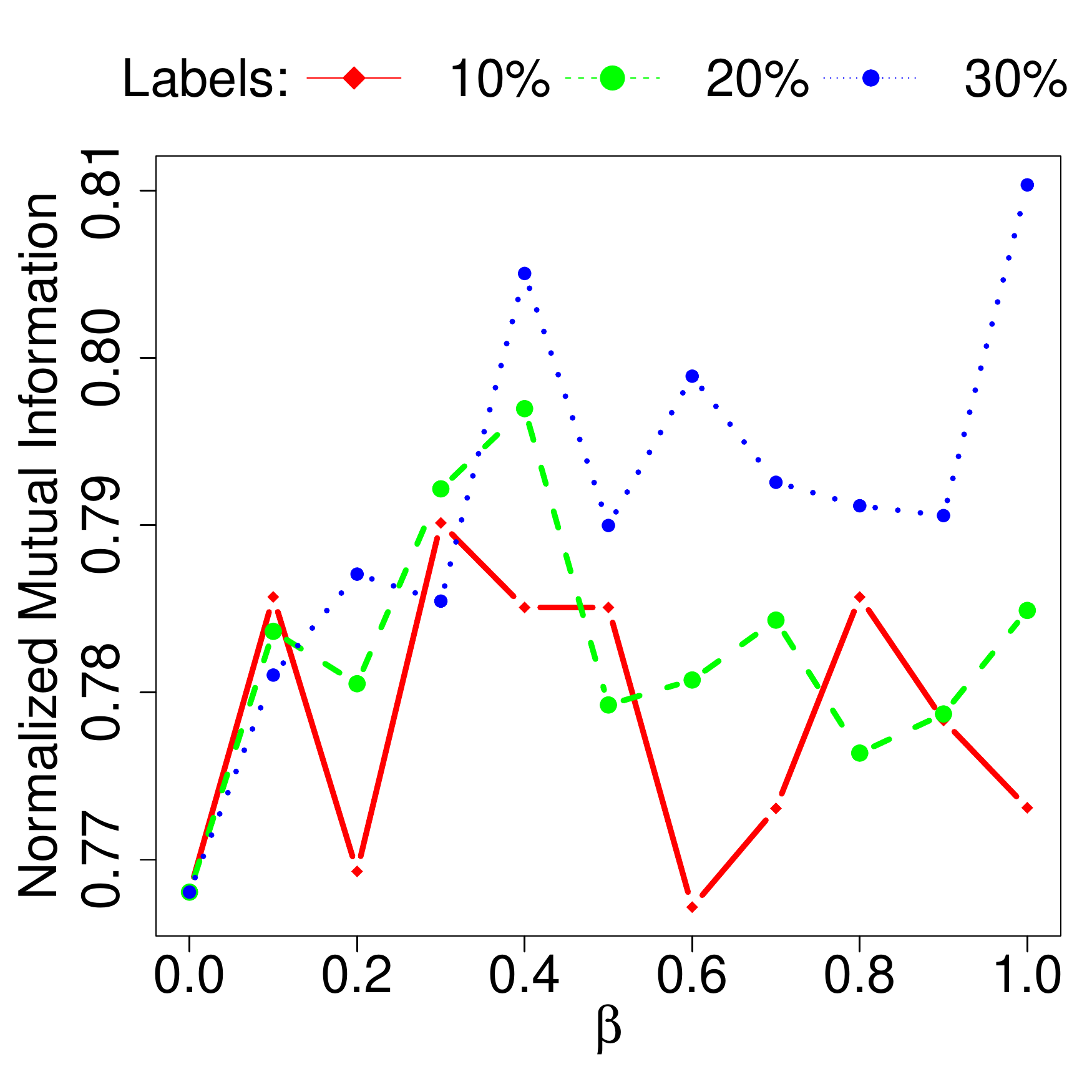}}
\end{center}
\caption{Dependence between the number of labeled data points and the value of parameter $\beta$.}
\label{fig:beta-labels}
\end{figure}

\begin{figure}
\begin{center}
\subfigure[Ecoli]{\includegraphics[width=1.7in]{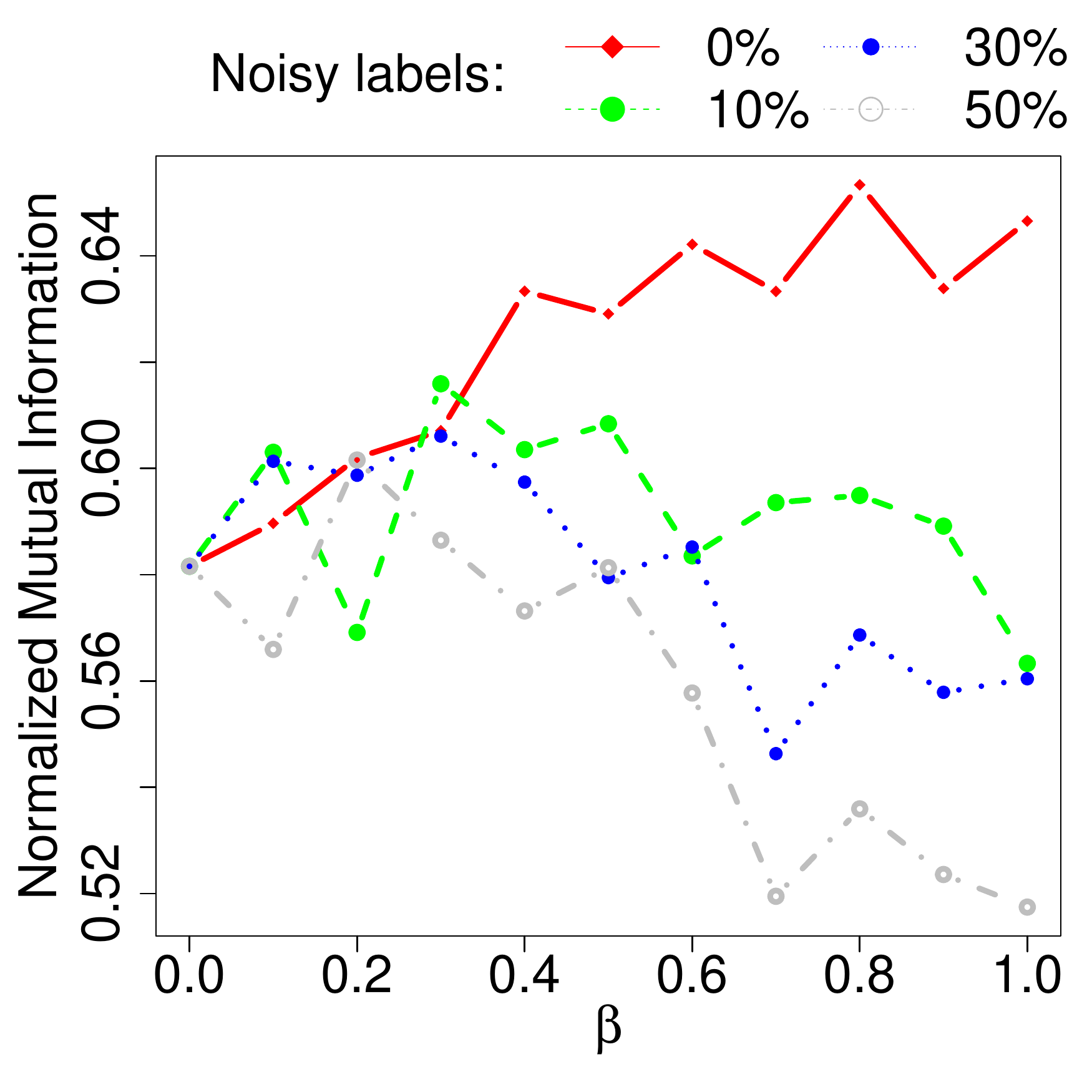}}
\subfigure[Glass]{\includegraphics[width=1.7in]{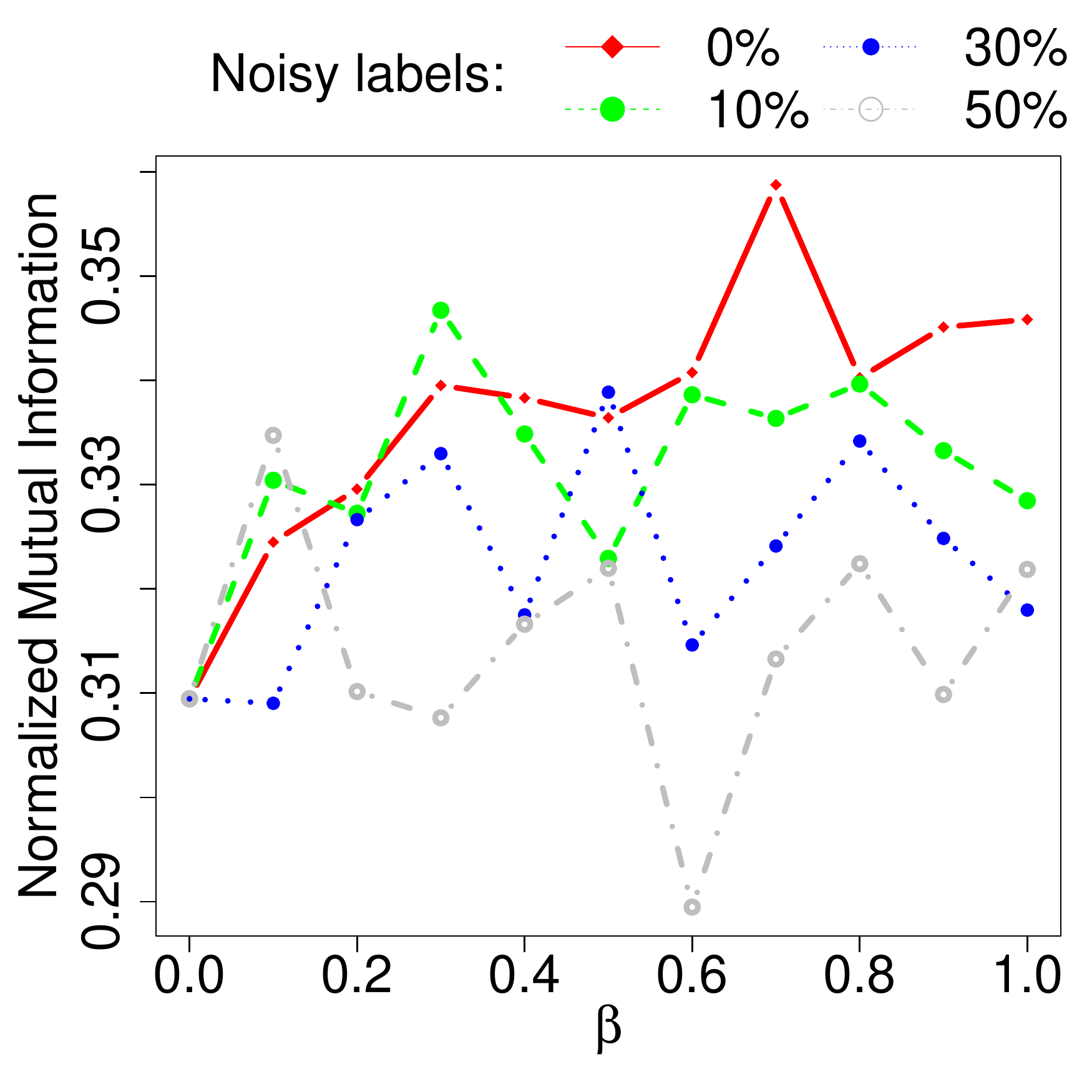}}
\subfigure[Segmentation]{\includegraphics[width=1.7in]{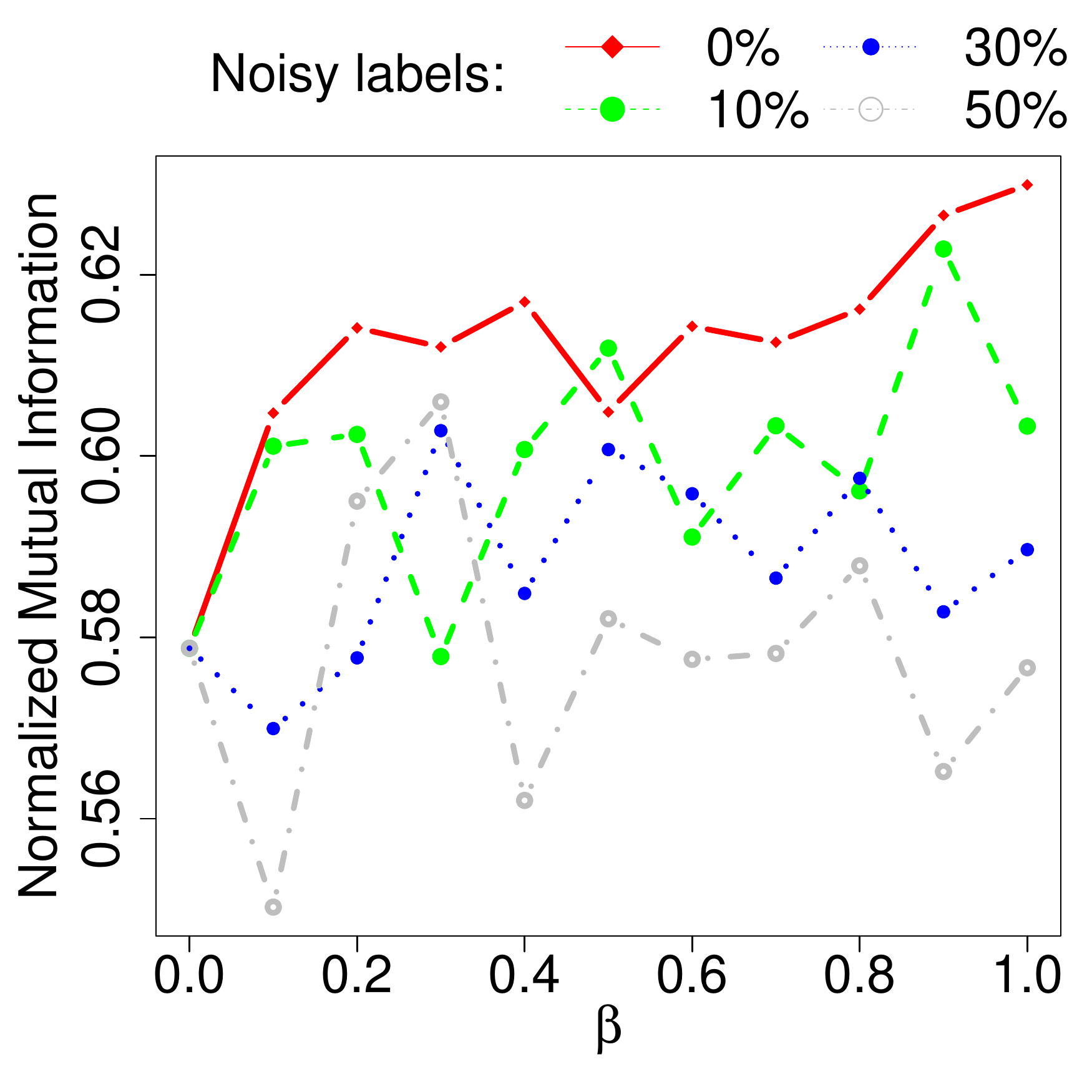}}

\subfigure[User Modeling]{\includegraphics[width=1.7in]{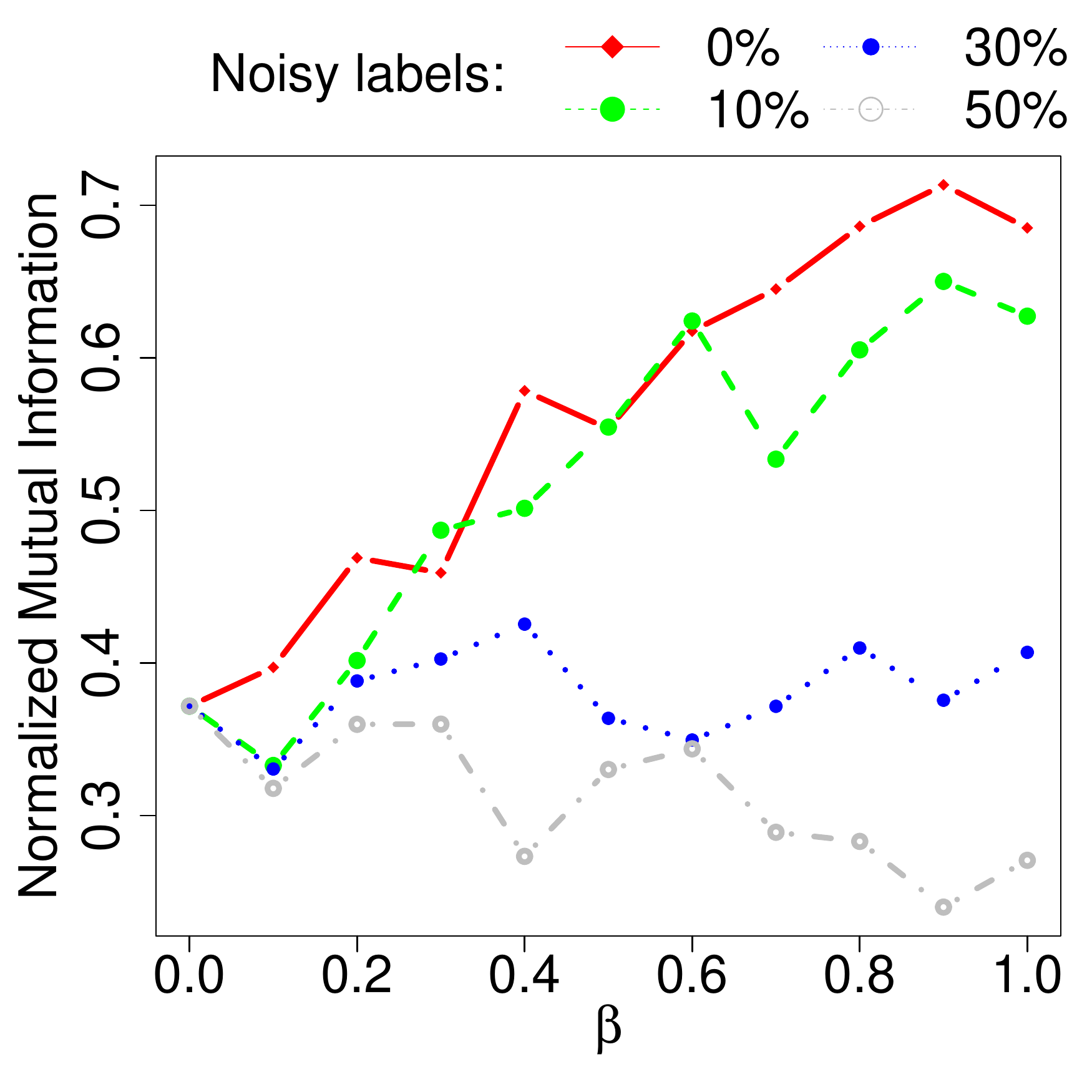}}
\subfigure[Vertebral]{\includegraphics[width=1.7in]{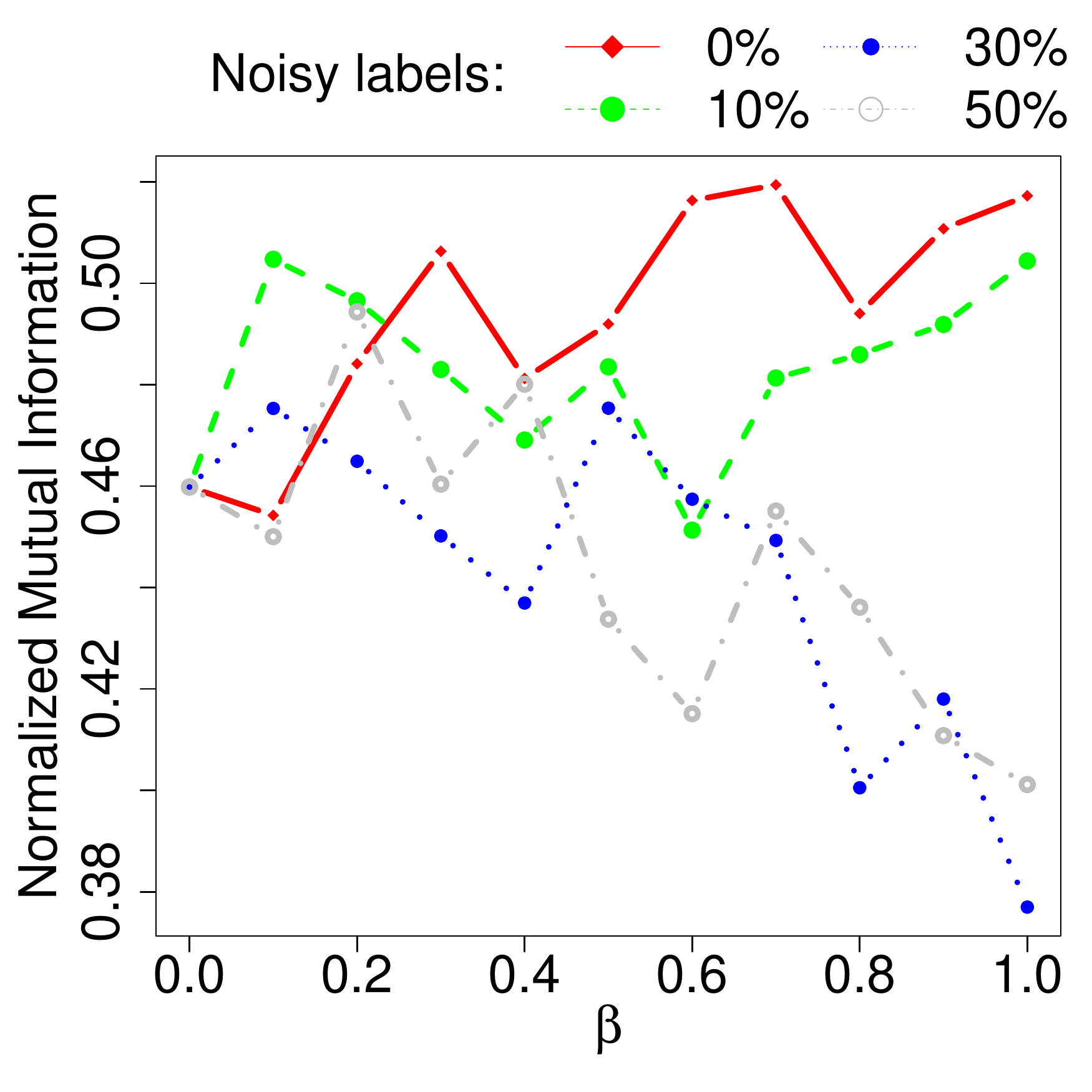}}
\subfigure[Wine]{\includegraphics[width=1.7in]{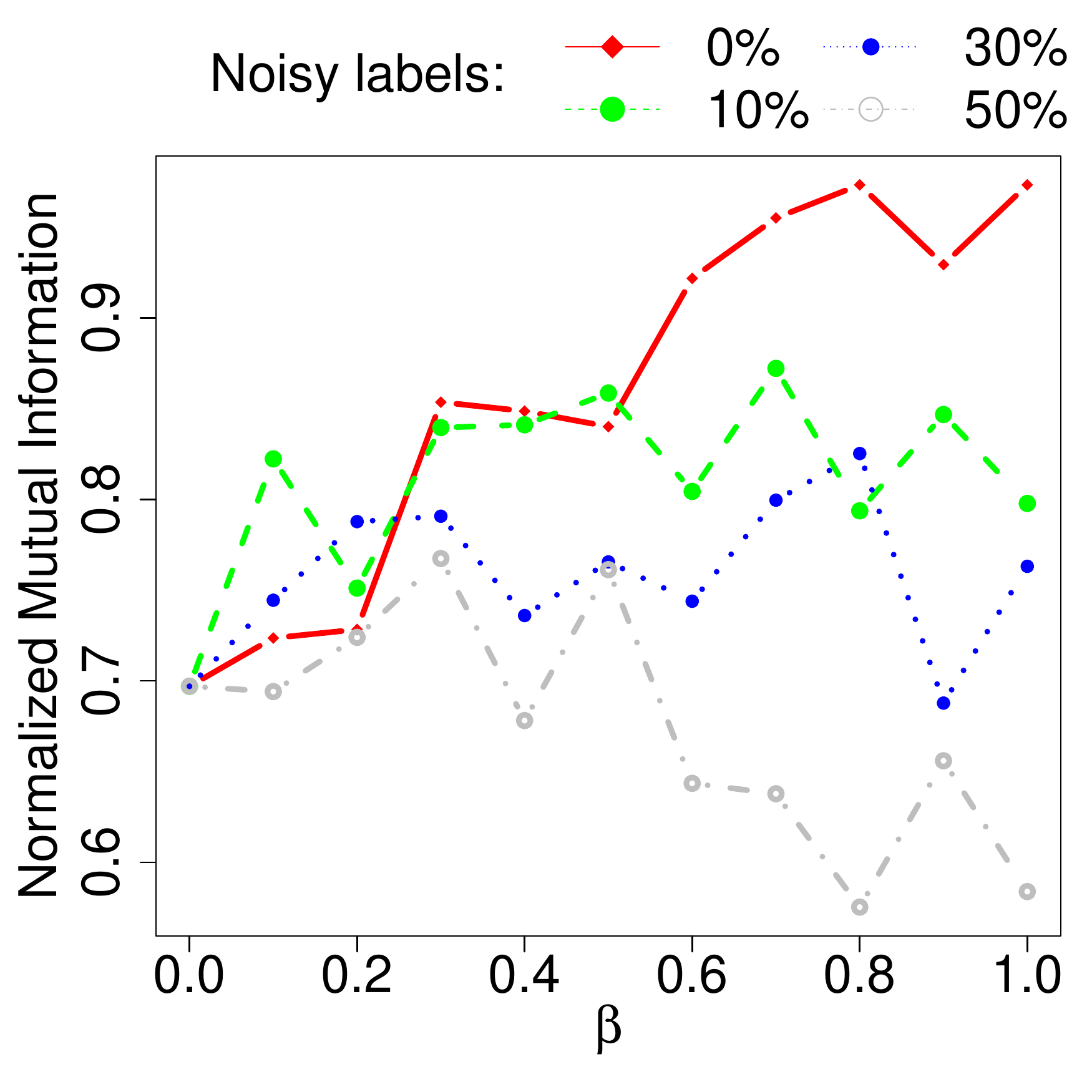}}

\subfigure[Iris]{\includegraphics[width=1.7in]{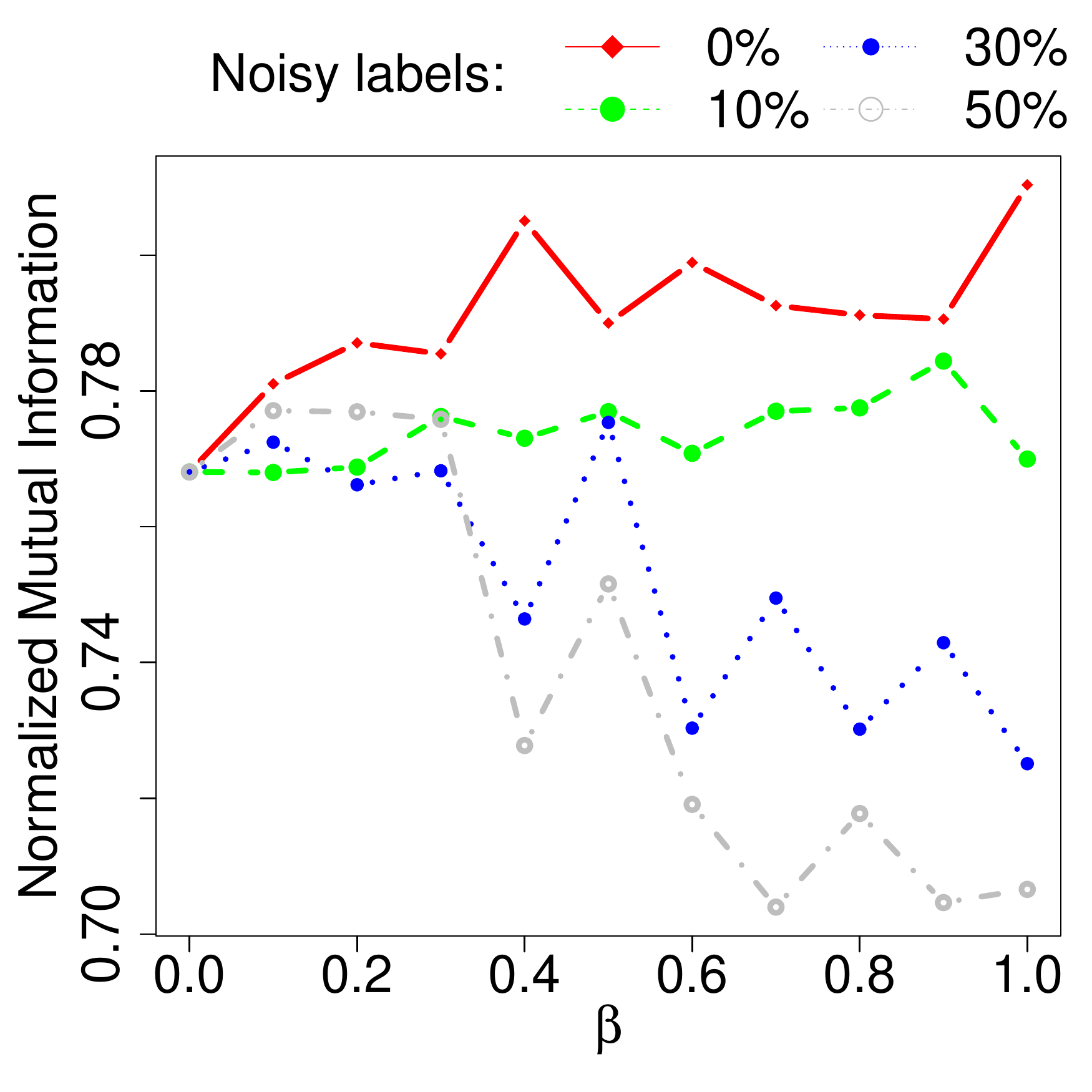}}
\end{center}
\caption{Dependence between the fraction of incorrect labels and the value of parameter $\beta$.}
\label{fig:beta-noise}
\end{figure}

First, we checked the performance of {\cmc} with different values of $\beta$ in the noiseless case. We randomly selected  10\%, 20\% and 30\% of the data points, respectively, and labeled them according to their true classes. Figure \ref{fig:beta-labels} shows that {\cmc} with $\beta= \beta_0$ run on 30\% labels performed similarly to {\cmc} with $\beta=1$ run on 10\% labels. Therefore, we see that the lack of labeled data can be compensated with a larger value of $\beta$. Moreover, a larger value of $\beta$ makes {\cmc} benefit more from a larger number of correctly labeled data points.

In the second experiment we investigated the relation between the fraction of noisy side information and the weight parameter. We drew 30\% of the data points and labeled 0\%, 10\%, 30\%, and 50\% of them incorrectly (the remaining selected data points were labeled with their correct class labels). We see in Figure \ref{fig:beta-noise} that a small noise of 10\% did not have severe negative effects on the clustering results. In this case NMI was almost always higher than in the unsupervised case (i.e., for $\beta=0$), even for $\beta=1$. For 50\% of incorrectly labeled data points, increasing $\beta$ has a negative effect on the clustering performance, while $\beta=\beta_0$ provided high robustness to the large amount of noisy labels and in most cases performed at least as well as the unsupervised scenario. For the case where 30\% of labels were misspecified, choosing $\beta < 0.6$ seems to produce results at least as good as when no side information is available.

\subsection{Hierarchy of chemical classes -- a case study} \label{sec:chem}

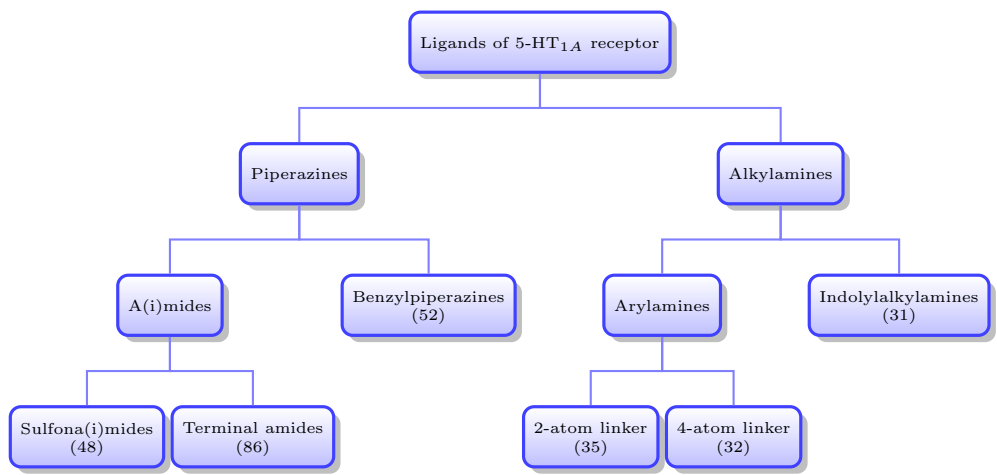
\begin{figure}[t]
\centering
\begin{tikzpicture}[every tree node/.style={top color=white,
    bottom color=blue!25,
    rectangle,rounded corners,
    minimum height=8mm,
    draw=blue!75,
    very thick,
    drop shadow,
    align=center,
    text depth = 0pt},edge from parent/.style={draw=blue!50,
    thick},font=\tiny]
\Tree [.{Ligands of 5-HT$_{1A}$ receptor}
        [.{Piperazines}
            [.{A(i)mides} 
                [.{Sulfona(i)mides\\(48)} ]
                [.{Terminal amides\\(86)} ] ]
            [.{Benzylpiperazines\\(52)} ] ] 
        [.{Alkylamines}
            [.{Arylamines} 
                [.{2-atom linker\\(35)} ]
                [.{4-atom linker\\(32)} ] ]
            [.{Indolylalkylamines\\(31)} ] ] 
]
\end{tikzpicture}
\caption{Hierarchy tree of chemical compounds classes. The numbers in brackets indicate the number of times the corresponding class appeared in the data set.}
\label{fig:hierarchy}
\end{figure}

Our {\cmc} cost function only penalizes including elements from different categories into the same cluster. Covering a single category by more than one cluster is not penalized if the cost for model accuracy outweighs the cost for model complexity. In this experiment, we will show that this property is useful in discovering subgroups from side information derived from a cluster hierarchy.

We considered a data set of chemical compounds that act on 5-HT$_{1A}$ receptor ligands, one of the proteins responsible for the regulation of the central nervous system \cite{olivier1999, smieja2016chem}. Part of this data set was classified hierarchically by an expert~\cite{warszycki2013linear}, as shown in Figure~\ref{fig:hierarchy}. For an expert it is easier to provide a coarse categorization rather than a full hierarchical classificiation, especially if it is not clear how many subgroups exist. Therefore, in some cases, it might be easier to get a hierarchical structure based on this coarse categorization made by the expert and an automatic clustering algorithm that finds a partition corresponding to the clusters at the bottom of the hierarchy.

We used the Klekota-Roth fingerprint representation of chemical compounds \cite{klekota2008chemical}, which describes each object by a binary vector, where ``1'' means presence and ``0'' denotes absence of a predefined chemical pattern. Since this representation contains 4860 features in total, its direct application to model-based clustering can lead to singular covariance matrices of clusters. Therefore, PCA was used to reduce its dimension to the 10 most informative components. This data set contains 284 examples in total (see Figure \ref{fig:hierarchy} for the cardinalities of particular chemical classes).

\begin{figure}
\begin{center}
\includegraphics[width=3.7in]{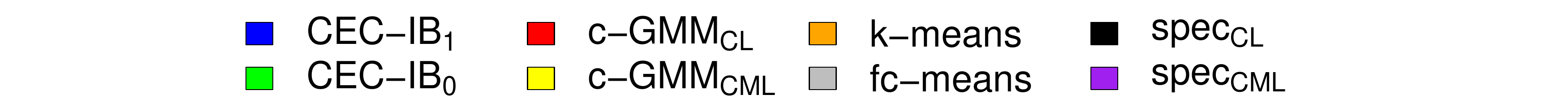}

\subfigure[Noiseless]{\includegraphics[width=2.0in]{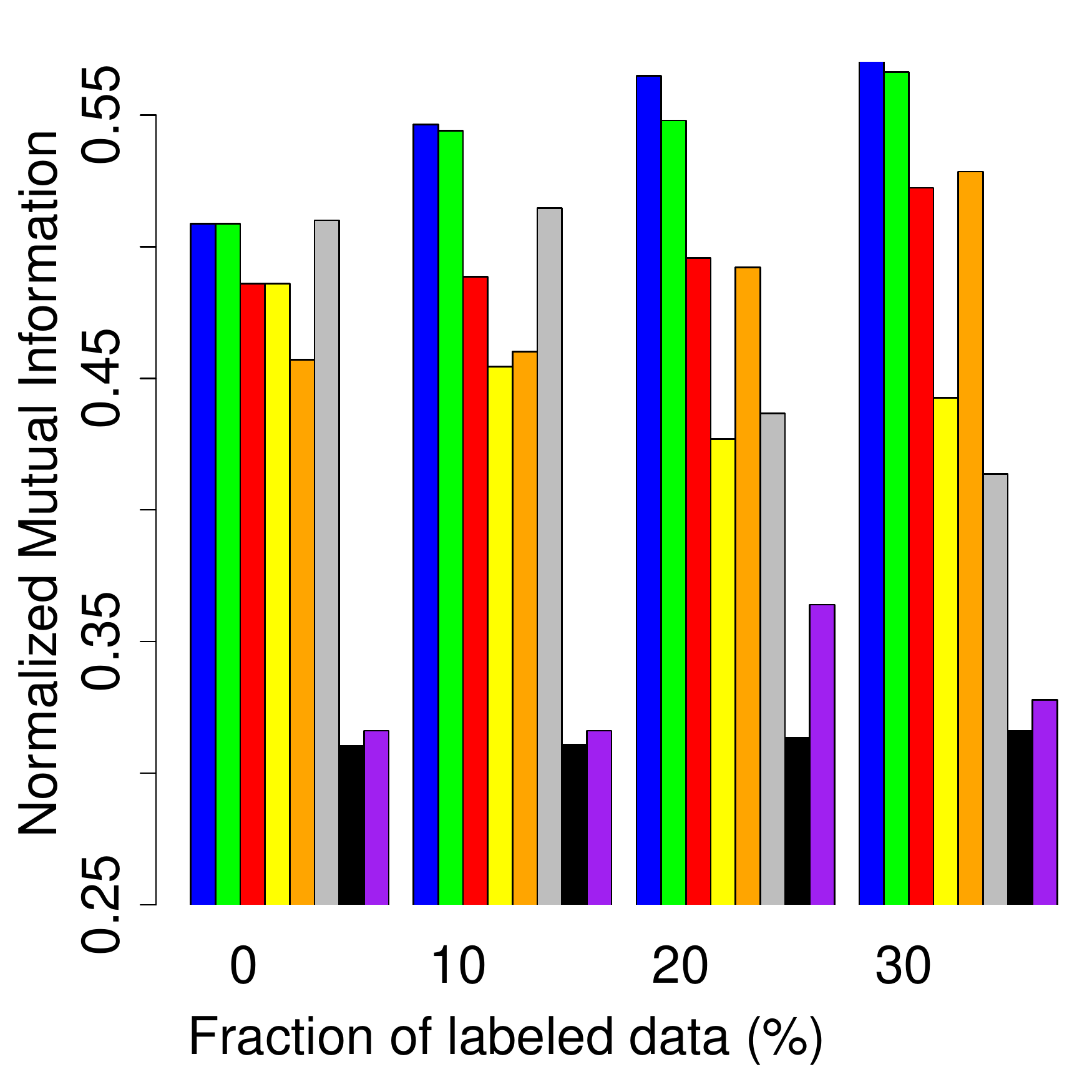}}
\subfigure[10\% of noise]{\includegraphics[width=2.0in]{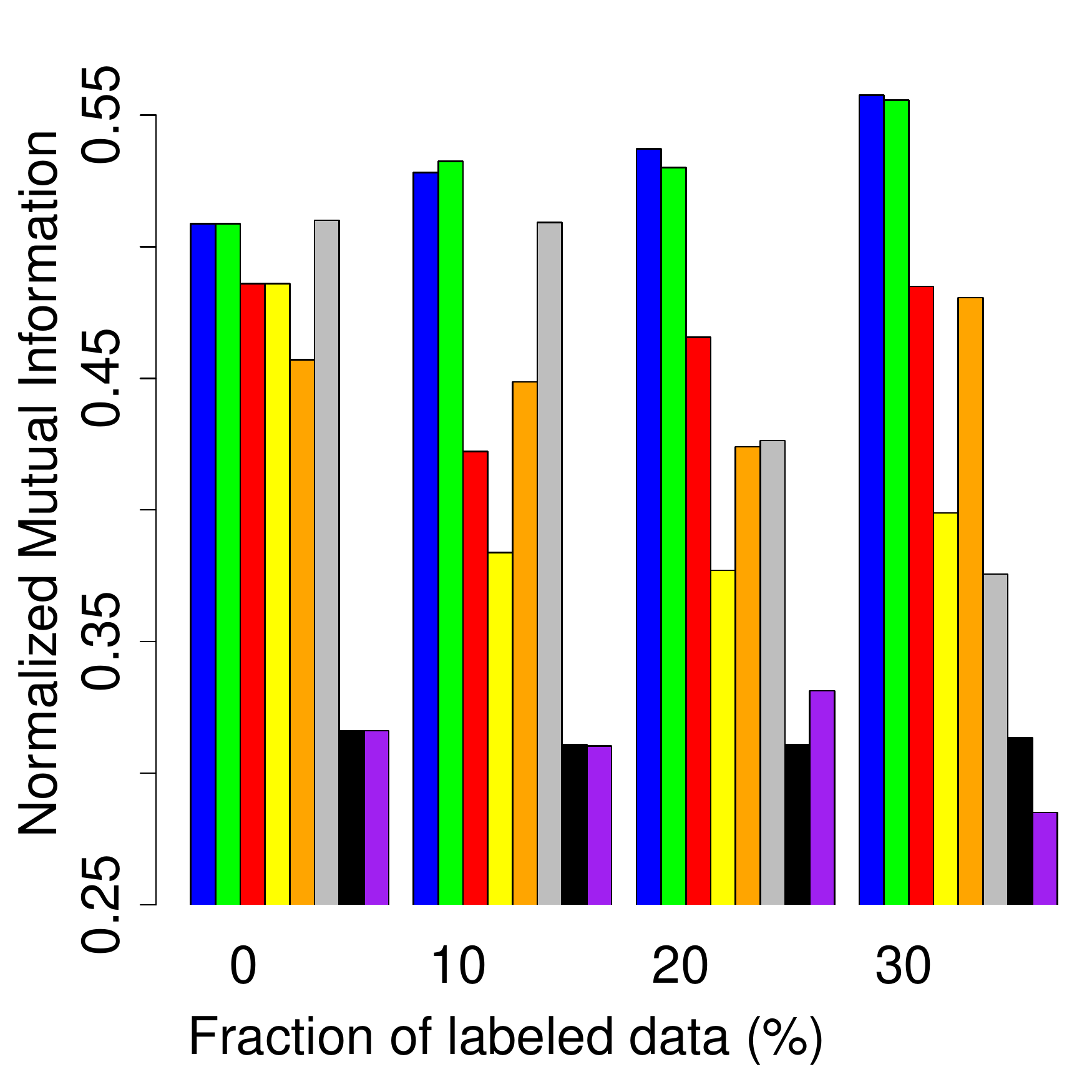}}
\subfigure[20\% of noise]{\includegraphics[width=2.0in]{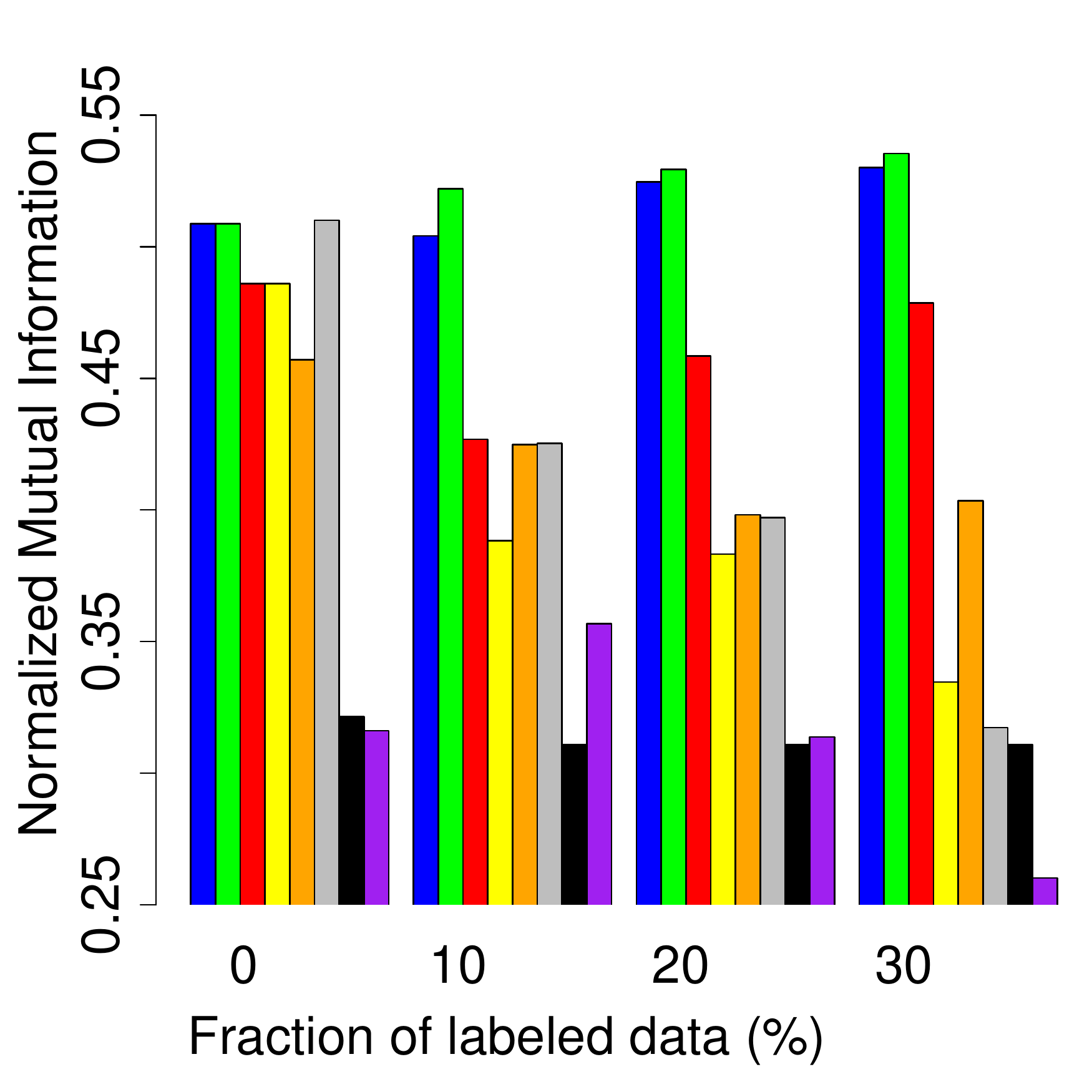}}
\subfigure[30\% of noise]{\includegraphics[width=2.0in]{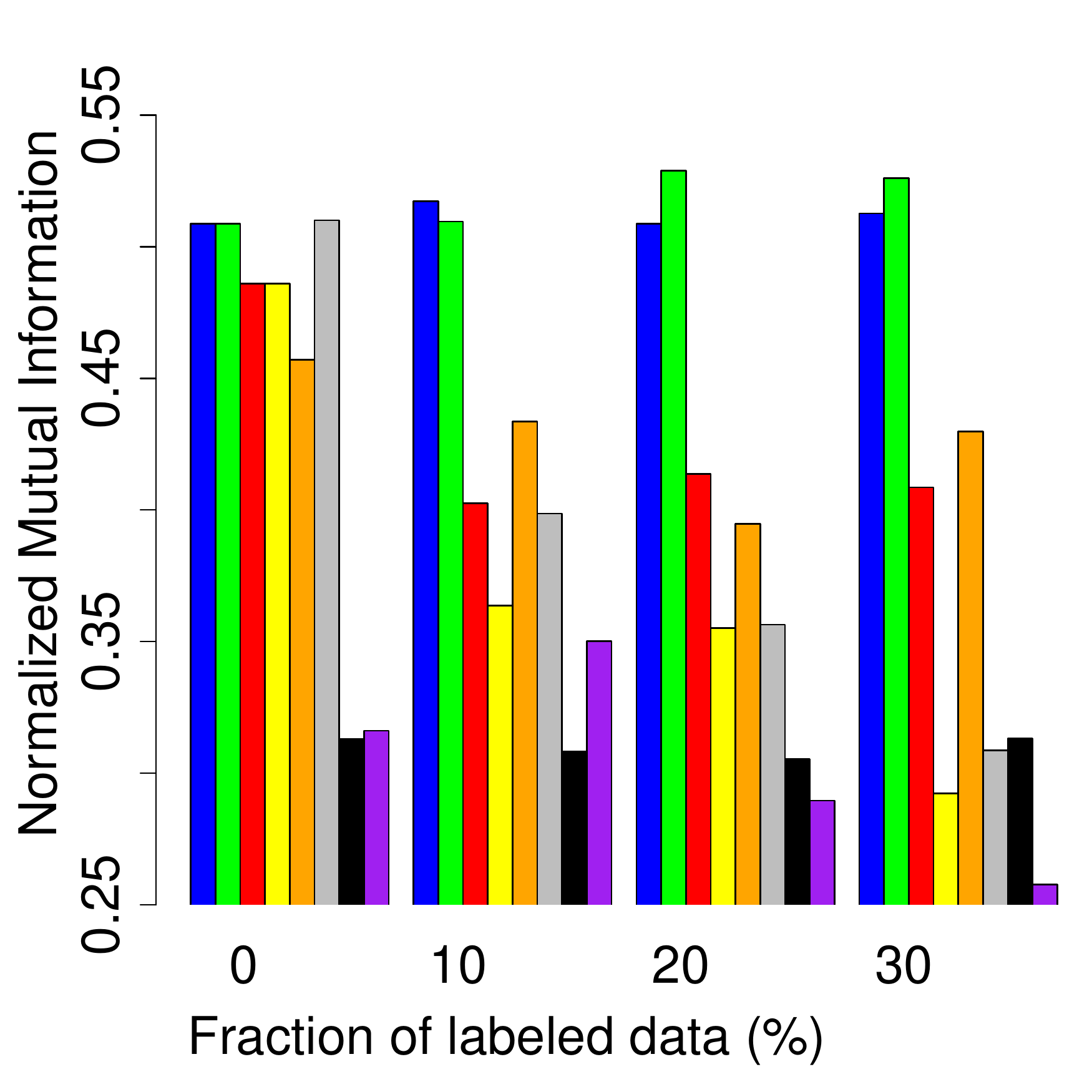}}
\end{center}
\caption{Detection of chemical subgroups. }
\label{fig:chem}
\end{figure}

We generated partition-level side information from the division of chemical data set into two classes: Piperazines and Alkylamines. We considered 0\%, 10\%, 20\% and 30\% of data points to be labeled and supposed that the human expert assigns incorrect labels with probabilities:  0\%, 10\%, 20\%, and 30\% respectively. Based on the results from previous subsection we used $\beta = 0.6$ instead of $\beta = \beta_0$, which is denoted by $\cmc_{0.6}$. Our method was run with 10 initial groups, while the other algorithms used the knowledge of the correct number of clusters. As mentioned in Section~\ref{sec:experiment:few}, it is not possible to run mixmod in this case, since the desired number of clusters is larger than the number of categories.

It can be seen from Figure \ref{fig:chem} that $\cmc_1$ gave the highest score among all methods when the expert was always assigning the correct labels and it was only slightly better than $\cmc_{0.6}$. In the case of 20\% and 30\% of misspecified labels it was slightly better to use $\beta = 0.6$, although the differences were very small. {\cmc} terminated usually with 6 or 7 groups.

One can observe that GMM with negative constraints was able to use this type of side information effectively. In the noiseless case, its results improved with the number of labeled data points, but not as much as with our method. In the noisy case, however, its performance dropped down. It is worth mentioning that the implementation of negative constraints with hidden Markov random fields is very costly, while our method is computationally efficient. k-means benefited from the side information in noiseless case, but deteriorated its performance when incorrect labels were introduced. GMM with positive and negative constraints, and fc-means were not able to use this type of knowledge to full effect. We observed that the use of negative constraints only has no effect on spec, i.e. its results were almost identical for any number of labeled data\footnote{We observed similar effects for most UCI data sets when we used negative constraints only in the setting of Section \ref{sec:experiment:few}. Changing the parametrization of the method did not overcome this negative behavior.}. The results of spec with both types of constraints led to some improvements, but its overall performance was quite low. We were unable to provide a satisfactory explanation for this behavior.

\subsection{Image segmentation}\label{sec:image}

\begin{figure}
\begin{center}
\subfigure[Image]{\includegraphics[width=1.7in]{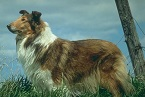}\label{fig:dog}} 
\subfigure[Side Information]{\includegraphics[width=1.7in]{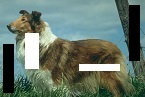}\label{fig:doglab}}
\end{center}
\caption{Test image and sample labeling dividing the picture into background and foreground. }
\label{fig:image}
\end{figure}

To further illustrate the performance of {\cmc} we applied it to an image segmentation task. We chose the picture of a dog retrieved from Berkeley Image Segmentation database \footnote{\url{https://www2.eecs.berkeley.edu/Research/Projects/CS/vision/bsds/}} presented in Figure \ref{fig:dog} (picture no. 247085 resized into $70 \times 46$ resolution) and tried to separate the shape of the dog from the background. As partition-level side information, we marked four regions by one of two labels (indicated by white and black colors, see Figure \ref{fig:doglab}). This information was passed to all considered clustering methods. In this example we focus on noiseless side information, thus we put $\beta = 1$ for {\cmc}.

To transform the image into vector data, we selected a window of size $7 \times 7$ around each pixel and used it as a feature vector of dimension $147$ (3 color intensities for 49 pixels each). Then, we applied PCA to reduce the dimension of these vectors to 5 most informative components. In consequence, we obtained a data set with 3220 data points in $\mathbb{R}^5$.

\begin{figure}
\begin{center}
\subfigure[$\cmc$]{\includegraphics[width=1.7in]{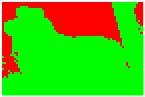}}
\subfigure[mixmod]{\includegraphics[width=1.7in]{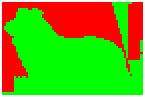}}
\subfigure[c-GMM]{\includegraphics[width=1.7in]{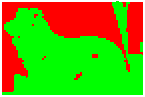}}
\subfigure[k-means]{\includegraphics[width=1.7in]{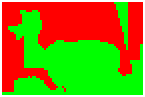}}
\subfigure[fc-means]{\includegraphics[width=1.7in]{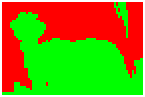}}
\subfigure[spec]{\includegraphics[width=1.7in]{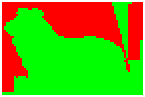}}
\end{center}
\caption{Image segmentation into foreground and background.}
\label{fig:image2}
\end{figure}

Figure \ref{fig:image2} shows the clustering results when the algorithms were run with two clusters. It can be seen that {\cmc}, mixmod, c-GMM and spec provided reasonable results. Generally though, in all cases the shape of the dog was often mixed with a part of background. This is not surprising, since 1) {\cmc}, mixmod and c-GMM are ``unimodal'', i.e. they try to detect compact groups described by single Gaussians, and 2) k-means and fc-means represent clusters by a single central point. Both background and foreground are too complex to be generalized by so simple patterns.

In order to take this into account, we first tried to detect what is a ``natural'' number of segments. For this purpose, we ran {\cmc} with 10 initial groups, which was finally reduced to 5 clusters and used this number in other algorithms. As it was shown in previous experiments using chemical compounds, must-link constraints cannot help when we have a partial labeling for two coarse classes, but we are interested in discovering their subgroups. Thus, the partition-level side information was only transformed into cannot-link constraints.

\begin{figure}
\begin{center}
\subfigure[$\cmc$]{\includegraphics[width=1.7in]{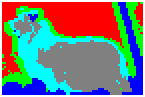}}
\subfigure[c-GMM]{\includegraphics[width=1.7in]{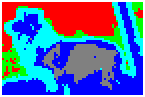}}\\
\subfigure[k-means]{\includegraphics[width=1.7in]{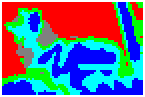}}
\subfigure[fc-means]{\includegraphics[width=1.7in]{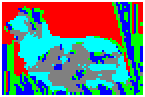}}
\end{center}
\caption{Image segmentation with five clusters. }
\label{fig:image5}
\end{figure}

The results presented in Figure \ref{fig:image5} show that {\cmc} separated the background from the foreground quite well. Each of these two regions was described by two clusters, while the fifth group was used for detecting the boundary between them. The creation of such an additional group is natural, because feature vectors were constructed using overlapping windows and the contrast between background and foreground was sharp. One may notice that c-GMM and k-means also allocated one group for the boundary (green colored cluster). Nevertheless, both methods created clusters which mixed some elements from the background and foreground (blue and cyan colored clusters). The result returned by fc-means separated the two main parts of the image, but also contains a lot of small artifacts. As in the chemical example, spec was not able to achieve reasonable results with cannot-link constraints only. Similarly, it was not possible to run mixmod in this case.

\section{Conclusion}

We introduced a semi-supervised clustering method that combines model-based clustering realized by CEC with the constraint  used by the information bottleneck method. The proposed cost function consists of three terms: the first tries to minimize the final number of clusters, the second penalizes the model for being inconsistent with side information, and the third controls the quality of data modeling. The performance of our method can be tuned by changing a weight parameter that trades between these three conflicting goals. Our method is flexible in the sense that it can be applied to both classical semi-supervised clustering tasks, as well as to tasks in which either not all classes appear in the labels or in which subgroups should be discovered based on the labels. For the latter problems, it is difficult or computationally expensive to use existing techniques. Setting the weight parameter appropriately, for which we provide a deep theoretical analysis, makes our method robust to incorrect labels. We evaluated the performance of our method on several data sets, including a case study using chemical compounds data set and an image segmentation task.

\appendix

\section{Cross-Entropy Clustering}\label{app:cec}
The empirical cross-entropy between the data set $X$ and the parametric mixture $f$ of Gaussian densities is, for a given clustering $\Y$,
\begin{align*}
 H^\times(X \| f)  &= -\frac{1}{|X|} \sum_{i=1}^k \sum_{x \in Y_i} \log(p_i f_i(x))\\
  &= - \sum_{i=1}^k \frac{|Y_i|}{|X|} \left( \log p_i + \frac{1}{|Y_i|} \sum_{x \in Y_i} \log f_i(x) \right)\\
  &= - \sum_{i=1}^k \frac{|Y_i|}{|X|} \log p_i +\sum_{i=1}^k \frac{|Y_i|}{|X|} H^\times(Y_i \| f_i) .
\end{align*}
The first sum is minimized by selecting $p_i=|Y_i|/|X|$, in which case the cross-entropy reduces to the entropy of the cluster partition
$$
H(\Y) := - \sum_{i=1}^k \frac{|Y_i|}{|X|} \log \frac{|Y_i|}{|X|}.
$$

For the second sum, recall that the cross-entropy of a Gaussian density $f=\N(\mu,\Sigma)$ with mean vector $\mu$ and covariance matrix $\Sigma$ equals:
$$
H^\times(X \| f) = \tfrac{N}{2}\ln(2\pi)+
\tfrac{1}{2}\|\mu_X-\mu\|_{\Sigma}+\tfrac{1}{2}\mathrm{tr}(\Sigma^{-1}\Sigma_X)+\tfrac{1}{2}\ln \det (\Sigma),
$$
where $\mu_X$ and $\Sigma_X$ are the sample mean vector and sample covariance matrix of $X$, respectively, and where $\|x\|_\Sigma$ is the Mahalanobis norm of $x$ with respect to $\Sigma$. The density $f \in \F$ minimizing the cross-entropy function is $f=\N(\mu_X,\Sigma_X)$, i.e., its mean equals the sample mean of $X$, and its covariance matrix equals the sample covariance matrix of $X$ \cite[Theorem 4.1]{tabor2014cross}. In this case, the cross entropy equals the differential Shannon entropy of $\N(\mu_X, \Sigma_X)$, i.e.,
$$
H^\times(X \| \N(\mu_X,\Sigma_X)) = \tfrac{N}{2}\ln(2\pi e)+\tfrac{1}{2}\ln \det (\Sigma_X) = H(\N(\mu_X, \Sigma_X)).
$$
It follows that the second sum is minimized by selecting, for every $i$, the maximum likelihood estimator of $Y_i$, $f_i=\N(\mu_{Y_i},\Sigma_{Y_i})$ \cite[Theorem 4.1, Proposition 4.1]{tabor2014cross}.

\section{Proof of Theorem~\ref{thm:condcross}}\label{app:derivation}
We consider the CEC cost function~\eqref{eq:cec} separately for each category $Z_j$ and define the conditional cross-entropy as
$$
H^\times((X \| f) |\Z) = \sum_{j=1}^m \frac{|Z_j|}{|X|} H^\times(Z_j \| f_{|j})
$$
where 
$$
f_{|j} = \max(p_1(j) f_1,\ldots,p_k(j) f_k).
$$
In other words, we assume a parameterized density model in which the weights $p_i(j)$ may depend on the category, while the densities $f_i$ may not. Rewriting above cost yields
\begin{align*}
 H^\times((X \| f)|\Z) 
 &= -\sum_{j=1}^m \frac{|Z_j|}{|X|} \sum_{x\in Z_j} \frac{1}{|Z_j|} \log f_{|j}(x)\\
 &= -\frac{1}{|X|}\sum_{j=1}^m \sum_{i=1}^k\sum_{x\in Z_j\cap Y_i} \log p_i(j)f_i(x)\\
 &= -\sum_{j=1}^m \sum_{i=1}^k\frac{|Z_j\cap Y_i|}{|X|}\log p_i(j)-\frac{1}{|X|}\sum_{i=1}^k\sum_{x\in Y_i} \log f_i(x).
\end{align*}
The second sum is minimized by the maximum likelihood estimates $f_i=\N(\mu_{Y_i},\Sigma_{Y_i})$, while the first term is minimized for $p_i(j)=\frac{|Z_j \cap Y_i|}{|Z_j|}$. We thus get
\begin{align*}
 -\sum_{j=1}^m \sum_{i=1}^k\frac{|Z_j\cap Y_i|}{|X|}\log p_i(j) &= -\sum_{j=1}^m \sum_{i=1}^k\frac{|Z_j\cap Y_i|}{|X|}\log \frac{|Z_j\cap Y_i|}{|Z_j|}\\
 &= H(\Y|\Z)\\
 &= H(\Z|\Y) + H(\Y) - H(\Z)
\end{align*}
by the chain rule of entropy. Since $H(\Z)$ does not depend on the clustering $\Y$,  the minimization of the above conditional cross-entropy is equivalent to the minimization of
\begin{align}
 H(\Y) + \sum_{i=1}^k \frac{|Y_i|}{|X|} \crossfi + H(\Z| \Y).\label{costFull}
\end{align}
This is exactly the cost~\eqref{eq:costCond} for $\beta=1$.\qed

\section{Convergence Speed of Hartigan-Based CEC} \label{app:optimizationTime}

\begin{table}[t]\footnotesize
\caption{Mean number of iterations that Hartigan CEC, EM-based GMM and Lloyd k-means need to converge.}
\label{tab:convergence}	
\medskip
\setlength{\arrayrulewidth}{0.1mm}
\setlength{\tabcolsep}{3pt}
\centering
\begin{tabular}{lcccc}
\bf \shortstack{Data \\ set} & \bf \shortstack{Hartigan \\ CEC} & \bf \shortstack{EM-based \\ GMM} & \bf \shortstack{Hartigan \\ k-means}  & \bf \shortstack{Lloyd \\ k-means}   \\ \hline
Ecoli & 6.4 & 18.6 & 3.2 & 10.4\\
Glass & 5.5 & 15.7 & 3.3 & 9.7 \\
Iris & 5.1 & 19.1 & 2.3 & 7\\
Segmentation & 4.4 & 16.7 & 3.3 & 9.3\\
User Modeling & 8.5 & 48.2 & 3.9 & 10.9\\
Vertebral & 7.6 & 17.8 & 2.3 & 9\\
Wine & 7.6 & 13.6 & 2.3 & 8.6\\
\hline
\end{tabular}
\\ \vspace{0.1in}
\end{table}

We compared the number of iterations that CEC, EM, and k-means required to converge to a local minimum. We used the seven UCI data sets from Table \ref{tab:data1} and averaged the results over ten runs. Side information was not considered in these experiments. Table \ref{tab:convergence} shows that the Hartigan heuristic applied to the CEC cost function converges faster than EM does for fitting a GMM. The same holds when comparing the Hartigan algorithm with Lloyd's method applied to k-means. Similar results were obtained in an experimental evaluation \cite{cec_guide}. We also found out that that {\cmc} uses a similar number of iterations as CEC; however, the convergence speed varies with the particular sample of side information, which makes a reliable comparison more difficult.

\section{Chain Rule for Proportional Partitions}\label{app:chainrule}
We now show that, for partitions $\Y$ proportional to $\Z$, the chain rule of entropy can be applied, i.e., 
$$
H(\Y) + H(\Z|\Y)=H(\Z,\Y).
$$
We have
\begin{align*}
 &H(\Y) + H(\Z|\Y)\\
 &= -\sum_{i=1}^k \frac{|Y_i|}{|X|} \log \frac{|Y_i|}{|X|} - \sum_{i=1}^k  \frac{|Y_i|}{|X|} \sum_{j=1}^m \frac{|Y_i \cap Z_j|}{|Y_i \cap X_\ell|} \log \frac{|Y_i \cap Z_j|}{|Y_i \cap X_\ell|}\\
 &\stackrel{(a)}{=} -\sum_{i=1}^k \frac{|Y_i \cap X_\ell|}{|X_\ell|} \log \frac{|Y_i \cap X_\ell|}{|X_\ell|} - \sum_{i=1}^k  \frac{|Y_i \cap X_\ell|}{|X_\ell|} \sum_{j=1}^m \frac{|Y_i \cap Z_j|}{|Y_i \cap X_\ell|} \log \frac{|Y_i \cap Z_j|}{|Y_i \cap X_\ell|}\\
 &= -\sum_{i=1}^k \sum_{j=1}^m \frac{|Y_i \cap Z_j|}{|X_\ell|} \log \frac{|Y_i \cap X_\ell|}{|X_\ell|} - \sum_{i=1}^k  \sum_{j=1}^m \frac{|Y_i \cap Z_j|}{|X_\ell|} \log \frac{|Y_i \cap Z_j|}{|Y_i \cap X_\ell|}\\
 &= -\sum_{i=1}^k \sum_{j=1}^m \frac{|Y_i \cap Z_j|}{|X_\ell|} \log  \frac{|Y_i \cap Z_j|}{|X_\ell|}\\
 &= H(\Z,\Y)
\end{align*}
where $(a)$ is because $\Y$ is proportional to $\Z$ and thus $\frac{|Y_i \cap X_\ell|}{|X_\ell|}=\frac{|Y_i|}{|X|}$. In a similar manner it can be shown that
$$
H(\Z) + H(\Y|\Z)=H(\Z,\Y)
$$
where $H(\Z)=-\sum_{j=1}^m \frac{|Z_j|}{|X_\ell|} \log \frac{|Z_j|}{|X_\ell|}$ and where
$$
H(\Y|\Z) = \sum_{j=1}^m \frac{|Z_j|}{|X_\ell|} H(\Y|Z_j) = -\sum_{j=1}^m\sum_{i=1}^k \frac{|Y_i \cap Z_j|}{|X_\ell|} \log\frac{|Y_i\cap Z_j|}{|Z_j|}.
$$

\section{Proof of Theorem~\ref{thm:split}}\label{app:proofsplit}
\begin{lemma} \label{lem:split}
Let the data set $X \subset \R^N$ be partitioned into two clusters $Y_1$ and $Y_2$ such that the sample covariance matrix $\Sigma_i$ of $Y_i$ is positive definite for $i=1,2$.

Then 
$$
H(\N(\mu_X,\Sigma_X)) \geq   \frac{|Y_1|}{|X|} H(\N(\mu_{Y_1},\Sigma_{Y_1}))+ \frac{|Y_2|}{|X|} H(\N(\mu_{Y_2},\Sigma_{Y_2})).
$$
\end{lemma}

\begin{proof}
Let $p=|Y_1|/|X|$. By the law of total (co-)variance, we have
\begin{multline}
  \Sigma_X = p\Sigma_{Y_1} + (1-p)\Sigma_{Y_2}\\ + \underbrace{p(\mu_X-\mu_{Y_1})(\mu_X-\mu_{Y_1})^T+ (1-p)(\mu_X-\mu_{Y_2})(\mu_X-\mu_{Y_2})^T}_{=:\tilde\Sigma}
\end{multline}
where $\tilde\Sigma$ is the covariance matrix obtained from the sample mean vectors $\mu_{Y_1}$ and $\mu_{Y_2}$ of $Y_1$ and $Y_2$. Consequently, we get
\begin{align*}
& p H(\N(\mu_{Y_1},\Sigma_{Y_1})) + (1-p) H(\N(\mu_{Y_2},\Sigma_{Y_2})) \\ 
&= \frac{Np}{2} \ln(2 \pi e) + \frac{p}{2} \ln(\det\Sigma_{Y_1})+ \frac{N(1-p)}{2}\ln(2 \pi e) + \frac{(1-p)}{2} \ln(\det\Sigma_{Y_2})\\
&= \frac{N}{2}\ln(2 \pi e) + \frac{1}{2}\ln\left((\det\Sigma_{Y_1})^{p} (\det\Sigma_{Y_2})^{(1-p)}\right)\\
&\stackrel{(a)}{\leq} \frac{N}{2}\ln(2 \pi e) + \frac{1}{2}\ln(\det(p\Sigma_{Y_1}+(1-p)\Sigma_{Y_2}))\\
&\stackrel{(b)}{\leq} \frac{N}{2}\ln(2 \pi e) + \frac{1}{2}\ln(\det(p\Sigma_{Y_1}+(1-p)\Sigma_{Y_2}+\tilde\Sigma))\\
&= \frac{N}{2}\ln(2 \pi e) + \frac{1}{2}\ln(\det\Sigma)\\
&= H( \N(m,\Sigma_X))
\end{align*}
where $(a)$ follows because $\Sigma_{Y_1}$ and $\Sigma_{Y_2}$ are positive definite and from~\cite[Cor.~7.6.8]{Horn_Matrix}, and where $(b)$ follows because $\tilde\Sigma$ is positive semi-definite and from, e.g.,~\cite[Cor.~4.3.12]{Horn_Matrix}.
\end{proof}

\begin{proof}[Proof of Theorem \ref{thm:split}]
Since $\Y$ is proportional to $\Z$ and since the coarsening of a proportional partition is proportional, we can apply the chain rule of entropy to get (see~\ref{app:chainrule})
\begin{align*}
 H(\Y) + H(\Z|\Y) &= H(\Z,\Y)\\
 H(\tilde{\Y}) + H(\Z|\tilde{\Y}) &= H(\Z,\tilde{\Y}).
\end{align*}
We hence get
\begin{multline}\label{eq:proportionalCoarsening}
 H(\Y) + H(\Z|\Y) = H(\Z,\Y)
 \stackrel{(a)}{=} H(\Z,\Y,\tilde{\Y})
 = H(\Z,\tilde{\Y})+H(\Y|\Z,\tilde{\Y})\\
 \stackrel{(b)}{=} H(\Z,\tilde{\Y})
 = H(\tilde{\Y}) + H(\Z|\tilde{\Y})
\end{multline}
where $(a)$ is because $\tilde{\Y}$ is a coarsening of $\Y$ and $(b)$ is because $\Y$ is a coarsening of $\Z$, respectively. In other words, for proportional coarsenings of $\Z$, consistency with $\Z$ (measured by the conditional entropy) can be freely traded for model simplicity (measured by entropy).

For the remaining part of the RHS of \eqref{eq:inquality}, we write:
$$
\sum_{i=1}^k \frac{|Y_i|}{|X|} \crossfi = \sum_{j=1}^{k'} \frac{|\tilde{Y}_j|}{|X|} \sum_{i: Y_i\subseteq\tilde{Y}_j} \frac{|Y_i|}{|\tilde{Y}_j|} \crossfi.
$$

If the inner sums on the RHS consist of at most two terms, i.e. $\tilde{Y}_j=\{Y_{i_1},Y_{i_2}\}$ or $\tilde{Y}_j=Y_{i}$, the inequality is established by Lemma~\ref{lem:split}. If the inner sum consists of more than two clusters, one needs to apply Lemma~\ref{lem:split} recursively. For example, if $\tilde{Y}_1=\{Y_1,Y_2,Y_3\}$,
\begin{multline*}
\sum_{i=1}^3 \frac{|Y_i|}{|\tilde{Y}_1|} \crossfi\\
 \le  \frac{|Y_1\cup Y_2|}{|\tilde{Y}_1|} H(\N(\mu_{Y_1 \cup Y_2},\Sigma_{Y_1 \cup Y_2})) + \frac{|Y_3|}{|\tilde{Y}_1|} H(\N(\mu_{Y_3},\Sigma_{Y_3}))\\
 \le H(\N(\mu_{\tilde{Y}_1},\Sigma_{\tilde{Y}_1}).
\end{multline*}
This completes the proof.
\end{proof}

\section{Proof of Theorem \ref{thm:limitBeta}} \label{app:proof:lastthm}

Note that, with~\eqref{eq:betacost} and~\eqref{eq:proportionalCoarsening} (since both $\Y$ and $\tilde\Y$ are proportional coarsenings of $\Z$), we obtain
\begin{align*}
 &H(\tilde\Y)+\beta H(\Z|\tilde\Y)-H(\Y) - \beta H(\Z|\Y)\\
 &= H(\tilde\Y)+\beta H(\Z|\tilde\Y) - (1-\beta)H(\Y) - \beta (H(\tilde\Y)+H(\Z|\tilde\Y))\\
 &=(1-\beta)\left(H(\tilde\Y)-H(\Y)\right)\\
 &=(\beta-1) H(\Y|\tilde\Y)=(\beta-1) H\left(\frac{p_{k'}}{q_{k'}},\dots,\frac{p_k}{q_{k'}}\right).
\end{align*}
For $i=1,\dots,k'-1$, we have $\crossfi=H(\N(\mu_{\tilde{Y}_i},\Sigma_{\tilde{Y}_i}))$, hence only the last term remains. We obtain with the proof of Theorem~\ref{thm:split},
\begin{align*}
H(\N(\mu,\Sigma))-\sum_{i=k'}^k \frac{p_i}{q_{k'}} \crossfi = \sum_{i=k'}^k \frac{p_i}{2q_{k'}} \ln\left(\frac{\det\Sigma}{\det\Sigma_{Y_i}}\right).
\end{align*}
It follows that the two costs in the statement are equal for $\beta_0$ such that
$$
(\beta_0-1) H\left(\frac{p_{k'}}{q_{k'}},\dots,\frac{p_k}{q_{k'}}\right) = \sum_{i=k'}^k \frac{p_i}{2q_{k'}} \ln\left(\frac{\det\Sigma_{Y_i}}{\det\Sigma}\right)
$$
from which we get
\begin{equation}
 \beta_0 = 1+\frac{\sum_{i=k'}^k \frac{p_i}{2q_{k'}} \ln\left(\frac{\det\Sigma_{Y_i}}{\det\Sigma}\right) }{H\left(\frac{p_{k'}}{q_{k'}},\dots,\frac{p_k}{q_{k'}}\right)}.
\end{equation}\qed

\section*{Acknowledgement}

The authors thank Hongfu Liu, Pengjiang Qian and Daphne Teck Ching Lai for sharing their codes implementing semi-supervised versions of k-means, spectral clustering and fuzzy clustering. We also thank Jacek Tabor for useful discussions and comments.

The work of Marek \'Smieja was supported by the National Science Centre (Poland) grant no. 2016/21/D/ST6/00980. The work of Bernhard C. Geiger has been funded by the Erwin Schr\"odinger
Fellowship J 3765 of the Austrian Science Fund and by the German Ministry of Education and Research in the framework of an Alexander von Humboldt Professorship.

\section*{References}

\bibliographystyle{plain}      
\bibliography{bib}   

\end{document}